\documentclass[twoside,11pt]{article}
\usepackage{jmlr2e}
\usepackage{amsmath}
\usepackage{mathrsfs}
\usepackage{dsfont}
\usepackage{enumerate}
\usepackage{subfig}
\usepackage{float}

\newtheorem{assumption}{Assumption}

\DeclareMathOperator*{\argmin}{arg\,min}

\hypersetup{
colorlinks = true,
urlcolor = blue,
linkcolor = blue,
citecolor = blue,
}

\ShortHeadings{Some Theoretical Insights into Wasserstein GANs}{Biau, Sangnier and Tanielian}
\firstpageno{1}

\usepackage{lastpage}
\jmlrheading{22}{2021}{1-\pageref{LastPage}}{6/20}{5/21}{20-553}{G\'erard Biau, Maxime Sangnier and Ugo Tanielian}
\ShortHeadings{Some Theoretical Insights into Wasserstein GANs}{Biau, Sangnier and Tanielian}

\begin{document}
\title{Some Theoretical Insights into Wasserstein GANs}

\author{\name G\'erard Biau 
        \email gerard.biau@sorbonne-universite.fr\\
       \addr Laboratoire de Probabilit\'es, Statistique et Mod\'elisation\\
       Sorbonne Universit\'e\\
       4 place Jussieu\\
       75005 Paris, France
       \AND
       \name Maxime Sangnier \email maxime.sangnier@sorbonne-universite.fr \\
       \addr Laboratoire de Probabilit\'es, Statistique et Mod\'elisation\\
       Sorbonne Universit\'e\\
       4 place Jussieu\\
       75005 Paris, France
      \AND
       \name Ugo Tanielian \email u.tanielian@criteo.com \\
       \addr Laboratoire de Probabilit\'es, Statistique et Mod\'elisation \& Criteo AI Lab\\
       Criteo AI Lab\\
       32 rue Blanche\\
       75009 Paris, France}

\editor{Nicolas Vayatis}
\maketitle

\begin{abstract}
Generative Adversarial Networks (GANs) have been successful in producing outstanding results in areas as diverse as image, video, and text generation. Building on these successes, a large number of empirical studies have validated the benefits of the cousin approach called Wasserstein GANs (WGANs), which brings stabilization in the training process. In the present paper, we add a new stone to the edifice by proposing some theoretical advances in the properties of WGANs. First, we properly define the architecture of WGANs in the context of integral probability metrics parameterized by neural networks and highlight some of their basic mathematical features. We stress in particular interesting optimization properties arising from the use of a parametric $1$-Lipschitz discriminator. Then, in a statistically-driven approach, we study the convergence of empirical WGANs as the sample size tends to infinity, and clarify the adversarial effects of the generator and the discriminator by underlining some trade-off properties. These features are finally illustrated with experiments using both synthetic and real-world datasets.
\end{abstract}


\section{Introduction}\label{section:introduction}
Generative Adversarial Networks (GANs) is a generative framework proposed by \citet{GANs}, in which two models (a generator and a discriminator) act as adversaries in a zero-sum game. Leveraging the recent advances in deep learning, and specifically convolutional neural networks \citep{lecun98gradientbasedlearning}, a large number of empirical studies have shown the impressive possibilities of GANs in the field of image generation \citep{radford2015unsupervised,ledig2017photo,karras2017progressive,brock2018large}. Lately, \citet{karras2018style} proposed an architecture able to generate hyper-realistic fake human faces that cannot be differentiated from real ones (see the website \href{http://www.thispersondoesnotexist.com}{thispersondoesnotexist.com}). The recent surge of interest in the domain also led to breakthroughs in video \citep{acharya2018towards}, music \citep{mogren2016c}, and text generation \citep{SeqGANs,MaskGANs}, among many other potential applications.

The aim of GANs is to generate data that look ``similar'' to samples collected from some unknown probability measure $\mu^\star$, defined on a Borel subset $E$ of $\mathds{R}^D$. In the targeted applications of GANs, $E$ is typically a submanifold (possibly hard to describe) of a high-dimensional $\mathds{R}^D$, which therefore prohibits the use of classical density estimation techniques. GANs approach the problem by making two models compete: the generator, which tries to imitate $\mu^\star$ using the collected data, vs.~the discriminator, which learns to distinguish the outputs of the generator from the samples, thereby forcing the generator to improve its strategy.

Formally, the generator has the form of a parameterized class of Borel functions from $\mathds{R}^d$ to $E$, say $\mathscr{G}= \{G_\theta: \theta \in \Theta \}$, where $\Theta \subseteq \mathds{R}^P$ is the set of parameters describing the model. Each function $G_\theta$ takes as input a $d$-dimensional random variable $Z$---it is typically uniform or Gaussian, with $d$ usually small---and outputs the ``fake'' observation $G_{\theta}(Z)$ with distribution $\mu_\theta$. Thus, the collection of probability measures $\mathscr{P}= \{ \mu_\theta : \theta \in \Theta \}$ is the natural class of distributions associated with the generator, and the objective of GANs is to find inside this class the distribution that generates the most realistic samples, closest to the ones collected from the unknown $\mu^\star$. On the other hand, the discriminator is described by a family of Borel functions from $E$ to $[0,1]$, say $\mathscr{D} = \{D_\alpha: \alpha \in \Lambda \}$, $\Lambda \subseteq \mathds{R}^Q$, where each $D_{\alpha}$ must be thought of as the probability that an observation comes from $\mu^\star$ (the higher $D(x)$, the higher the probability that $x$ is drawn from $\mu^\star$).

In the original formulation of \citet{GANs}, GANs make $\mathscr G$ and $\mathscr D$ fight each other through the following objective:
\begin{equation}\label{standard_GANs}
    \underset{\theta \in \Theta}{\inf} \ \underset{\alpha \in \Lambda}{\sup} \Big[ \mathds{E} \log(D_\alpha(X)) + \mathds{E} \log(1-D_{\alpha}(G_\theta(Z)))\Big],
\end{equation}
where $X$ is a random variable with distribution $\mu^{\star}$ and the symbol $\mathds{E}$ denotes expectation. Since one does not have access to the true distribution, $\mu^\star$  is replaced in practice with the empirical measure $\mu_n$ based on independent and identically distributed (i.i.d.) samples $X_1, \hdots, X_n$ distributed as $X$, and the practical objective becomes
\begin{equation}\label{empirical_standard_GANs}
    \underset{\theta \in \Theta}{\inf} \ \underset{\alpha \in \Lambda}{\sup} \Big[ \frac{1}{n} \sum_{i=1}^n \log(D_\alpha(X_i)) + \mathds{E} \log(1-D_{\alpha}(G_\theta(Z)))\Big].
\end{equation}
In the literature on GANs, both $\mathscr{G}$ and $\mathscr{D}$ take the form of neural networks (either feed-forward or convolutional, when dealing with image-related applications). This is also the case in the present paper, in which the generator and the discriminator will be parameterized by feed-forward neural networks with, respectively, rectifier \citep{glorot2011deep} and GroupSort \citep{chernodub2016norm} activation functions. We also note that from an optimization standpoint, the minimax optimum in \eqref{empirical_standard_GANs} is found by using stochastic gradient descent alternatively on the generator's and the discriminator's parameters.

In the initial version \eqref{standard_GANs}, GANs were shown to reduce, under appropriate conditions, the Jensen-Shanon divergence between the true distribution and the class of parameterized distributions \citep{GANs}. This characteristic was further explored by \citet{biau2018some} and \cite{Moulines}, who stressed some theoretical guarantees regarding the approximation and statistical properties of problems \eqref{standard_GANs} and \eqref{empirical_standard_GANs}. However, many empirical studies \citep[e.g.,][]{metz2016unrolled,salimans2016improved} have described cases where the optimal generative distribution computed by solving \eqref{empirical_standard_GANs} collapses to a few modes of the distribution $\mu^\star$. This phenomenon is known under the term of mode collapse and has been theoretically explained by \citet{ArBo17}. As a striking result, in cases where both $\mu^\star$ and $\mu_\theta$ lie on disjoint supports, these authors proved the existence of a perfect discriminator with null gradient on both supports, which consequently does not convey meaningful information to the generator.

To cancel this drawback and stabilize training, \citet{arjovsky2017wasserstein} proposed a modification of criterion \eqref{standard_GANs}, with a framework called Wasserstein GANs (WGANs). In a nutshell, the objective of WGANs is to find, inside the class of parameterized distributions $\mathscr P$, the one that is the closest to the true $\mu^\star$ with respect to the Wasserstein distance \citep{villani2008optimal}. In its dual form, the Wasserstein distance can be considered as an integral probability metric \citep[IPM,][]{IPMsMuller} defined on the set of $1$-Lipschitz functions. Therefore, the proposal of \citet{arjovsky2017wasserstein} is to replace the $1$-Lipschitz functions with a discriminator parameterized by neural networks. To practically enforce this discriminator to be a subset of $1$-Lipschitz functions, the authors use a weight clipping technique on the set of parameters. A decisive step has been taken by \citet{gulrajani2017improved}, who stressed the empirical advantage of the WGANs architecture by replacing the weight clipping with a gradient penalty. Since then, WGANs have been largely recognized and studied by the Machine Learning community \citep[e.g.,][]{roth2017stabilizing,petzka2018regularization,wei2018improving,karras2018style}.

A natural question regards the theoretical ability of WGANs to learn $\mu^{\star}$, considering that one only has access to the parametric models of generative distributions and discriminative functions. Previous works in this direction are those of \citet{liang2018well} and \citet{zhang2018discrimination}, who explore generalization properties of WGANs. In the present paper, we make one step further in the analysis of mathematical forces driving WGANs and contribute to the literature in the following ways:
\begin{enumerate}[$(i)$]
    \item We properly define the architecture of WGANs parameterized by neural networks. Then, we highlight some properties of the IPM induced by the discriminator, and finally stress some basic mathematical features of the WGANs framework (Section~\ref{section:framework}).
    \item We emphasize the impact of operating with a parametric discriminator contained in the set of $1$-Lipschitz functions. We introduce in particular the notion of monotonous equivalence and discuss its meaning in the mechanism of WGANs. We also highlight the essential role played by piecewise linear functions (Section~\ref{section:approximation_properties}).
    \item In a statistically-driven approach, we derive convergence rates for the IPM induced by the discriminator, between the target distribution $\mu^{\star}$ and the distribution output by the WGANs based on i.i.d.~samples (Section~\ref{section:asymptotic_properties}). We show in particular that when studying such IPMs, the smaller the network is, the faster the empirical measure converges towards $\mu^\star$. 
    \item Building upon the above, we clarify the adversarial effects of the generator and the discriminator by underlining some trade-off properties. These features are illustrated with experiments using both synthetic and real-world datasets (Section~\ref{section:trade_off_properties}).
\end{enumerate}
For the sake of clarity, proofs of the most technical results are gathered in the Appendix.
\section{Wasserstein GANs}\label{section:framework}
The present section is devoted to the presentation of the WGANs framework. After having given a first set of definitions and results, we stress the essential role played by IPMs and study some optimality properties of WGANs.
\subsection{Notation and definitions}\label{section:notation}
Throughout the paper, $E$ is a Borel subset of $\mathds R^D$, equipped with the Euclidean norm $\|\cdot\|$, on which $\mu^\star$ (the target probability measure) and the $\mu_{\theta}$'s (the candidate probability measures) are defined. Depending on the practical context, $E$ can be equal to $\mathds R^D$, but it can also be a submanifold of it. We emphasize that there is no compactness assumption on $E$.

For $K\subseteq E$, we let $C(K)$ (respectively, $C_b(K)$) be the set of continuous (respectively, continuous bounded) functions from $K$ to $\mathds{R}$. We denote by $\text{Lip}_1$ the set of $1$-Lipschitz real-valued functions on $E$, i.e.,
\begin{equation*}
    \text{Lip}_1 = \big\{ f :E\to \mathds R: |f(x)-f(y)|\leqslant {\|x-y\|}, \ (x,y) \in E^2\big\}.
\end{equation*}
The notation $P(E)$ stands for the collection of Borel probability measures on $E$, and ${P}_1 (E)$ for the subset of probability measures with finite first moment, i.e.,
\begin{equation*}
    {P}_1 (E) = \big\{ \mu \in {P} (E) : \int_{E} \|x_0-x\| \mu(\rm{d}x) < \infty \big\},
\end{equation*}
where $x_0 \in E$ is arbitrary (this set does not depend on the choice of the point $x_0$). Until the end, it is assumed that $\mu^\star \in {P}_1 (E)$. It is also assumed throughout that the random variable $Z \in \mathds R^d$ is a sub-Gaussian random vector \citep{jin2019short}, i.e., $Z$ is integrable and there exists $\gamma>0$ such that
    \begin{equation*}
        \forall v \in \mathds{R}^d, \ \mathds{E} e^{v \cdot (Z -\mathds{E} Z)} \leqslant e^{\frac{\gamma^2 \|v\|^2}{2}},
    \end{equation*}
where $\cdot$ denotes the dot product in $\mathds{R}^d$ and $\|\cdot\|$ the Euclidean norm. The sub-Gaussian property is a constraint on the tail of the probability distribution. As an example, Gaussian random variables on the real line are sub-Gaussian and so are bounded random vectors. We note that $Z$ has finite moments of all nonnegative orders \citep[][Lemma 2]{jin2019short}. Assuming that $Z$ is sub-Gaussian is a mild requirement since, in practice, its distribution is most of the time uniform or Gaussian.

As highlighted earlier, both the generator and the discriminator are assumed to be parameterized by feed-forward neural networks, that is,
\begin{equation*}
    \mathscr{G} = \{G_\theta: \theta \in \Theta \}
    \qquad \text{and} \qquad
    \mathscr{D} = \{D_\alpha: \alpha \in \Lambda \}
\end{equation*}
with $\Theta \subseteq \mathds{R}^P$, $\Lambda \subseteq \mathds{R}^Q$, and, for all $z \in \mathds R^d$,
\begin{equation}\label{eq:def_generators}
  G_{\theta}(z)=\underset{D \times u_{p-1}}{U_{p}} \sigma \big(\underset{u_{p-1} \times u_{p-2}}{U_{p-1}} \cdots \sigma(\underset{u_2 \times u_1}{U_2}\sigma(\underset{u_1 \times d}{U_1} z + \underset{u_1 \times 1}{b_1}) + \underset{u_2 \times 1}{b_2}) \cdots + \underset{u_{p-1} \times 1}{b_{p-1}} \big) + \underset{D\times 1}{b_p},
\end{equation}
for all $x\in E$,
\begin{equation}\label{eq:def_discriminators}
    D_{\alpha}(x)= \underset{1 \times v_{q-1}}{V_{q}} \Tilde{\sigma} \big(\underset{v_{q-1} \times v_{q-2}}{V_{q-1}} \cdots \Tilde{\sigma} (\underset{v_2 \times v_1}{V_2} \Tilde{\sigma} (\underset{v_1 \times D}{V_1} x + \underset{v_1\times 1}{c_1}) + \underset{v_2 \times 1}{c_2}) + \cdots +\underset{v_{q-1}\times 1}{c_{q-1}}\big) + \underset{1 \times 1}{c_q},
\end{equation}
where $p,q \geqslant 2$ and the characters below the matrices indicate their dimensions ($\mbox{lines} \times \mbox{columns}$). Some comments on the notation are in order. Networks in $\mathscr G$ and $\mathscr D$ have, respectively, $(p-1)$ and $(q-1)$ hidden layers. Hidden layers from depth $1$ to $(p-1)$ (for the generator) and from depth $1$ to $(q-1)$ (for the discriminator) are assumed to be of respective even widths $u_i$, $i=1, \hdots, p-1$, and $v_i$, $i=1, \hdots, q-1$. The matrices $U_i$ (respectively, $V_i$) are the matrices of weights between layer $i$ and layer $(i+1)$ of the generator (respectively, the discriminator), and the $b_i$'s (respectively, the $c_i$'s) are the corresponding offset vectors (in column format). We let $\sigma(x) = \text{max}(x, 0)$ be the rectifier activation function (applied componentwise) and
be the GroupSort activation function with a grouping size equal to 2 (applied on pairs of components, which makes sense in \eqref{eq:def_discriminators} since the widths of the hidden layers are assumed to be even). GroupSort has been introduced in \citet{chernodub2016norm} as a $1$-Lipschitz activation function that preserves the gradient norm of the input. This activation can recover the rectifier, in the sense that $\Tilde{\sigma}(x, 0) = (\sigma(x), -\sigma(-x))$, but the converse is not true. The presence of GroupSort is critical to guarantee approximation properties of Lipschitz neural networks \citep{Anil2018SortingOL,Huster2018LimitationsOT}, as we will see later.

Therefore, denoting by $\mathscr{M}_{(j,k)}$ the space of matrices with $j$ rows and $k$ columns, we have $U_1 \in \mathscr{M}_{(u_1,d)}$, $V_1 \in \mathscr{M}_{(v_1,D)}$, $b_1 \in \mathscr{M}_{(u_1,1)}$, $c_1 \in \mathscr{M}_{(v_1,1)}$, $U_p \in \mathscr{M}_{(D,u_{p-1})}$, $V_q \in \mathscr{M}_{(1,v_{q-1})}$, $b_p \in \mathscr{M}_{(D,1)}$, $c_q \in \mathscr{M}_{(1,1)}$. All the other matrices $U_i$, $i = 2, \hdots, p-1$, and $V_i$, $i =2, \hdots, q-1$, belong to $\mathscr{M}_{(u_i,u_{i-1})}$ and $\mathscr{M}_{(v_i,v_{i-1})}$, and vectors $b_i$, $i = 2, \hdots, p-1$, and $c_i$, $i =2, \hdots, q-1$, belong to $\mathscr{M}_{(u_i,1)}$ and $\mathscr{M}_{(v_i,1)}$. So, altogether, the vectors $\theta=(U_1, \hdots, U_p, b_1, \hdots, b_p)$ (respectively, the vectors $\alpha=(V_1, \hdots, V_q, c_1, \hdots, c_q)$) represent the parameter space $\Theta$ of the generator $\mathscr G$ (respectively, the parameter space $\Lambda$ of the discriminator $\mathscr D$). We stress the fact that the outputs of networks in $\mathscr{D}$ are not restricted to $[0,1]$ anymore, as is the case for the original GANs of \citet{GANs}. We also recall the notation $\mathscr P=\{\mu_{\theta}: \theta \in \Theta\}$, where, for each $\theta$, $\mu_\theta$ is the probability distribution of $G_\theta(Z)$. Since $Z$ has finite first moment and each $G_{\theta}$ is piecewise linear, it is easy to see that $\mathscr{P} \subset P_1(E)$.

Throughout the manuscript, the notation $\|\cdot\|$ (respectively, $\|\cdot\|_{\infty}$) means the Euclidean (respectively, the supremum) norm on $\mathds R^k$, with no reference to $k$ as the context is clear. For $W=(w_{i,j})$ a matrix in $\mathscr{M}_{(k_1,k_2)}$, we let $\|W\|_2 = \sup_{\|x\|=1} \|Wx\|$ be the $2$-norm of $W$. Similarly, the $\infty$-norm of $W$ is $\|W\|_\infty = \sup_{\|x\|_{\infty}=1} \|Wx\|_{\infty}=\max_{i=1, \hdots, k_1} \sum_{j=1}^{k_2} |w_{i,j}|$. We will also use the $(2,\infty)$-norm of $W$, i.e., $\|W\|_{2, \infty} = \sup_{\|x\|=1} \|Wx\|_\infty$. We shall constantly need the following assumption:
\begin{assumption}[Compactness]\label{ass:compactness}
For all $\theta=(U_1, \hdots, U_p, b_1, \hdots, b_p) \in \Theta$,
\begin{equation*}
    \max(\|U_i\|_2, \|b_i\| : i =1,\hdots,p) \leqslant K_1,
\end{equation*}
where $K_1>0$ is a constant. Besides, for all $\alpha=(V_1, \hdots, V_q, c_1, \hdots, c_q) \in \Lambda$,
\begin{equation*}
    \|V_1\|_{2, \infty} \leqslant 1, \ \max(\|V_i\|_\infty : i =2,\hdots,q) \leqslant 1, \ \emph{and} \ \max(\|c_i\|_\infty: i = 1,\hdots,q) \leqslant K_2,
\end{equation*}
    where $K_2 \geqslant 0$ is a constant.
\end{assumption}

This compactness requirement is classical when parameterizing WGANs \citep[e.g.,][]{arjovsky2017wasserstein,zhang2018discrimination, Anil2018SortingOL}. In practice, one can satisfy Assumption \ref{ass:compactness} by clipping the parameters of neural networks as proposed by \citet{arjovsky2017wasserstein}. An alternative approach to enforce $\mathscr{D} \subseteq \text{Lip}_1$ consists in penalizing the gradient of the discriminative functions, as proposed by \citet{gulrajani2017improved}, \citet{kodali2017convergence}, \citet{wei2018improving}, and \citet{zhou2019lipschitz}. This solution was  empirically found to be more stable. The usefulness of Assumption \ref{ass:compactness} is captured by the following lemma.
\begin{lemma}\label{lem:uniformly_lipschitz_neural_nets}
    Assume that Assumption \ref{ass:compactness} is satisfied. Then, for each $\theta \in \Theta$, the function $G_\theta$ is $K_1^p$-Lipschitz on $\mathds R^{d}$. In addition, $\mathscr D \subseteq \emph{Lip}_1$.
\end{lemma}

Recall \citep[e.g.,][]{dudley_2002} that a sequence of probability measures $(\mu_k)$ on $E$ is said to converge weakly to a probability measure $\mu$ on $E$ if, for all $\varphi \in C_b(E)$,
\begin{equation*}
    \int_E \varphi \ {\rm d}\mu_k \underset{k \to \infty}{\to} \int_E \varphi \ {\rm d}\mu.
\end{equation*}
In addition, the sequence of probability measures $(\mu_k)$ in $P_1(E)$ is said to converge weakly in $P_1(E)$ to a probability measure $\mu$ in $P_1(E)$ if $(i)$ $(\mu_k)$ converges weakly to $\mu$ and if $(ii)$ $\int_{E} \|x_0-x\| \mu_k({\rm d} x) \to \int_{E} \|x_0-x\| \mu ({\rm d} x)$, where $x_0 \in E$ is arbitrary \citep[][Definition 6.7]{villani2008optimal}. The next proposition offers a characterization of our collection of generative distributions $\mathscr{P}$ in terms of compactness with respect to the weak topology in $P_1(E)$. This result is interesting as it gives some insight into the class of probability measures generated by neural networks.
\begin{proposition}\label{prop:neural_nets_tightness}
Assume that Assumption \ref{ass:compactness} is satisfied. Then the function $\Theta \ni \theta \mapsto \mu_\theta$ is continuous with respect to the weak topology in $P_1(E)$, and the set of generative distributions $\mathscr{P}$ is compact with respect to the weak topology in $P_1(E)$.
\end{proposition}
\subsection{The WGANs and T-WGANs problems}
We are now in a position to formally define the WGANs problem. The Wasserstein distance (of order $1$) between two probability measures $\mu$ and $\nu$ in $P_1(E)$ is defined by
\begin{equation*}
W_1(\mu, \nu)=\inf_{\pi \in \Pi (\mu, \nu)}\int_{E \times E} \|x-y\|\pi({\rm d}x,{\rm d}y),
\end{equation*}
where $\Pi(\mu, \nu)$ denotes the collection of all joint probability measures on $E \times E$ with marginals $\mu$ and $\nu$ \citep[e.g.,][]{villani2008optimal}. It is a finite quantity. In the present article, we will use the dual representation of $W_1(\mu, \nu)$, which comes from the duality theorem of \citet{kantorovich1958space}:
\begin{equation*}
    W_1(\mu, \nu) = \underset{f \in  {\text{Lip}}_1}{\sup} |\mathds{E}_{\mu} f - \mathds{E}_{\nu} f|,
\end{equation*}
where, for a probability measure $\pi$, $\mathds E_{\pi} f=\int_{E} f {\rm d} \pi$ (note that for $f \in \text{Lip}_1$ and $\pi \in P_1(E)$, the function $f$ is Lebesgue integrable with respect to $\pi$).

In this context, it is natural to define the theoretical-WGANs (T-WGANs) problem as minimizing over $\Theta$ the Wasserstein distance between $\mu^{\star}$ and the $\mu_{\theta}$'s, i.e.,
\begin{equation}
\label{eq:theoretical_wgans}
    \underset{\theta \in \Theta}{\inf} \ W_1(\mu^\star, \mu_\theta)
    = \underset{\theta \in \Theta}{\inf} \ \underset{f \in {\text{Lip}}_1}{\sup}  | \mathds{E}_{\mu^\star} f - \mathds{E}_{\mu_\theta} f |.
\end{equation}
In practice, however, one does not have access to the class of $1$-Lipschitz functions, which cannot be parameterized. Therefore, following \citet{arjovsky2017wasserstein}, the class ${\text{Lip}}_1$ is restricted to the smaller but parametric set of discriminators $\mathscr{D}=\{D_{\alpha}:\alpha \in \Lambda\}$ (it is a subset of ${\text{Lip}}_1$, by Lemma \ref{lem:uniformly_lipschitz_neural_nets}), and this defines the actual WGANs problem:
\begin{equation}\label{eq:wgans}
    \underset{\theta \in \Theta}{\inf} \ \underset{\alpha \in \Lambda}{\sup} \  | \mathds{E}_{\mu^\star} D_\alpha - \mathds{E}_{\mu_\theta} D_\alpha |.
\end{equation}
Problem \eqref{eq:wgans} is the Wasserstein counterpart of problem \eqref{standard_GANs}.
Provided Assumption \ref{ass:compactness} is satisfied, $\mathscr{D} \subseteq \text{Lip}_1$, and the IPM \citep{IPMsMuller} $d_{\mathscr{D}}$ is defined for $(\mu, \nu) \in P_1(E)^2$ by
\begin{equation} \label{eq:IPMs}
    d_{\mathscr{D}} (\mu, \nu)= \underset{f \in  \mathscr{D}}{\sup}  \ |\mathds{E}_{\mu} f - \mathds{E}_{\nu} f|.
\end{equation}
With this notation, $d_{{\text{Lip}}_1} = W_1$ and
problems \eqref{eq:theoretical_wgans} and \eqref{eq:wgans} can be rewritten as the minimization over $\Theta$ of, respectively, $d_{{\text{Lip}}_1}(\mu^\star, \mu_\theta) $ and $d_{\mathscr{D}}(\mu^\star, \mu_\theta)$. So,
\begin{equation*}
    \text{T-WGANs:} \ \underset{\theta \in \Theta}{\inf} \ d_{{\text{Lip}}_1}(\mu^\star, \mu_\theta) \quad \text{and} \quad
    \text{WGANs:} \ \underset{\theta \in \Theta}{\inf} \  d_{\mathscr{D}}(\mu^\star, \mu_\theta).
\end{equation*}
Similar objectives have been proposed in the literature, in particular neural net distances \citep{Arora0LMZ17} and adversarial divergences \citep{LiBoCh07}. These two general approaches include {f-GANs} \citep{GANs, NoCsTo16}, but also WGANs \citep{arjovsky2017wasserstein}, MMD-GANs \citep{li2017mmd}, and  energy-based GANs \citep{zhao2016energy}. Using the terminology of \citet{Arora0LMZ17}, $d_{\mathscr{D}}$ is called a neural IPM. If the theoretical properties of the Wasserstein distance $d_{\text{Lip}_1}$ have been largely studied \citep[e.g.,][]{villani2008optimal}, the story is different for neural IPMs. This is why our next subsection is devoted to the properties of $d_{\mathscr{D}}$.
\subsection{Some properties of the neural IPM $d_{\mathscr D}$}
The study of the neural IPM $d_{\mathscr D}$ is essential to assess the driving forces of WGANs architectures. Let us first recall that a mapping $\ell: P_1(E) \times P_1(E) \to [0, \infty)$ is a metric if it satisfies the following three requirements:
\begin{enumerate}[$(i)$]
    \item $\ell(\mu, \nu) = 0 \iff \mu = \nu$ (discriminative property)
    \item  $\ell(\mu, \nu) = \ell(\nu, \mu)$ (symmetry)
    \item  $\ell(\mu, \nu) \leqslant \ell(\mu, \pi) + \ell(\pi, \nu)$ (triangle inequality).
\end{enumerate}
If $(i)$ is replaced by the weaker requirement $\ell(\mu, \mu)= 0$ for all $\mu \in P_1(E)$, then one speaks of a pseudometric. Furthermore, the (pseudo)metric $\ell$ is said to metrize weak convergence in $P_1(E)$ \citep{villani2008optimal} if, for all sequences $(\mu_k)$ in $P_1(E)$ and all $\mu$ in $P_1(E)$, one has $\ell(\mu,\mu_k) \to 0 \iff \mu_k \text{ converges weakly to } \mu$ in $P_1(E)$ as $k \to \infty$. According to \citet[][Theorem 6.8]{villani2008optimal}, $d_{{\text{Lip}}_1}$ is a metric that metrizes weak convergence in $P_1(E)$.

As far as $d_{\mathscr D}$ is concerned, it is clearly a pseudometric on $P_1(E)$ as soon as Assumption \ref{ass:compactness} is satisfied. Moreover, an elementary application of \citet[][Lemma 9.3.2]{dudley_2002} shows that if $\text{span}(\mathscr{D})$ (with $\text{span}(\mathscr{D}) = \{ \gamma_0 + \sum_{i=1}^{n} \gamma_i D_i: \gamma_i \in \mathds{R}, D_i \in \mathscr {D}, n \in \mathds{N}\}$) is dense in $C_b(E)$, then $d_{\mathscr{D}}$ is a metric on $P_1(E)$, which, in addition, metrizes weak convergence. As in \citet[][]{zhang2018discrimination}, Dudley's result can be exploited in the case where the space $E$ is compact to prove that, whenever $\mathscr D$ is of the form \eqref{eq:def_discriminators}, $d_{\mathscr{D}}$ is a metric metrizing weak convergence. However, establishing the discriminative property of the pseudometric $d_{\mathscr{D}}$ turns out to be more challenging without an assumption of compactness on $E$, as is the case in the present study. Our result is encapsulated in the following proposition.
\begin{proposition}\label{cor:neural_distance}
    Assume that Assumption \ref{ass:compactness} is satisfied. Then there exists a discriminator of the form \eqref{eq:def_discriminators} (i.e., a depth $q$ and widths $v_1, \hdots, v_{q-1}$) such that $d_{\mathscr{D}}$ is a metric on $\mathscr{P} \cup \{\mu^\star\}$. In addition, $d_{\mathscr{D}}$ metrizes weak convergence in $\mathscr{P} \cup \{\mu^\star\}$.
\end{proposition}

Standard universal approximation theorems \citep{cybenko1989approximation, hornik1989multilayer, hornik1991approximation} state the density of neural networks in the family of continuous functions defined on compact sets but do not guarantee that the approximator respects a Lipschitz constraint. The proof of Proposition \ref{cor:neural_distance} uses the fact that, under Assumption \ref{ass:compactness}, neural networks of the form \eqref{eq:def_discriminators} are dense in the space of Lipschitz continuous functions on compact sets, as revealed by \citet{Anil2018SortingOL}.

We deduce from Proposition \ref{cor:neural_distance} that, under Assumption \ref{ass:compactness}, provided enough capacity, the pseudometric $d_{\mathscr{D}}$ can be topologically equivalent to $d_{{\text{Lip}}_1}$ on $\mathscr{P} \cup \{\mu^\star\}$, i.e., the convergent sequences in $(\mathscr{P} \cup \{\mu^\star\}, d_{\mathscr{D}})$ are the same as the convergent sequences in $(\mathscr{P} \cup \{\mu^\star\}, d_{{\text{Lip}}_1})$ with the same limit---see \citet[][Corollary 13.1.3]{osearcoid2006metric}. We are now ready to discuss some optimality properties of the T-WGANs and WGANs problems, i.e., conditions under which the infimum in $\theta \in \Theta$ and the supremum in $\alpha \in \Lambda$ are reached.
\subsection{Optimality properties}
Recall that for T-WGANs, we minimize over $\Theta$ the distance
\begin{equation*}
d_{\text{Lip}_1} (\mu^{\star}, \mu_{\theta})= \underset{f \in  \text{Lip}_1}{\sup}  \ |\mathds{E}_{\mu^{\star}} f - \mathds{E}_{\mu_{\theta}} f|,
\end{equation*}
whereas for WGANs, we use
\begin{equation*}
d_{\mathscr{D}} (\mu^{\star}, \mu_{\theta})= \underset{\alpha \in  \Lambda}{\sup}  \ |\mathds{E}_{\mu^{\star}} D_{\alpha} - \mathds{E}_{\mu_{\theta}} D_{\alpha}|.
\end{equation*}
A first natural question is to know whether for a fixed generator parameter $\theta \in \Theta$, there exists a $1$-Lipschitz function (respectively, a discriminative function) that achieves the supremum in $d_{\text{Lip}_1}(\mu^{\star}, \mu_{\theta})$ (respectively, in $d_{\mathscr D}(\mu^{\star}, \mu_{\theta})$) over all $f\in \text{Lip}_1$ (respectively, all $\alpha \in \Lambda$). For T-WGANs, \citet[Theorem 5.9]{villani2008optimal} guarantees that the maximum exists, i.e.,
\begin{equation}\label{fstar}
\{ f \in \text{Lip}_1 : | \mathds{E}_{\mu^\star} f - \mathds{E}_{\mu_\theta} f  | = d_{\text{Lip}_1}(\mu^\star, \mu_\theta) \}  \neq \varnothing.
\end{equation}
For WGANs, we have the following:
\begin{lemma}\label{lem:d_star_not_empty}
   Assume that Assumption \ref{ass:compactness} is satisfied.  Then, for all $\theta \in \Theta$,
    \begin{equation*}
       \{ \alpha \in \Lambda : | \mathds{E}_{\mu^\star} D_\alpha - \mathds{E}_{\mu_\theta} D_\alpha| = d_{\mathscr {D}}(\mu^\star, \mu_\theta) \}  \neq \varnothing.
    \end{equation*}
 \end{lemma}

Thus, provided Assumption \ref{ass:compactness} is verified, the supremum in $\alpha$ in the neural IPM $d_{\mathscr D}$ is always reached. A similar result is proved by \citet{biau2018some} in the case of standard GANs.

We now turn to analyzing the existence of the infimum in $\theta$ in the minimization over $\Theta$ of $d_{\text{Lip}_1} (\mu^{\star}, \mu_{\theta})$ and $d_{\mathscr{D}} (\mu^{\star}, \mu_{\theta})$. Since the optimization scheme is performed over the parameter set $\Theta$, it is worth considering the following two functions:
\begin{align*}\label{eq:xhi_functions}
    \xi_{{\text{Lip}}_1} : \Theta &\to \mathds{R} &\text{and} \qquad \qquad \xi_{\mathscr{D}} : \Theta &\to \mathds{R} \nonumber \\
    \theta &\mapsto d_{{\text{Lip}}_1}(\mu^\star, \mu_\theta)  & \theta &\mapsto d_{\mathscr{D}}(\mu^\star, \mu_\theta).
\end{align*}
\begin{theorem}\label{th:continuity}
    Assume that Assumption \ref{ass:compactness} is satisfied. Then $\xi_{{\emph{Lip}}_1}$ and $\xi_{{\mathscr{D}}}$ are Lipschitz continuous on $\Theta$, and the Lipschitz constant of $\xi_{{\mathscr{D}}}$ is independent of $\mathscr{D}$.
\end{theorem}

Theorem \ref{th:continuity} extends \citet[][Theorem 1]{arjovsky2017wasserstein}, which states that $d_{\mathscr {D}}$ is locally Lipschitz continuous under the additional assumption that $E$ is compact. In contrast, there is no compactness hypothesis in Theorem \ref{th:continuity} and the Lipschitz property is global. The lipschitzness of the function $\xi_{\mathscr D}$ is an interesting property of WGANS, in line with many recent empirical works that have shown that gradient-based regularization techniques are efficient for stabilizing the training of GANs and preventing mode collapse \citep{kodali2017convergence, roth2017stabilizing, spectral_normGANs, petzka2018regularization}.

In the sequel, we let $\Theta^\star$ and $\bar{\Theta}$ be the sets of optimal parameters, defined by
\begin{equation*} \Theta^\star =  \underset{\theta \in \Theta}{\argmin} \ d_{{\rm Lip}_1}(\mu^\star, \mu_\theta) \quad \text{and} \quad \bar{\Theta} =\underset{\theta \in \Theta}{\argmin} \ d_{\mathscr{D}}(\mu^\star, \mu_\theta).\end{equation*}
An immediate but useful corollary of Theorem \ref{th:continuity} is as follows:
\begin{corollary}\label{cor:min_reached}
    Assume that Assumption \ref{ass:compactness} is satisfied. Then $\Theta^\star$ and $\bar{\Theta}$ are non empty.
\end{corollary}

Thus, any $\theta^\star \in \Theta^\star$ (respectively, any $\bar{\theta} \in \bar{\Theta}$) is an optimal parameter for the T-WGANs (respectively, the WGANs) problem. Note however that, without further restrictive assumptions on the models, we cannot ensure that $\Theta^\star$ or $\bar{\Theta}$ are reduced to singletons.
\section{Optimization properties}\label{section:approximation_properties}
We are interested in this section in the error made when minimizing over $\Theta$ the pseudometric $d_{\mathscr{D}}(\mu^{\star},\mu_{{\theta}})$ (WGANs problem) instead of $d_{\text{Lip}_1}(\mu^{\star},\mu_{{\theta}})$ (T-WGANs problem). This optimization error is represented by the difference
\begin{equation*}
    \varepsilon_{\text{optim}} = \underset{\bar{\theta} \in \bar{\Theta}}{\sup} \  d_{\text{Lip}_1}(\mu^\star, \mu_{\bar{\theta}}) - \underset{\theta \in \Theta}{\inf} \ d_{\text{Lip}_1}(\mu^\star, \mu_{\theta}).
\end{equation*}
It is worth pointing out that we take the supremum over all $\bar{\theta} \in \bar{\Theta}$ since there is no guarantee that two distinct elements $\bar{\theta}_{1}$ and $\bar{\theta}_{2}$ of $\bar{\Theta}$ lead to the same distances $d_{\text{Lip}_1}(\mu^\star, \mu_{\bar{\theta}_{1}})$ and $d_{\text{Lip}_1}(\mu^\star, \mu_{\bar{\theta}_{2}})$. The quantity $\varepsilon_{\text{optim}}$ captures the largest discrepancy between the scores achieved by distributions solving the WGANs problem and the scores of distributions solving the T-WGANs problem. We emphasize that the scores are quantified by the Wasserstein distance $d_{\text{Lip}_1}$, which is the natural metric associated with the problem. We note in particular that $\varepsilon_{\text{optim}} \geqslant 0$. A natural question is whether we can upper bound the difference and obtain some control of $\varepsilon_{\text{optim}}$.
\subsection{Approximating $d_{\text{Lip}_1}$ with $d_{\mathscr{D}}$}\label{sec:approximation_properties}
As a warm-up, we observe that in the simple but unrealistic case where $\mu^{\star} \in \mathscr{P}$, provided Assumption \ref{ass:compactness} is satisfied and the neural IPM $d_{\mathscr{D}}$ is a metric on $\mathscr{P}$ (see Proposition \ref{cor:neural_distance}), then $\Theta^\star = \bar{\Theta}$ and $\varepsilon_{\text{optim}}=0$. However, in the high-dimensional context of WGANs, the parametric class of distributions $\mathscr{P}$ is likely to be ``far'' from the true distribution $\mu^{\star}$. This phenomenon is thoroughly discussed in \citet[][Lemma~2 and Lemma~3]{ArBo17} and is often referred to as dimensional misspecification \citep{roth2017stabilizing}.

From now on, we place ourselves in the general setting where we have no information on whether the true distribution belongs to $\mathscr{P}$, and start with the following simple observation. Assume that Assumption \ref{ass:compactness} is satisfied. Then, clearly, since \(\mathscr D \subseteq \text{Lip}_1\),
\begin{equation}\label{treize}
    \underset{\theta \in \Theta}{\inf} \  d_{\mathscr{D}}(\mu^\star, \mu_{{\theta}}) \leqslant \underset{\theta \in \Theta}{\inf} \  d_{\text{Lip}_1}(\mu^\star, \mu_{\theta}).
\end{equation}
Inequality \eqref{treize} is useful to upper bound $\varepsilon_{\text{optim}}$. Indeed,
\begin{align}
    0 \leqslant \varepsilon_{\text{optim}} &= \underset{\bar{\theta} \in \bar{\Theta}}{\sup} \ d_{\text{Lip}_1}(\mu^\star, \mu_{\bar{\theta}}) - \underset{\theta \in \Theta}{\inf} \ d_{\text{Lip}_1}(\mu^\star, \mu_{\theta}) \nonumber \\
    & \leqslant \underset{\bar{\theta} \in \bar{\Theta}}{\sup} \ d_{\text{Lip}_1}(\mu^\star, \mu_{\bar{\theta}}) - \underset{\theta \in \Theta}{\inf} \ d_{\mathscr{D}}(\mu^\star, \mu_{\theta}) \nonumber\\
    & = \underset{\bar{\theta} \in \bar{\Theta}}{\sup} \ \big[ d_{\text{Lip}_1}(\mu^\star, \mu_{\bar{\theta}}) - d_{\mathscr{D}}(\mu^\star, \mu_{\bar{\theta}})\big] \nonumber\\
    & \quad \mbox{(since $\underset{\theta \in \Theta}{\inf} \ d_{\mathscr{D}}(\mu^\star, \mu_{\theta}) = d_{\mathscr{D}}(\mu^\star, \mu_{\bar \theta})$ for all $\bar \theta \in \bar \Theta$)} \nonumber\\
    & \leqslant T_{\mathscr{P}}(\text{Lip}_1, \mathscr{D}) \label{eq:eps_optim_inequality},
\end{align}
where, by definition,
\begin{equation}\label{eq:bounding_eps_optim}
    T_{\mathscr{P}}(\text{Lip}_1, \mathscr{D}) = \underset{\theta \in \Theta}{\sup} \big[d_{\text{Lip}_1}(\mu^\star, \mu_\theta) - d_{\mathscr{D}}(\mu^\star, \mu_\theta)\big]
\end{equation}
is the maximum difference in distances on the set of candidate probability distributions in $\mathscr{P}$.
Note, since $\Theta$ is compact (by Assumption \ref{ass:compactness}) and $\xi_{\text{Lip}_1}$ and $\xi_{\mathscr{D}}$ are Lipschitz continuous (by Theorem \ref{th:continuity}), that $T_{\mathscr{P}}(\text{Lip}_1, \mathscr{D}) < \infty$. Thus, the loss in performance when comparing T-WGANs and WGANs can be upper-bounded by the maximum difference over $\mathscr{P}$ between the Wasserstein distance and $d_{\mathscr{D}}$.

Observe that when the class of discriminative functions is increased (say $\mathscr{D} \subset \mathscr{D}'$) while keeping the generator fixed, then the bound \eqref{eq:bounding_eps_optim} gets reduced since $d_{\mathscr D}(\mu^\star,\cdot) \leqslant d_{\mathscr D'}(\mu^\star, \cdot)$. Similarly, when increasing the class of generative distributions (say $\mathscr{P} \subset \mathscr{P}'$) with a fixed discriminator, then the bound gets bigger, i.e., $T_{\mathscr{P}}(\text{Lip}_1, \mathscr{D}) \leqslant T_{\mathscr{P}'}(\text{Lip}_1, \mathscr{D})$. It is important to note that the conditions $\mathscr{D} \subset \mathscr{D}'$ and/or $\mathscr{P} \subset \mathscr{P}'$ are easily satisfied for classes of functions parameterized with neural networks using either rectifier or GroupSort activation functions, just by increasing the width and/or the depth of the networks.

Our next theorem states that, as long as the distributions of $\mathscr{P}$ are generated by neural networks with bounded parameters (Assumption \ref{ass:compactness}), then one can control $T_{\mathscr{P}}(\text{Lip}_1, \mathscr{D})$.
\begin{theorem}\label{th:approx_properties}
    Assume that Assumption \ref{ass:compactness} is satisfied. Then, for all $\varepsilon >0$, there exists a discriminator $\mathscr{D}$ of the form \eqref{eq:def_discriminators} such that
    \begin{equation*}
        0 \leqslant \varepsilon_{\emph{optim}} \leqslant T_{\mathscr{P}}(\emph{Lip}_1, \mathscr{D}) \leqslant c \varepsilon,
    \end{equation*}
    where $c>0$ is a constant independent from $\varepsilon$.
\end{theorem}

Theorem \ref{th:approx_properties} is important because it shows that for any collection of generative distributions $\mathscr{P}$ and any approximation threshold $\varepsilon$, one can find a discriminator such that the loss in performance $\varepsilon_{\text{optim}}$ is (at most) of the order of $\varepsilon$. In other words, there exists $\mathscr{D}$ of the form \eqref{eq:def_discriminators} such that $T_{\mathscr{P}}(\text{Lip}_1, \mathscr{D})$ is arbitrarily small. We note however that Theorem \ref{th:approx_properties} is an existence theorem that does not give any particular information on the depth and/or the width of the neural networks in $\mathscr D$. The key argument to prove Theorem \ref{th:approx_properties} is \citet[][Theorem 3]{Anil2018SortingOL}, which states that the set of Lipschitz neural networks are dense in the set of Lipschitz continuous functions on a compact space.
\subsection{Equivalence properties} \label{section:equivalence_properties}
The quantity $T_{\mathscr{P}}(\text{Lip}_1, \mathscr{D})$ is of limited practical interest, as it involves a supremum over all $\theta \in \Theta$. Moreover, another caveat is that the definition of $\varepsilon_{\text{optim}}$ assumes that one has access to $\bar{\Theta}$. Therefore, our next goal is to enrich Theorem~\ref{th:approx_properties} by taking into account the fact that numerical procedures do not reach \(\bar \theta \in \bar \Theta\) but rather an \(\varepsilon\)-approximation of it.

One way to approach the problem is to look for another form of equivalence between $d_{\text{Lip}_1}$ and $d_{\mathscr{D}}$. As one is optimizing $d_{\mathscr{D}}$ instead of $d_{\text{Lip}_1}$, we would ideally like that the two IPMs behave ``similarly'', in the sense that minimizing $d_{\mathscr{D}}$ leads to a solution that is still close to the true distribution with respect to $d_{\text{Lip}_1}$. Assuming that Assumption \ref{ass:compactness} is satisfied, we let, for any $\mu \in P_1(E)$ and $\varepsilon >0$, $\mathscr{M}_{\ell}(\mu, \varepsilon)$ be the set of $\varepsilon$-solutions to the optimization problem of interest, that is the subset of $\Theta$ defined by
\begin{equation*}
    \mathscr{M}_{\ell}(\mu, \varepsilon) = \big\{\theta \in \Theta: \ell(\mu, \mu_\theta) - \underset{\theta \in \Theta}{\inf}\ \ell(\mu, \mu_\theta) \leqslant \varepsilon \big\},
\end{equation*}
with $\ell=d_{\text{Lip}_1}$ or $\ell=d_{\mathscr{D}}$.
\begin{definition}\label{def:substitution}
Let $\varepsilon >0$. We say that $d_{\emph{Lip}_1}$ can be $\varepsilon$-substituted by $d_{\mathscr{D}}$  if there exists $\delta >0$ such that
\begin{equation*}
    \mathscr{M}_{d_\mathscr D}(\mu^\star, \delta) \subseteq \mathscr{M}_{d_{\emph{Lip}_1}}(\mu^\star, \varepsilon).
\end{equation*}
In addition, if $d_{\emph{Lip}_1}$ can be $\varepsilon$-substituted by $d_{\mathscr{D}}$ for all $\varepsilon>0$, we say that $d_{\emph{Lip}_1}$ can be fully substituted by $d_{\mathscr{D}}$.
\end{definition}

The rationale behind this definition is that by minimizing the neural IPM $d_{\mathscr{D}}$ close to optimality, one can be guaranteed to be also close to optimality with respect to the Wasserstein distance $d_{\text{Lip}_1}$. In the sequel, given a metric $d$, the notation $d(x, F)$ denotes the distance of $x$ to the set $F$, that is, $d(x, F)= {\inf}_{f \in F} \ d(x, f)$.
\begin{proposition}\label{prop:substitution}
    Assume that Assumption \ref{ass:compactness} is satisfied. Then, for all $\varepsilon>0$, there exists $\delta>0$ such that, for all $\theta \in \mathscr{M}_{d_\mathscr D}(\mu^\star, \delta)$, one has $d(\theta, \bar{\Theta}) \leqslant \varepsilon$.
\end{proposition}
\begin{corollary}\label{cor:substitution}
    Assume that Assumption \ref{ass:compactness} is satisfied and that $\Theta^\star=\bar{\Theta}$. Then $d_{\emph{Lip}_1}$ can be fully substituted by $d_{\mathscr{D}}$.
\end{corollary}
\begin{proof}
Let $\varepsilon>0$. By Theorem \ref{th:continuity}, we know that the function $\Theta \ni  \theta \mapsto d_{\text{Lip}_1}(\mu^\star, \mu_\theta)$ is Lipschitz continuous. Thus, there exists $\eta >0$ such that, for all $(\theta, \theta') \in \Theta^2$ satisfying $\|\theta-\theta'\| \leqslant \eta$, one has $|d_{\text{Lip}_1}(\mu^\star, \mu_\theta) - d_{\text{Lip}_1}(\mu^\star, \mu_{\theta'})| \leqslant \varepsilon$. Besides, using Proposition \ref{prop:substitution}, there exists $\delta>0$ such that, for all $\theta \in \mathscr{M}_{d_\mathscr D}(\mu^\star, \delta)$, one has $d(\theta, \bar{\Theta}) \leqslant \eta$.

Now, let $\theta \in \mathscr{M}_{d_\mathscr D}(\mu^\star, \delta)$. Since $d(\theta, \bar{\Theta}) \leqslant \eta$ and $\bar{\Theta} = \Theta^\star$, we have $d(\theta, \Theta^\star) \leqslant \eta$. Consequently, $|d_{\text{Lip}_1}(\mu^\star, \mu_\theta) - \inf_{\theta \in \Theta} \ d_{\text{Lip}_1}(\mu^\star, \mu_{\theta})| \leqslant \varepsilon$, and so, $\theta \in \mathscr{M}_{d_{\text{Lip}_1}}(\mu^\star, \varepsilon)$.
\end{proof}

Corollary \ref{cor:substitution} is interesting insofar as when both $d_{\mathscr{D}}$ and $d_{\text{Lip}_1}$ have the same minimizers over $\Theta$, then minimizing one close to optimality is the same as minimizing the other. The requirement $\Theta^\star=\bar{\Theta}$ can be relaxed by leveraging what has been studied in the previous subsection about $T_{\mathscr{P}}(\text{Lip}_1, \mathscr{D})$.
\begin{lemma}\label{lem:equivalence_properties_both_distances}
    Assume that Assumption \ref{ass:compactness} is satisfied, and let $\varepsilon >0$. If $T_{\mathscr{P}}(\emph{Lip}_1, \mathscr{D}) \leqslant \varepsilon$, then $d_{\emph{Lip}_1}$ can be $(\varepsilon + \delta)$-substituted by $d_{\mathscr{D}}$ for all $\delta>0$.
\end{lemma}
    \begin{proof}
        Let $\varepsilon >0$, $\delta>0$, and $\theta \in \mathscr{M}_{d_{\mathscr{D}}}(\mu^\star, \delta)$, i.e., $d_{\mathscr{D}}(\mu^\star, \mu_\theta) - \underset{\theta \in \Theta}{\inf} \ d_{\mathscr{D}}(\mu^\star, \mu_{{\theta}}) \leqslant \delta$. We have
        \begin{align*}
        d_{\text{Lip}_1}(\mu^\star, \mu_\theta) - \underset{\theta \in     \Theta}{\inf}\ d_{\text{Lip}_1}(\mu^\star, \mu_{\theta}) & \leqslant d_{\text{Lip}_1}(\mu^\star, \mu_\theta) - \underset{\theta \in \Theta}{\inf} \ d_{\mathscr{D}}(\mu^\star, \mu_{\theta})\\
        &\quad \mbox{(by inequality \eqref{treize})}\\
        & \leqslant d_{\text{Lip}_1}(\mu^\star, \mu_\theta) - d_{\mathscr{D}}(\mu^\star, \mu_\theta) + \delta\\
        & \leqslant T_{\mathscr{P}}(\text{Lip}_1, \mathscr{D})+\delta \leqslant \varepsilon + \delta.
        \end{align*}
    \end{proof}

Lemma \ref{lem:equivalence_properties_both_distances} stresses the importance of $T_{\mathscr{P}}(\text{Lip}_1, \mathscr{D})$ in the performance of WGANs. Indeed, the smaller $T_{\mathscr{P}}(\text{Lip}_1, \mathscr{D})$, the closer we will be to optimality after training. Moving on, to derive sufficient conditions under which $d_{\text{Lip}_1}$ can be substituted by $d_{\mathscr{D}}$ we introduce the following definition:
\begin{definition}\label{def:monotonous_equivalence}
    We say that $d_{\emph{Lip}_1}$ is monotonously equivalent to $d_{\mathscr{D}}$ on $\mathscr{P}$ if there exists a continuously differentiable, strictly increasing function $f: \mathds{R}^+ \to \mathds{R}^+$ and $(a, b) \in (\mathds{R}^{\star}_+)^2$ such that
    \begin{equation*}
        \forall \mu \in \mathscr{P}, \ a  f(d_{\mathscr{D}}(\mu^\star, \mu)) \leqslant d_{\emph{Lip}_1}(\mu^\star, \mu) \leqslant b  f(d_{\mathscr{D}}(\mu^\star, \mu)).
    \end{equation*}
\end{definition}

Here, it is assumed implicitly that $\mathscr{D} \subseteq \text{Lip}_1$. At the end of the subsection, we stress, empirically, that Definition \ref{def:monotonous_equivalence} is easy to check for simple classes of generators. A consequence of this definition is encapsulated in the following lemma.
\begin{lemma}\label{lem:monotonous_equivalence_a_equal_b}
    Assume that Assumption \ref{ass:compactness} is satisfied, and that $d_{\emph{Lip}_1}$ and $d_{\mathscr{D}}$ are monotonously equivalent on $\mathscr{P}$ with $a=b$ (that is, $d_{\emph{Lip}_1} = f\circ d_{\mathscr D}$). Then $\Theta^\star = \bar{\Theta}$ and $d_{\emph{Lip}_1}$ can be fully substituted by $d_{\mathscr{D}}$.
\end{lemma}

To complete Lemma \ref{lem:monotonous_equivalence_a_equal_b}, we now tackle the case $a<b$.
\begin{proposition}\label{prop:monotonous_equivalence}
Assume that Assumption \ref{ass:compactness} is sa\-tis\-fi\-ed, and that $d_{\emph{Lip}_1}$ and $d_{\mathscr{D}}$ are mo\-no\-to\-nous\-ly equivalent on $\mathscr{P}$. Then, for any $\delta \in (0,1)$, $d_{\emph{Lip}_1}$ can be $\varepsilon$-substituted by $d_{\mathscr{D}}$ with $\varepsilon = (b-a) f (\underset{\theta \in \Theta}{\inf} \ d_{\mathscr{D}}(\mu^\star, \mu_{{\theta}})) + O(\delta)$.
\end{proposition}

Proposition~\ref{prop:monotonous_equivalence} states that we can reach \(\varepsilon\)-minimizers of \(d_{\text{Lip}_1}\) by solving the WGANs problem up to a precision sufficiently small, for all \(\varepsilon\) larger than a bias induced by the model \(\mathscr P\) and by the discrepancy between \(d_{\text{Lip}_1}\) and \(d_{\mathscr D}\).

In order to validate Definition \ref{def:monotonous_equivalence}, we slightly depart from the WGANs setting and run a series of small experiments in the simplified setting where both $\mu^{\star}$ and $\mu \in \mathscr{P}$ are bivariate mixtures of independent Gaussian distributions with $K$ components ($K=1, 4, 9, 25$). We consider two classes of discriminators $\{ \mathscr{D}_q: q=2, 5 \}$ of the form \eqref{eq:def_discriminators}, with growing depth $q$ (the width of the hidden layers is kept constant equal to $20$). Our goal is to exemplify the relationship between the distances $d_{\text{Lip}_1}$ and $d_{\mathscr{D}_q}$ by looking whether $d_{\text{Lip}_1}$ is monotonously equivalent to $d_{\mathscr{D}_q}$.

First, for each $K$, we randomly draw $40$ different pairs of distributions $(\mu^\star, \mu)$ among the set of mixtures of bivariate Gaussian densities with $K$ components. Then, for each of these pairs, we compute an approximation of $d_{\text{Lip}_1}$ by averaging the Wasserstein distance between finite samples of size 4096 over 20 runs. This operation is performed using the Python package by \cite{flamary2017pot}. For each pair of distributions, we also calculate the corresponding IPMs $d_{\mathscr{D}_q}(\mu^\star, \mu)$. We finally compare $d_{\text{Lip}_1}$ and $d_{\mathscr{D}_q}$ by approximating their relationship with a parabolic fit. Results are presented in Figure \ref{fig:neural_div_figure}, which depicts in particular the best parabolic fit, and shows the corresponding Least Relative Error (LRE) together with the width $(b-a)$ from Definition \ref{def:monotonous_equivalence}. In order to enforce the discriminator to verify Assumption \ref{ass:compactness}, we use the orthonormalization of \citet{bjorck1971iterative}, as done for example in \citet{Anil2018SortingOL}.

Interestingly, we see that when the class of discriminative functions gets larger (i.e., when $q$ increases), then both metrics start to behave similarly (i.e., the range $(b-a)$ gets thinner), independently of $K$ (Figure \ref{fig:1a} to Figure \ref{fig:1f}). This tends to confirm that $d_{\text{Lip}_1}$ can be considered as monotonously equivalent to $d_{\mathscr{D}_q}$ for $q$ large enough. On the other hand, for a fixed depth $q$, when allowing for more complex distributions, the width $(b-a)$ increases. This is particularly clear in Figure \ref{fig:1g} and Figure \ref{fig:1h}, which show the fits obtained when merging all pairs for $K=1,4,9,25$ (for both $\mu^\star$ and $\mathscr{P}$).

\begin{figure}
    %\vspace{-0.5cm}
    \centering
    \subfloat[$d_{\mathscr{D}_q}$ vs.~$d_{\text{Lip}_1}$, $q=2$, $K=1$.\label{fig:1a}]
    {
        \includegraphics[width=0.36\linewidth]{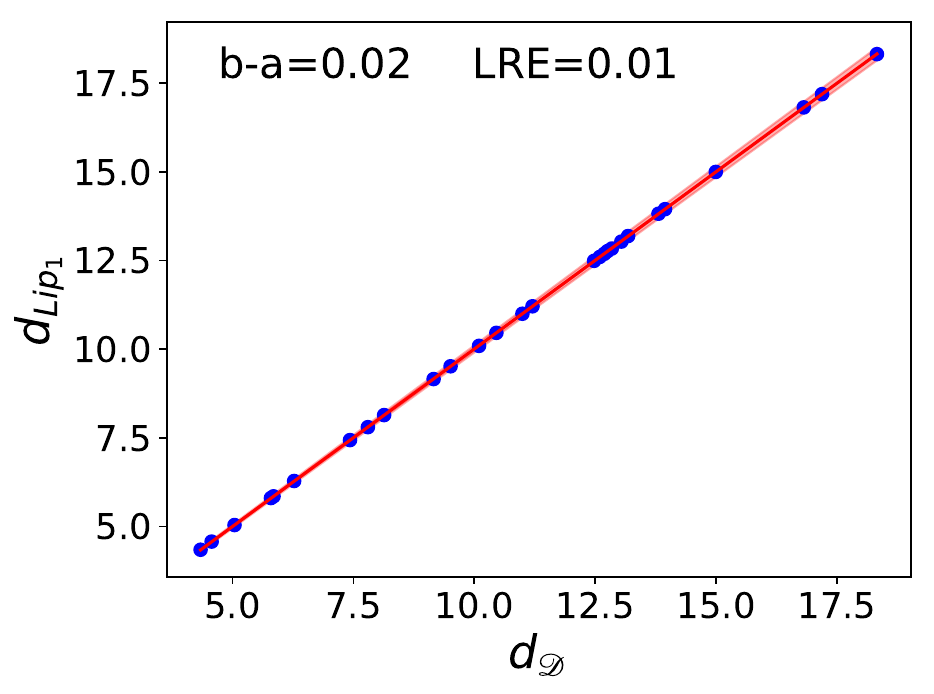}
    }
    \subfloat[$d_{\mathscr{D}_q}$ vs.~$d_{\text{Lip}_1}$, $q=5$, $K=1$.]
    {
        \includegraphics[width=0.35\linewidth]{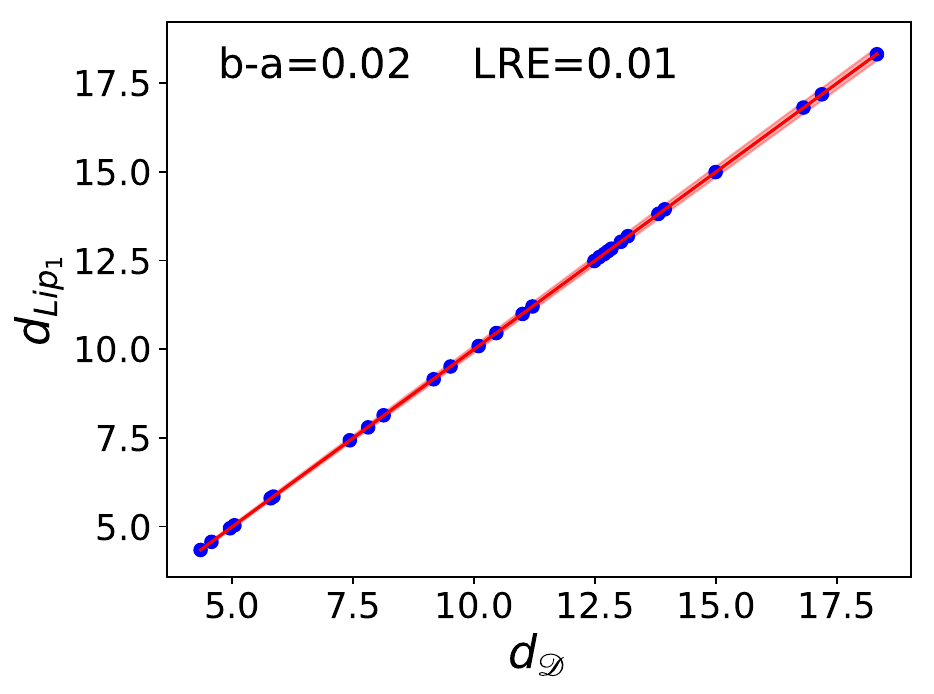}
    }
    \hfill
    \subfloat[$d_{\mathscr{D}_q}$ vs.~$d_{\text{Lip}_1}$, $q=2$, $K=4$.]
    {
        \includegraphics[width=0.35\linewidth]{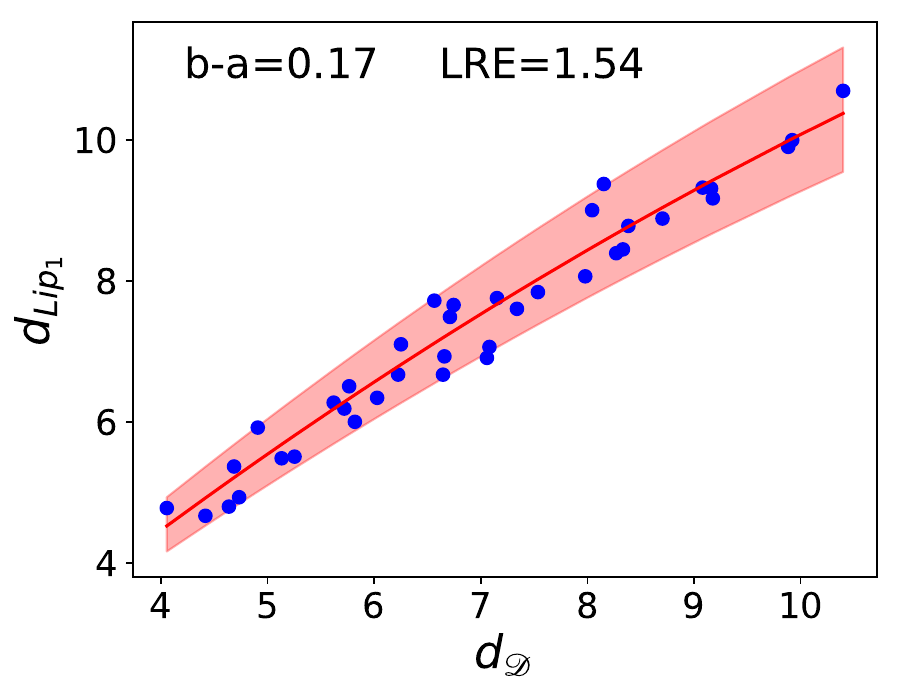}
    }
    \subfloat[$d_{\mathscr{D}_q}$ vs.~$d_{\text{Lip}_1}$, $q=5$, $K=4$.]
    {
        \includegraphics[width=0.35\linewidth]{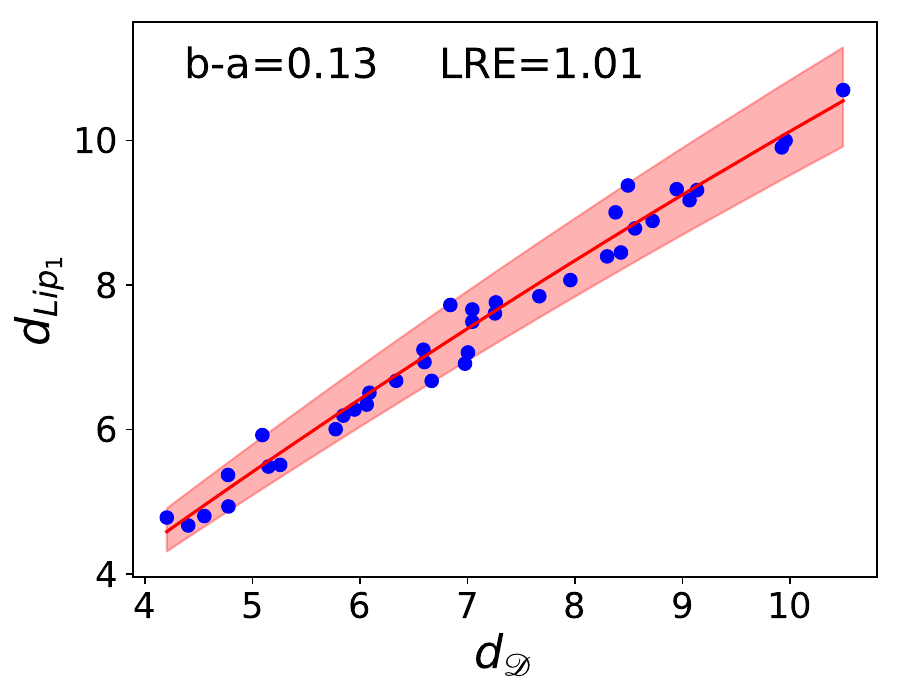}
    }
    \hfill
    \subfloat[$d_{\mathscr{D}_q}$ vs.~$d_{\text{Lip}_1}$, $q=2$, $K=9$.]
    {
        \includegraphics[width=0.35\linewidth]{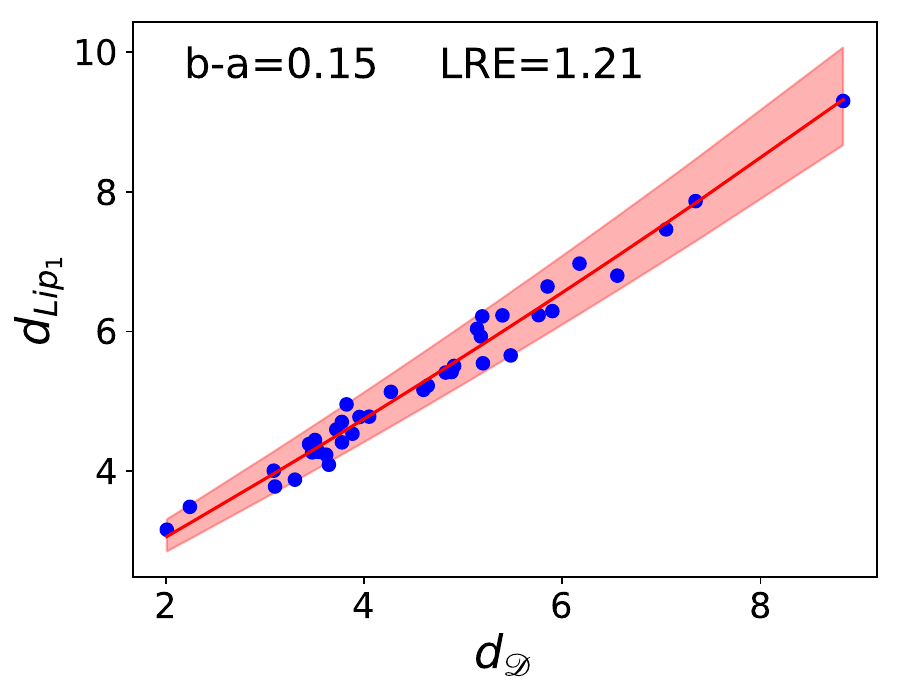}
    }
    \subfloat[$d_{\mathscr{D}_q}$ vs.~$d_{\text{Lip}_1}$, $q=5$, $K=9$.\label{fig:1f}]
    {
        \includegraphics[width=0.35\linewidth]{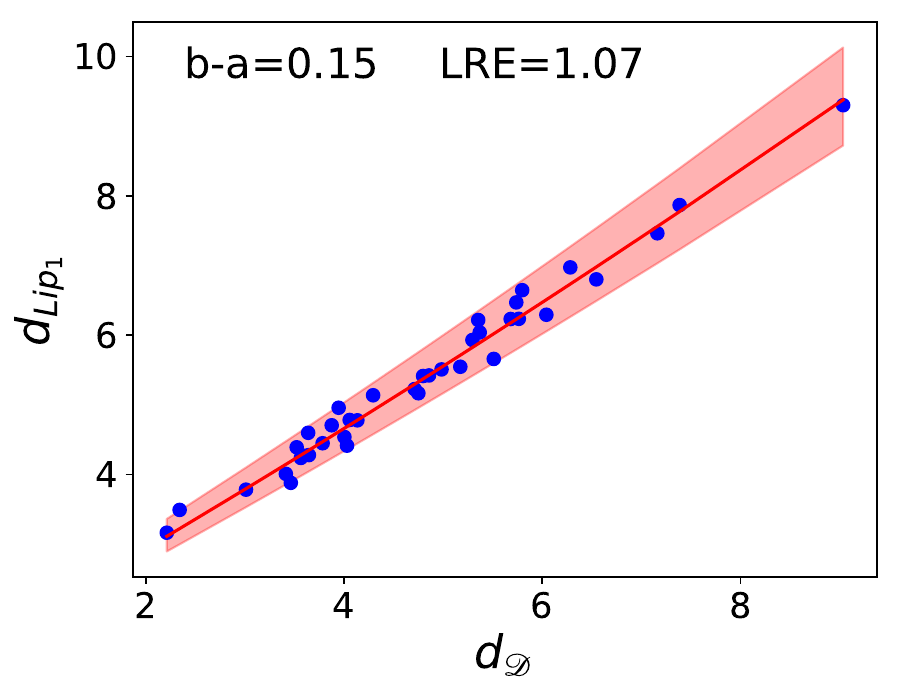}
    }
    \hfill
    \subfloat[$d_{\mathscr{D}_q}$ vs.~$d_{\text{Lip}_1}$, $q=2$, $K=1, 4, 9, 25$.\label{fig:1g}]
    {
        \includegraphics[width=0.35\linewidth]{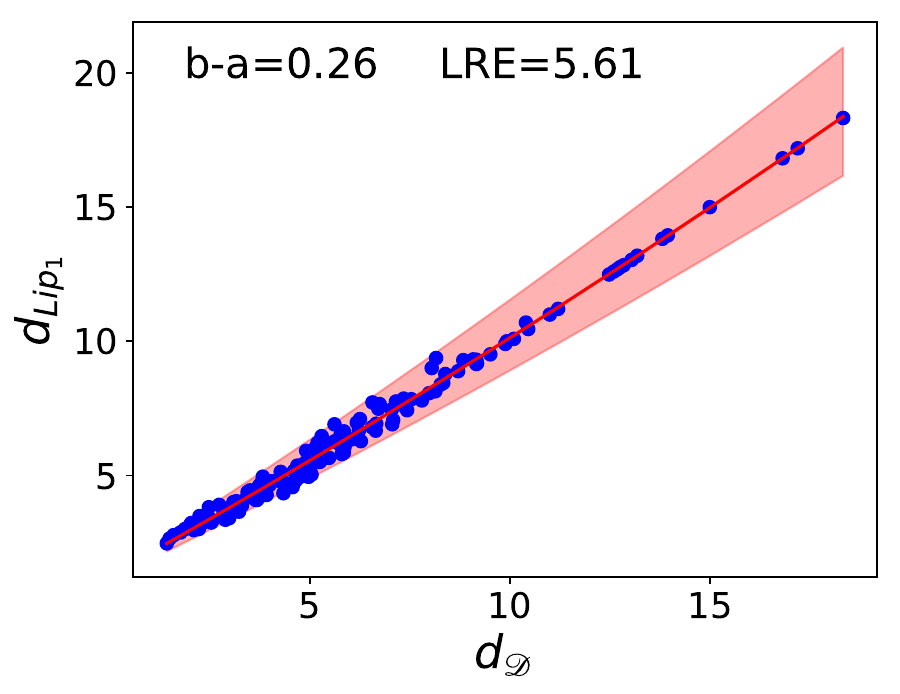}
    }
    \subfloat[$d_{\mathscr{D}_q}$ vs.~$d_{\text{Lip}_1}$, $q=5$, $K=1, 4, 9, 25$.\label{fig:1h}]
    {
        \includegraphics[width=0.35\linewidth]{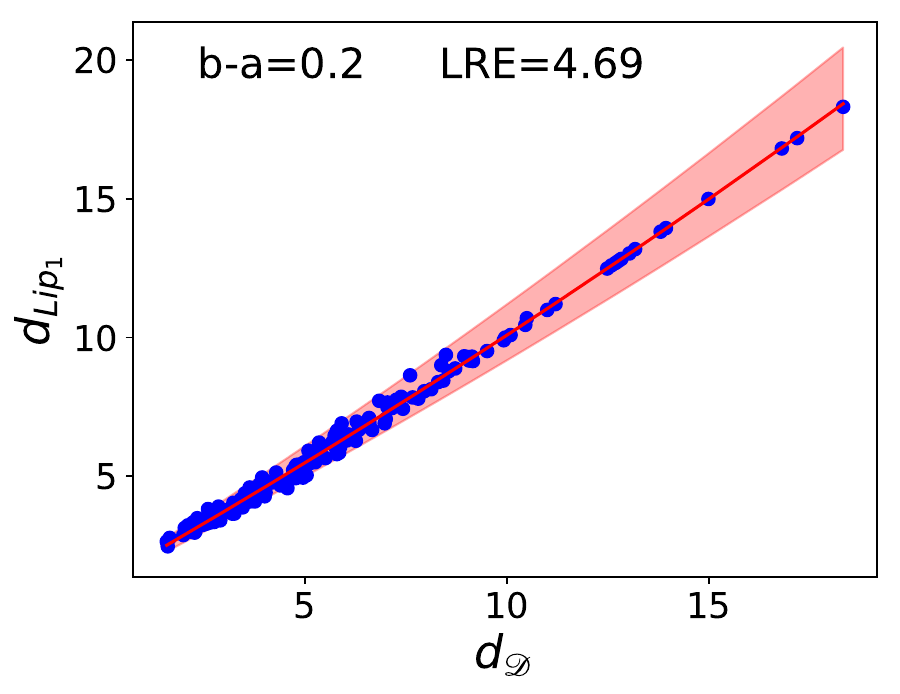}
    }
    \caption{Scatter plots of $40$ pairs of distances simultaneously measured with $d_{\text{Lip}_1}$ and $d_{\mathscr{D}_q}$, for $q=2, 5$ and $K=1, 4, 9, 25$. The red curve is the optimal parabolic fitting and LRE refers to the Least Relative Error. The red zone is the envelope obtained by stretching the optimal curve from $b$ to $a$.
    }
    \label{fig:neural_div_figure}
\end{figure}
These figures illustrate the fact that, for a fixed discriminator, the monotonous equivalence between $d_{\text{Lip}_1}$ and $d_{\mathscr{D}}$ seems to be a more demanding assumption when the class of generative distributions becomes too large.
\subsection{Motivating the use of deep GroupSort neural networks}
The objective of this subsection is to provide some justification for the use of deep GroupSort neural networks in the field of WGANs. This short discussion is motivated by the observation of \citet[][Theorem 1]{Anil2018SortingOL}, who stress that norm-constrained ReLU neural networks are not well-suited for learning non-linear $1$-Lipschitz functions.

The next lemma shows that networks of the form \eqref{eq:def_discriminators}, which use GroupSort activations, can recover any $1$-Lipschitz function belonging to the class $\text{AFF}$ of real-valued affine functions on $E$.
\begin{lemma}\label{lem:ReLU_net_and_affine_functions}
    Let $f: E \to \mathds{R}$ be in $\emph{AFF} \cap \emph{Lip}_1$. Then there exists a neural network of the form \eqref{eq:def_discriminators} verifying Assumption \ref{ass:compactness}, with $q=2$ and $v_1=2$, that represents $f$.
\end{lemma}

Motivated by Lemma \ref{lem:ReLU_net_and_affine_functions}, we show that, in some specific cases, the Wasserstein distance $d_{\text{Lip}_1}$ can be approached by only considering affine functions, thus motivating the use of neural networks of the form \eqref{eq:def_discriminators}. Recall that the support $S_\mu$ of a probability measure $\mu$ is the smallest subset of $\mu$-measure $1$.
\begin{lemma}\label{lem:wasserstein_distance_1d}
    Let $\mu$ and $\nu$ be two probability measures in $P_1(E)$. Assume that $S_\mu$ and $S_\nu$ are one-dimensional disjoint intervals included in the same line. Then $d_{\emph{Lip}_1}(\mu, \nu) = d_{\emph{AFF} \cap \emph{Lip}_1}(\mu, \nu)$.
\end{lemma}

Lemma \ref{lem:wasserstein_distance_1d} is interesting insofar as it describes a specific case where the discriminator can be restricted to affine functions while keeping the identity $d_{\text{Lip}_1} = d_{\mathscr D}$. We consider in the next lemma a slightly more involved setting, where the two distributions $\mu$ and $\nu$ are multivariate Gaussian with the same covariance matrix.
\begin{lemma}\label{lem:gaussian_case}
    Let $(m_1, m_2) \in (\mathds{R}^D)^2$, and let $\Sigma \in \mathscr{M}_{(D,D)}$ be a positive semi-definite matrix. Assume that $\mu$ is Gaussian $\mathscr{N}(m_1, \Sigma)$ and that $\nu$ is Gaussian $\mathscr{N}(m_2, \Sigma)$. Then $d_{\emph{Lip}_1}(\mu, \nu) = d_{\emph{AFF} \cap \emph{Lip}_1}(\mu, \nu)$.
\end{lemma}

Yet, assuming multivariate Gaussian distributions might be too restrictive. Therefore, we now assume that both distributions lay on disjoint compact supports sufficiently distant from one another. Recall that for a set $S\subseteq E$, the diameter of $S$ is $\text{diam}(S) = {\sup}_{(x, y) \in S^2} \ \|x-y\|$, and that the distance between two sets $S$ and $T$ is defined by $d(S, T) = {\inf}_{(x,y) \in S \times T}\ \|x-y\|$.
\begin{proposition}\label{prop:cluster_wasserstein}
   Let $\varepsilon >0$, and let $\mu$ and $\nu$ be two probability measures in $P_1(E)$ with compact convex supports $S_\mu$ and $S_\nu$. Assume that $\max( \emph{diam}(S_\mu), \emph{diam}(S_\nu)) \leqslant \varepsilon d(S_\mu, S_\nu)$. Then
\begin{equation*}d_{\emph{AFF} \cap \emph{Lip}_1}(\mu, \nu) \leqslant d_{\emph{Lip}_1}(\mu, \nu) \leqslant (1+2\varepsilon) d_{\emph{AFF} \cap \emph{Lip}_1}(\mu, \nu).\end{equation*}
\end{proposition}

Observe that in the case where neither $\mu$ nor $\nu$ are Dirac measures, then the assumption of the lemma imposes that $S_\mu \cap S_\nu=\emptyset$. In the context of WGANs, it is highly likely that the generator badly approximates the true distribution $\mu^\star$ at the beginning of training. The setting of Proposition \ref{prop:cluster_wasserstein} is thus interesting insofar as $\mu^\star$ and the generative distribution will most certainly verify the assumption on the diameters at this point in the learning process. However, in the common case where the true distribution lays on disconnected manifolds, the assumptions of the proposition are not valid anymore, and it would therefore be interesting to show a similar result using the broader set of piecewise linear functions on $E$.

As an empirical illustration, consider the synthetic setting where one tries to approximate a bivariate mixture of independent Gaussian distributions with respectively 4 (Figure \ref{fig:densities_4_modes}) and 9 (Figure \ref{fig:densities_9_modes}) modes. As expected, the optimal discriminator takes the form of a piecewise linear function, as illustrated by Figure \ref{fig:heatmap_disc_grad_4modes} and Figure \ref{fig:heatmap_disc_grad_9modes}, which display heatmaps of the discriminator's output. Interestingly, we see that the number of linear regions increases with the number $K$ of components of $\mu^\star$.

\begin{figure}[H]
    \centering
    \subfloat[True distribution $\mu^\star$ (mixture of $K=4$ bivariate Gaussian densities, green circles) and 2000 data points sampled from the generator $\mu_{\bar{\theta}}$ (blue dots). \label{fig:densities_4_modes}]
    {
        \includegraphics[width=0.4\linewidth]{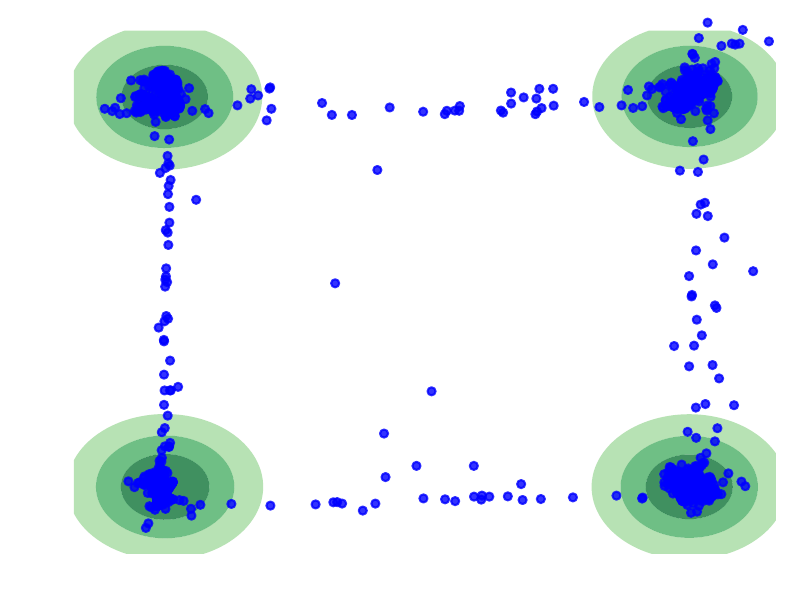}
    }\hfill
    \subfloat[Heatmap of the discriminator's output on a mixture of $K=4$ bivariate Gaussian densities. \label{fig:heatmap_disc_grad_4modes}]
    {
        \includegraphics[width=0.4\linewidth]{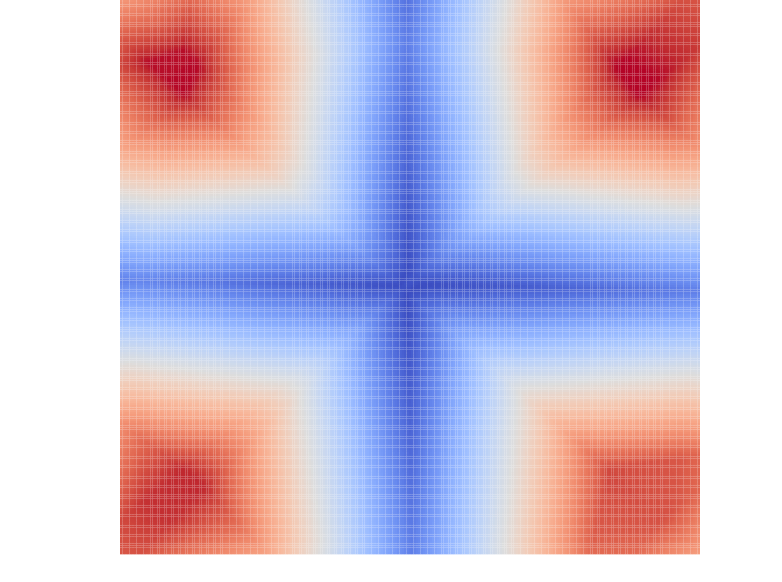}
    }\vfill
    \subfloat[True distribution $\mu^\star$ (mixture of $K=9$ bivariate Gaussian densities, green circles) and 2000 data points sampled from the generator $\mu_{\bar{\theta}}$ (blue dots). \label{fig:densities_9_modes}]
    {
        \includegraphics[width=0.4\linewidth]{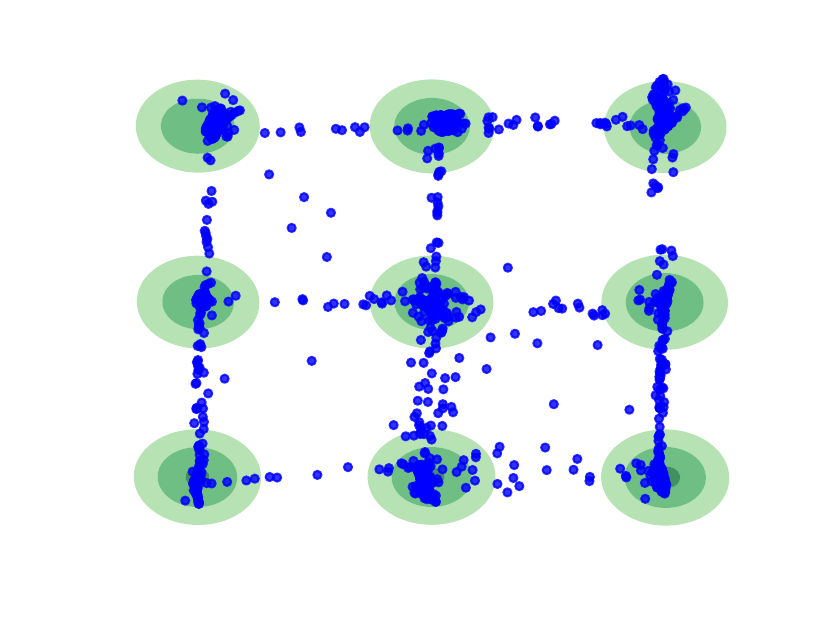}
    }\hfill
    \subfloat[Heatmap of the discriminator's output on a mixture of $K=9$ bivariate Gaussian densities. \label{fig:heatmap_disc_grad_9modes}]
    {
        \includegraphics[width=0.4\linewidth]{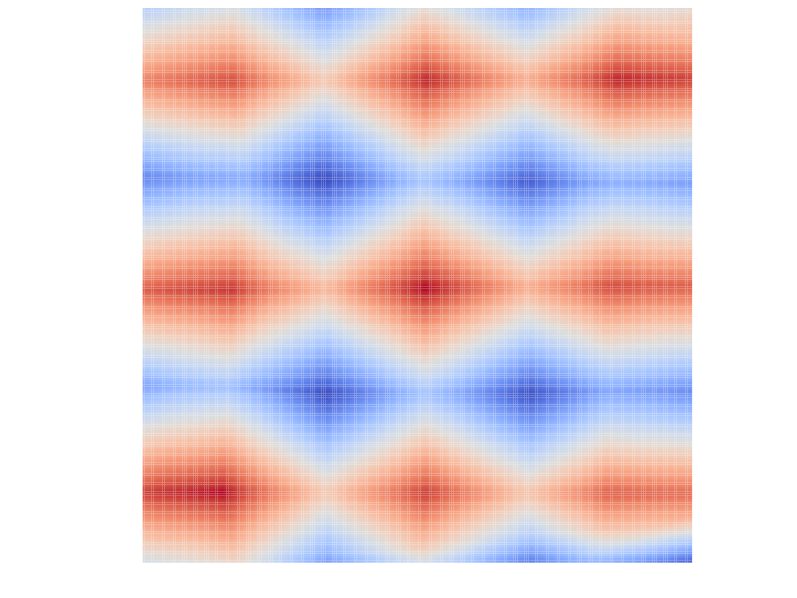}
    }
    \caption{Illustration of the usefulness of GroupSort neural networks when dealing with the learning of mixtures of Gaussian distributions. In both cases, we have $p=q=3$.}
\end{figure}
These empirical results stress that when $\mu^\star$ gets more complex, if the discriminator ought to correctly approximate the Wasserstein distance, then it should parameterize piecewise linear functions with growing numbers of regions. While we enlighten properties of Groupsort networks, many recent theoretical works have been studying the number of regions of deep ReLU neural networks \citep{pascanu2013number, montufar2014number, arora2018understanding, serra2018bounding}. In particular, \citet[][Theorem 5]{montufar2014number} states that the number of linear regions of deep models grows exponentially with the depth and polynomially with the width. This, along with our observations, is an interesting avenue to choose the architecture of the discriminator.
\section{Asymptotic properties} \label{section:asymptotic_properties}
In practice, one never has access to the distribution $\mu^\star$ but rather to a finite collection of i.i.d.~observations $X_1, \hdots, X_n$ distributed according to $\mu^\star$. Thus, for the remainder of the article, we let ${\mu}_n$ be the empirical measure based on $X_1, \hdots, X_n$, that is, for any Borel subset $A$ of $E$, $\mu_n(A)=\frac{1}{n}\sum_{i=1}^{n}\mathds 1_{X_i \in A}$. With this notation, the empirical counterpart of the WGANs problem is naturally defined as minimizing over $\Theta$ the quantity $d_{\mathscr{D}}(\mu_n,\mu_\theta)$. Equivalently, we seek to solve the following optimization problem:
\begin{equation}\label{empirical_approximated_wgans}
    \underset{\theta \in \Theta}{\inf} \ d_{\mathscr{D}}(\mu_n,\mu_\theta)=\underset{\theta \in \Theta}{\inf} \ \underset{\alpha \in \Lambda}{\sup} \Big[ \frac{1}{n} \sum_{i=1}^n D_\alpha(X_i) - \mathds{E} D_{\alpha}(G_{\theta}(Z))\Big].
\end{equation}
Assuming that Assumption \ref{ass:compactness} is satisfied, we have, as in Corollary \ref{cor:min_reached}, that the infimum in \eqref{empirical_approximated_wgans} is reached. We therefore consider the set of empirical optimal parameters
\begin{equation*}
    \hat{\Theta}_n = \underset{\theta \in \Theta}{\argmin} \ d_{\mathscr{D}}(\mu_n,\mu_\theta),
\end{equation*}
and let $\hat{\theta}_n$ be a specific element of $\hat{\Theta}_n$ (note that the choice of $\hat{\theta}_n$ has no impact on the value of the minimum). We note that $\hat{\Theta}_n$ (respectively, $\hat \theta_n$) is the empirical counterpart of $\bar{\Theta}$ (respectively, $\bar \theta$). Section \ref{section:approximation_properties} was mainly devoted to the analysis of the difference $\varepsilon_{\text{optim}}$. In this section, we are willing to take into account the effect of having finite samples. Thus, in line with the above, we are now interested in the generalization properties of WGANs and look for upper bounds on the quantity
\begin{equation}\label{eq:asymptotic_approximation}
    0 \leqslant d_{\text{Lip}_1}(\mu^\star, \mu_{\hat{\theta}_n}) - \underset{\theta \in \Theta}{\inf} \ d_{\text{Lip}_1}(\mu^\star, \mu_{\theta}).
\end{equation}
\citet[][Theorem 3.1]{Arora0LMZ17} states an asymptotic result showing that when provided enough samples, the neural IPM $d_{\mathscr{D}}$ generalizes well, in the sense that for any pair $(\mu, \nu) \in {P}_1 (E)^2$, the difference $| d_{\mathscr{D}}(\mu, \nu) - d_{\mathscr{D}}(\mu_n, \nu_n) |$ can be arbitrarily small with high probability. However, this result does not give any information on the quantity of interest $d_{\text{Lip}_1}(\mu^\star, \mu_{\hat{\theta}_n}) - \inf_{\theta \in \Theta} \ d_{\text{Lip}_1}(\mu^\star, \mu_{\theta})$. Closer to our current work, \citet{zhang2018discrimination} provide bounds for $d_{\mathscr{D}}(\mu^\star, \mu_{\hat{\theta}_n}) - \inf_{\theta \in \Theta} \ d_{\mathscr{D}}(\mu^\star, \mu_{\theta}$), starting from the observation that
\begin{equation} \label{eq:epsilon_optim_bound}
        0 \leqslant d_{\mathscr{D}}(\mu^\star, \mu_{\hat{\theta}_n}) - \underset{\theta \in \Theta}{\inf} \ d_{\mathscr{D}}(\mu^\star, \mu_{{\theta}}) \leqslant 2 d_{\mathscr{D}}(\mu^\star, \mu_n).
\end{equation}
In the present article, we develop a complementary point of view and measure the generalization properties of WGANs on the basis of the Wasserstein distance $d_{\text{Lip}_1}$, as in equation \eqref{eq:asymptotic_approximation}. Our approach is motivated by the fact that the neural IPM $d_{\mathscr{D}}$ is only used for easing the optimization process and, accordingly, that the performance should be assessed on the basis of the distance $d_{\text{Lip}_1}$, not $d_{\mathscr{D}}$.

Note that $\hat{\theta}_n$, which minimizes $d_{\mathscr{D}}(\mu_n, \mu_\theta)$ over $\Theta$, may not be unique. Besides, there is no guarantee that two distinct elements $\theta_{n,1}$ and $\theta_{n,2}$ of $\hat{\Theta}_n$ lead to the same distance $d_{\text{Lip}_1}(\mu^\star, \mu_{{\theta}_{n,1}})$ and $d_{\text{Lip}_1}(\mu^\star, \mu_{{\theta}_{n,2}})$ (again, $\hat{\theta}_n$ is computed with $d_{\mathscr{D}}$, not with $d_{\text{Lip}_1}$). Therefore, in order to upper-bound the error in \eqref{eq:asymptotic_approximation}, we let, for each $\theta_n \in \hat{\Theta}_n$,
\begin{equation*}
    \bar{\theta}_n \in  \underset{\bar{\theta} \in \bar{\Theta}}{\argmin} \ \|\theta_n-\bar{\theta}\|.
\end{equation*}
The rationale behind the definition of $\bar{\theta}_n$ is that we expect it to behave ``similarly'' to $\theta_n$. Following our objective, the error can be decomposed as follows:
\begin{align}
    0 &\leqslant d_{\text{Lip}_1}(\mu^\star, \mu_{\hat{\theta}_n}) - \underset{\theta \in \Theta}{\inf} \ d_{\text{Lip}_1}(\mu^\star, \mu_{\theta}) \nonumber \\
    &\leqslant \underset{{\theta}_n \in \hat{\Theta}_n}{\sup} \ d_{\text{Lip}_1}(\mu^\star, \mu_{{\theta}_n}) - \underset{\theta \in \Theta}{\inf}\ d_{\text{Lip}_1}(\mu^\star, \mu_{\theta}) \nonumber \\
    & = \underset{{\theta}_n \in \hat{\Theta}_n}{\sup} \ \big[ d_{\text{Lip}_1}(\mu^\star, \mu_{{\theta}_n}) - d_{\text{Lip}_1}(\mu^\star, \mu_{\bar{\theta}_n}) + d_{\text{Lip}_1}(\mu^\star, \mu_{\bar{\theta}_n}) \big] - \underset{\theta \in \Theta}{\inf} \ d_{\text{Lip}_1}(\mu^\star, \mu_{{\theta}}) \nonumber \\
    &\leqslant \underset{{\theta}_n \in \hat{\Theta}_n}{\sup} \big[ d_{\text{Lip}_1}(\mu^\star, \mu_{{\theta}_n}) - d_{\text{Lip}_1}(\mu^\star, \mu_{\bar{\theta}_n}) \big] + \underset{\bar{\theta} \in \bar{\Theta}}{\sup} \ d_{\text{Lip}_1}(\mu^\star, \mu_{\bar{\theta}}) - \underset{\theta \in \Theta}{\inf}\ d_{\text{Lip}_1}(\mu^\star, \mu_{\theta}) \nonumber \\
    &= \varepsilon_{\text{estim}} + \varepsilon_{\text{optim}} \label{eq:anticipating},
\end{align}
where we set $\varepsilon_{\text{estim}} = \sup_{\theta_n \in \hat{\Theta}_n}\ [d_{\text{Lip}_1}(\mu^\star, \mu_{{\theta}_n}) - d_{\text{Lip}_1}(\mu^\star, \mu_{\bar{\theta}_n})]$. Notice that this supremum can be positive or negative. However, it can be shown to converge to $0$ almost surely when $n\to \infty$.
\begin{lemma}\label{10032020}
Assume that Assumption \ref{ass:compactness} is satisfied. Then $\underset{n \to \infty}{\lim} \ \varepsilon_{\emph{estim}} = 0$ almost surely.
\end{lemma}

Going further with the analysis of \eqref{eq:asymptotic_approximation}, the sum $\varepsilon_{\text{estim}} + \varepsilon_{\text{optim}}$ is bounded as follows:
\begin{align*}
    \varepsilon_{\text{estim}} + \varepsilon_{\text{optim}} &\leqslant \underset{{\theta}_n \in \hat{\Theta}_n}{\sup} \big[ d_{\text{Lip}_1}(\mu^\star, \mu_{{\theta}_n}) - d_{\text{Lip}_1}(\mu^\star, \mu_{\bar{\theta}_n}) \big] + T_{\mathscr{P}}(\text{Lip}_1, \mathscr{D}) \\
    & \quad \mbox{(by inequality \eqref{eq:eps_optim_inequality})} \\
    &\leqslant \underset{\theta_n \in \hat{\Theta}_n}{\sup} \big[ d_{\text{Lip}_1}(\mu^\star, \mu_{\theta_n}) - \underset{\theta \in \Theta}{\inf}\ d_{\mathscr{D}}(\mu^\star, \mu_{{\theta}}) \big] + T_{\mathscr{P}}(\text{Lip}_1, \mathscr{D}).
\end{align*}
Hence,
\begin{align}
     &\varepsilon_{\text{estim}} + \varepsilon_{\text{optim}} \nonumber\\
     &\quad \leqslant \underset{\theta_n \in \hat{\Theta}_n}{\sup} \big[ d_{\text{Lip}_1}(\mu^\star, \mu_{\theta_n}) - d_{\mathscr{D}}(\mu^\star, \mu_{\theta_n}) +
     d_{\mathscr{D}}(\mu^\star, \mu_{\theta_n}) - \underset{\theta \in \Theta}{\inf}\ d_{\mathscr{D}}(\mu^\star, \mu_{{\theta}}) \big] + T_{\mathscr{P}}(\text{Lip}_1, \mathscr{D}) \nonumber \\
    & \quad \leqslant \underset{\theta_n \in \hat{\Theta}_n}{\sup} \big[ d_{\text{Lip}_1}(\mu^\star, \mu_{\theta_n}) - d_{\mathscr{D}}(\mu^\star, \mu_{\theta_n}) \big] + 2 d_{\mathscr{D}}(\mu^\star, \mu_n) + T_{\mathscr{P}}(\text{Lip}_1, \mathscr{D}) \nonumber \\
    & \qquad \mbox{(upon noting that inequality \eqref{eq:epsilon_optim_bound} is also valid for any $\theta_n \in \hat{\Theta}_n$)} \nonumber \\
    & \quad \leqslant 2 T_{\mathscr{P}}(\text{Lip}_1, \mathscr{D}) + 2 d_{\mathscr{D}}(\mu^\star, \mu_n) \label{eq:eps_estim_and_esp_optim}.
\end{align}
The above bound is a function of both the generator and the discriminator. The term $T_{\mathscr{P}}(\text{Lip}_1, \mathscr{D})$ is increasing when the capacity of the generator is increasing. The discriminator, however, plays a more ambivalent role, as already pointed out by \citet{zhang2018discrimination}. On the one hand, if the discriminator's capacity decreases,  the gap between $d_{\mathscr{D}}$ and $d_{\text{Lip}_1}$ gets bigger and $T_{\mathscr{P}}(\text{Lip}_1, \mathscr{D})$ increases. On the other hand, discriminators with bigger capacities ought to increase the contribution $d_{\mathscr{D}}(\mu^\star, \mu_n)$.
In order to bound $d_{\mathscr{D}}(\mu^\star, \mu_n)$, Proposition \ref{prop:generalization_bounds_for_dD} below extends \citet[][Theorem 3.1]{zhang2018discrimination}, in the sense that it does not require the set of discriminative functions nor the space $E$ to be bounded. Recall that, for $\gamma>0$, $\mu^\star$ is said to be $\gamma$ sub-Gaussian \citep{jin2019short} if
    \begin{equation*}
        \forall v \in \mathds{R}^d, \ \mathds{E} e^{v \cdot (T -\mathds{E} T)} \leqslant e^{\frac{\gamma^2\|v\|^2}{2}},
    \end{equation*}
where $T$ is a random vector with probability distribution $\mu^{\star}$ and $\cdot$ denotes the dot product in $\mathds{R}^D$.
\begin{proposition}\label{prop:generalization_bounds_for_dD}
Assume that Assumption \ref{ass:compactness} is satisfied, let $\eta \in(0,1)$, and let $\mathscr{D}$ be a discriminator of the form \eqref{eq:def_discriminators}.
\begin{enumerate}[$(i)$]
    \item If $\mu^\star$ has compact support with diameter $B$, then there exists a constant $c_1>0$ such that, with probability at least $1-\eta$,
    \begin{equation*}
        d_{\mathscr{D}}(\mu^\star, \mu_n) \leqslant \frac{c_1}{\sqrt{n}} + B\sqrt{ \frac{\log(1/\eta)}{2n}}.
    \end{equation*}
    \item More generally, if $\mu^\star$ is $\gamma$ sub-Gaussian, then there exists a constant $c_2>0$ such that, with probability at least $1-\eta$,
\begin{equation*}
    d_{\mathscr{D}}(\mu^\star, \mu_n) \leqslant \frac{c_2}{\sqrt{n}} + 8 \gamma \sqrt{eD} \sqrt{ \frac{\log(1/\eta)}{n}}.
\end{equation*}
\end{enumerate}
\end{proposition}

The result of Proposition \ref{prop:generalization_bounds_for_dD} has to be compared with convergence rates of the Wasserstein distance. According to \citet[][Theorem 1]{fournier2015rate}, when the dimension $D$ of $E$ is such that $D>2$, if $\mu^\star$ has a second-order moment, then there exists a constant $c$ such that
\begin{equation*}
    0 \leqslant \mathds{E} d_{\text{Lip}_1}(\mu^\star,\mu_n) \leqslant \frac{c}{n^{1/D}}.
\end{equation*}
Thus, when the space $E$ is of high dimension (e.g., in image generation tasks), under the conditions of Proposition \ref{prop:generalization_bounds_for_dD}, the pseudometric $d_{\mathscr{D}}$ provides much faster rates of convergence for the empirical measure. However, one has to keep in mind that both constants $c_1$ and $c_2$ grow in $O(qQ^{3/2}(D^{1/2}+q))$.

Our Theorem \ref{th:approx_properties} states the existence of a discriminator such that $\varepsilon_{\text{optim}}$ can be arbitrarily small. It is therefore reasonable, in view of inequality \eqref{eq:eps_estim_and_esp_optim}, to expect that the sum $\varepsilon_{\text{estim}}+\varepsilon_{\text{optim}}$ can also be arbitrarily small, at least in an asymptotic sense. This is encapsulated in Theorem \ref{theorem:asymptotic_approximation} below.
\begin{theorem}\label{theorem:asymptotic_approximation}
    Assume that Assumption \ref{ass:compactness} is satisfied, and let $\eta \in (0,1)$.
\begin{enumerate}[$(i)$]
    \item If $\mu^\star$ has compact support with diameter $B$, then, for all $\varepsilon >0$, there exists a discriminator $\mathscr{D}$ of the form \eqref{eq:def_discriminators} and a constant $c_1>0$ (function of $\varepsilon$) such that, with probability at least $1-\eta$,
    \begin{equation*}
        0 \leqslant \varepsilon_{\emph{estim}} + \varepsilon_{\emph{optim}} \leqslant 2\varepsilon + \frac{2c_1}{\sqrt{n}} + 2B \sqrt{ \frac{\log(1/\eta)}{2n}}.
    \end{equation*}
    \item More generally, if $\mu^\star$ is $\gamma$ sub-Gaussian, then, for all $\varepsilon >0$, there exists a discriminator $\mathscr{D}$ of the form \eqref{eq:def_discriminators} and a constant $c_2>0$ (function of $\varepsilon$) such that, with probability at least $1-\eta$,
\begin{equation*}
    0 \leqslant \varepsilon_{\emph{estim}} + \varepsilon_{\emph{optim}} \leqslant 2\varepsilon + \frac{2c_2}{\sqrt{n}} + 16 \gamma \sqrt{eD} \sqrt{\frac{\log(1/\eta)}{n}}.
\end{equation*}
\end{enumerate}
\end{theorem}

Theorem \ref{theorem:asymptotic_approximation} states that, asymptotically, the optimal parameters in $\hat{\Theta}_n$ behave properly. A caveat is that the definition of $\varepsilon_{\text{estim}}$ uses $\hat{\Theta}_n$. However, in practice, one never has access to $\hat{\theta}_n$, but rather to an approximation of this quantity obtained by gradient descent algorithms. Thus, in line with Definition \ref{def:substitution}, we introduce the concept of empirical substitution:
\begin{definition}\label{def:empirical_substitution}
    Let $\varepsilon >0$ and $\eta \in (0,1)$. We say that $d_{\emph{Lip}_1}$ can be empirically $\varepsilon$-substituted by $d_{\mathscr{D}}$ if there exists $\delta >0$ such that, for all $n$ large enough, with probability at least $1-\eta$,
    \begin{equation}\label{asymptotic_equivalence_ideally}
        \mathscr{M}_{d_{\mathscr{D}}}(\mu_n, \delta) \subseteq \mathscr{M}_{d_{\emph{Lip}_1}}(\mu^\star, \varepsilon).
    \end{equation}
\end{definition}

The rationale behind this definition is that if \eqref{asymptotic_equivalence_ideally} is satisfied, then by minimizing the IPM $d_{\mathscr{D}}$ close to optimality in \eqref{empirical_approximated_wgans}, one can be guaranteed to be also close to optimality in \eqref{eq:theoretical_wgans} with high probability. We stress that Definition \ref{def:empirical_substitution} is the empirical counterpart of Definition \ref{def:substitution}.
\begin{proposition}\label{prop:asymptotic_optimization_properties}
    Assume that Assumption \ref{ass:compactness} is satisfied and that $\mu^\star$ is sub-Gaussian. Let $\varepsilon >0$. If $T_{\mathscr{P}}(\emph{Lip}_1, \mathscr{D}) \leqslant \varepsilon$, then $d_{\emph{Lip}_1}$ can be empirically $(\varepsilon + \delta)$-substituted by $d_{\mathscr{D}}$ for all $\delta>0$.
\end{proposition}

This proposition is the empirical counterpart of Lemma \ref{lem:equivalence_properties_both_distances}. It underlines the fact that by minimizing the pseudometric $d_{\mathscr{D}}$ between the empirical measure $\mu_n$ and the set of generative distributions $\mathscr{P}$ close to optimality, one can control the loss in performance under the metric $d_{\text{Lip}_1}$.

Let us finally mention that it is also possible to provide asymptotic results on the sequences of parameters $(\hat{\theta}_n)$, keeping in mind that $\hat{\Theta}_n$ and $\bar{\Theta}$ are not necessarily reduced to singletons.
\begin{lemma}\label{lem:asymptotic_convergence_parameters}
    Assume that Assumption \ref{ass:compactness} is satisfied. Let $(\hat{\theta}_n)$ be a sequence of optimal parameters that converges almost surely to $z \in \Theta$. Then $z \in \bar{\Theta}$ almost surely.
\end{lemma}
\begin{proof}
    Let the sequence $(\hat{\theta}_n)$ converge almost surely to some $z \in \Theta$. By Theorem \ref{th:continuity}, the function $\Theta \ni \theta \mapsto d_\mathscr{D}(\mu^\star, \mu_\theta)$ is continuous, and therefore, almost surely, $\underset{n \to \infty}{\lim} d_\mathscr{D}(\mu^\star, \mu_{\hat{\theta}_n}) = d_\mathscr{D}(\mu^\star, \mu_z)$. Using inequality \eqref{eq:epsilon_optim_bound}, we see that, almost surely,
    \begin{align*}
        0 \leqslant d_{\mathscr{D}}(\mu^\star, \mu_{z}) - \underset{\theta \in \Theta}{\inf} \ d_{\mathscr{D}}(\mu^\star, \mu_{{\theta}}) &= \underset{n \to \infty}{\lim} \ d_\mathscr{D}(\mu^\star, \mu_{\hat{\theta}_n}) - \underset{\theta \in \Theta}{\inf} \ d_{\mathscr{D}}(\mu^\star, \mu_{{\theta}}) \\
        & \leqslant \underset{n \to \infty}{\liminf} \ 2 d_{\mathscr{D}}(\mu^\star, \mu_n).
    \end{align*}
Using \citet[][Theorem 11.4.1]{dudley_2002} and the strong law of large numbers, we have that the sequence of empirical measures $(\mu_n)$ almost surely converges weakly to $\mu^\star$ in $P_1(E)$. Besides, since $d_{\mathscr{D}}$ metrizes weak convergence in $P_1(E)$ (by Proposition \ref{cor:neural_distance}), we conclude that $z \in \bar{\Theta}$ almost surely.
\end{proof}
\section{Understanding the performance of WGANs} \label{section:trade_off_properties}
In order to better understand the overall performance of the WGANs architecture, it is instructive to decompose the final loss $d_{\text{Lip}_1}(\mu^\star, \mu_{\hat{\theta}_n})$ as in \eqref{eq:anticipating}:
\begin{align}
    d_{\text{Lip}_1}(\mu^\star, \mu_{\hat{\theta}_n}) &\leqslant {\varepsilon_{\text{estim}}}+{\varepsilon_{\text{optim}}}+{\underset{\theta \in \Theta}{\inf} \ d_{\text{Lip}_1}(\mu^\star, \mu_{\theta})} \nonumber \\
    &={\varepsilon_{\text{estim}}}+{\varepsilon_{\text{optim}}}+{\varepsilon_{\text{approx}}} \label{eq:performance_overall_wgans},
\end{align}
where
\begin{enumerate}[$(i)$]
    \item $\varepsilon_{\text{estim}}$ matches up with the use of a data-dependent optimal parameter $\hat{\theta}_n$, based on the training set $X_1, \hdots, X_n$ drawn from $\mu^\star$;
    \item $\varepsilon_{\text{optim}}$ corresponds to the loss in performance when using $d_{\mathscr{D}}$ as training loss instead of $d_{\text{Lip}_1}$ (this term has been thoroughly studied in Section \ref{section:approximation_properties});
    \item and $\varepsilon_{\text{approx}}$ stresses the capacity of the parametric family of generative distributions $\mathscr{P}$ to approach the unknown distribution $\mu^\star$.
\end{enumerate}

Close to our work are the articles by \citet{liang2018well}, \citet{nonparametric2018singh}, and \citet{nonparametric2019wallach}, who study statistical properties of GANs. \citet{liang2018well} and \citet{nonparametric2018singh} exhibit rates of convergence under an IPM-based loss for estimating densities that live in Sobolev spaces, while \citet{nonparametric2019wallach} explore the case of Besov spaces. More recently, \citet{schreuder2021statistical} have stressed the properties of IPM losses defined with smooth functions on a compact set. Remarkably, \citet{liang2018well} discusses bounds for the Kullback-Leibler divergence, the Hellinger divergence, and the Wasserstein distance between $\mu^\star$ and $\mu_{\hat{\theta}_n}$. These bounds are based on a different decomposition of the loss and offer a complementary point of view. We emphasize that, in the present article, no density assumption is made neither on the class of generative distributions $\mathscr{P}$ nor on the target distribution $\mu^\star$. Studying a different facet of the problem, \citet{luise2020generalization} analyze the interplay between the latent distribution and
the complexity of the pushforward map, and how it affects the overall performance. 

\subsection{Synthetic experiments}
Our goal in this subsection is to illustrate \eqref{eq:performance_overall_wgans} by running a set of experiments on synthetic datasets. The true probability measure $\mu^\star$ is assumed to be a mixture of bivariate Gaussian distributions with either 1, 4, or 9 components. This simple setting allows us to control the complexity of $\mu^\star$, and, in turn, to better assess the impact of both the generator's and discriminator's capacities. We use growing classes of generators of the form \eqref{eq:def_generators}, namely  $\{\mathscr{G}_p: p=2, 3, 5, 7\}$, and growing classes of discriminators of the form \eqref{eq:def_discriminators}, namely $\{ \mathscr{D}_q: q =2, 3, 5, 7 \}$. For both the generator and the discriminator, the width of the hidden layers is kept constant equal to $20$.

Two metrics are computed to evaluate the behavior of the different generative models. First, we use the Wasserstein distance between the true distribution (either $\mu^\star$ or its empirical version $\mu_n$) and the generative distribution (either $\mu_{\bar{\theta}}$ or $\mu_{\hat{\theta}_n}$). This distance is calculated by using the Python package by \cite{flamary2017pot}, via finite samples of size 4096 (average over 20 runs). Second, we use the recall metric (the higher, the better), proposed by \citet[][]{kynkaanniemi2019improved}. Roughly, this metric measures ``how much'' of the true distribution (either $\mu^\star$ or $\mu_n$) can be reconstructed by the generative distribution (either $\mu_{\bar{\theta}}$ or $\mu_{\hat{\theta}_n}$). At the implementation level, this score is based on $k$-nearest neighbor nonparametric density estimation. It is computed via finite samples of size 4096 (average over 20 runs).

Our experiments were run in two different settings:
\paragraph{Asymptotic setting:} in this first experiment, we assume that $\mu^\star$ is known from the experimenter (so, there is no dataset). At the end of the optimization scheme, we end up with one $\bar{\theta} \in \bar{\Theta}$. Thus, in this context, the performance of WGANs is captured by
\begin{equation*}
    \underset{\bar{\theta} \in \bar{\Theta}}{\sup} \ d_{\text{Lip}_1}(\mu^\star, \mu_{\bar{\theta}}) = {\varepsilon_{\text{optim}}}+{\varepsilon_{\text{approx}}}.
\end{equation*}
For a fixed discriminator, when increasing the generator's depth $p$, we expect $\varepsilon_{\text{approx}}$ to decrease. Conversely, as discussed in Subsection \ref{sec:approximation_properties}, we anticipate an augmentation of $\varepsilon_{\text{optim}}$, since the discriminator must now differentiate between larger classes of generative distributions. In this case, it is thus difficult to predict how $\sup_{\bar{\theta} \in \bar{\Theta}} \ d_{\text{Lip}_1}(\mu^\star, \mu_{\bar{\theta}})$ behaves when $p$ increases. On the contrary, in accordance with the results of Section \ref{section:approximation_properties}, for a fixed $p$ we expect the performance to increase with a growing $q$ since, with larger discriminators, the pseudometric $d_{\mathscr{D}}$ is more likely to behave similarly to the Wasserstein distance $d_{\text{Lip}_1}$.

These intuitions are validated by Figure \ref{fig:increasing_both_gen_and_disc_asymptotic_setting_emd} and Figure \ref{fig:increasing_both_gen_and_disc_asymptotic_setting_recall} (the bluer, the better). The first one shows an approximation of $\sup_{\bar{\theta} \in \bar{\Theta}} \ d_{\text{Lip}_1}(\mu^\star, \mu_{\bar{\theta}})$ computed over 5 different seeds as a function of $p$ and $q$. The second one depicts the average recall of the estimator $\mu_{\bar{\theta}}$ with respect to $\mu^\star$, as a function of $p$ and $q$, again computed over 5 different seeds. In both figures, we observe that for a fixed $p$, incrementing $q$ leads to better results. On the opposite, for a fixed discriminator's depth $q$, increasing the depth $p$ of the generator seems to deteriorate both scores (Wasserstein distance and recall). This consequently suggests that the term $\varepsilon_{\text{optim}}$ dominates $\varepsilon_{\text{approx}}$.

\begin{figure}[H]
    \subfloat[$\sup_{\bar{\theta} \in \bar{\Theta}} \ d_{\text{Lip}_1}(\mu^\star, \mu_{\bar{\theta}})$, $K=1$.]
    {
        \includegraphics[width=0.31\linewidth]{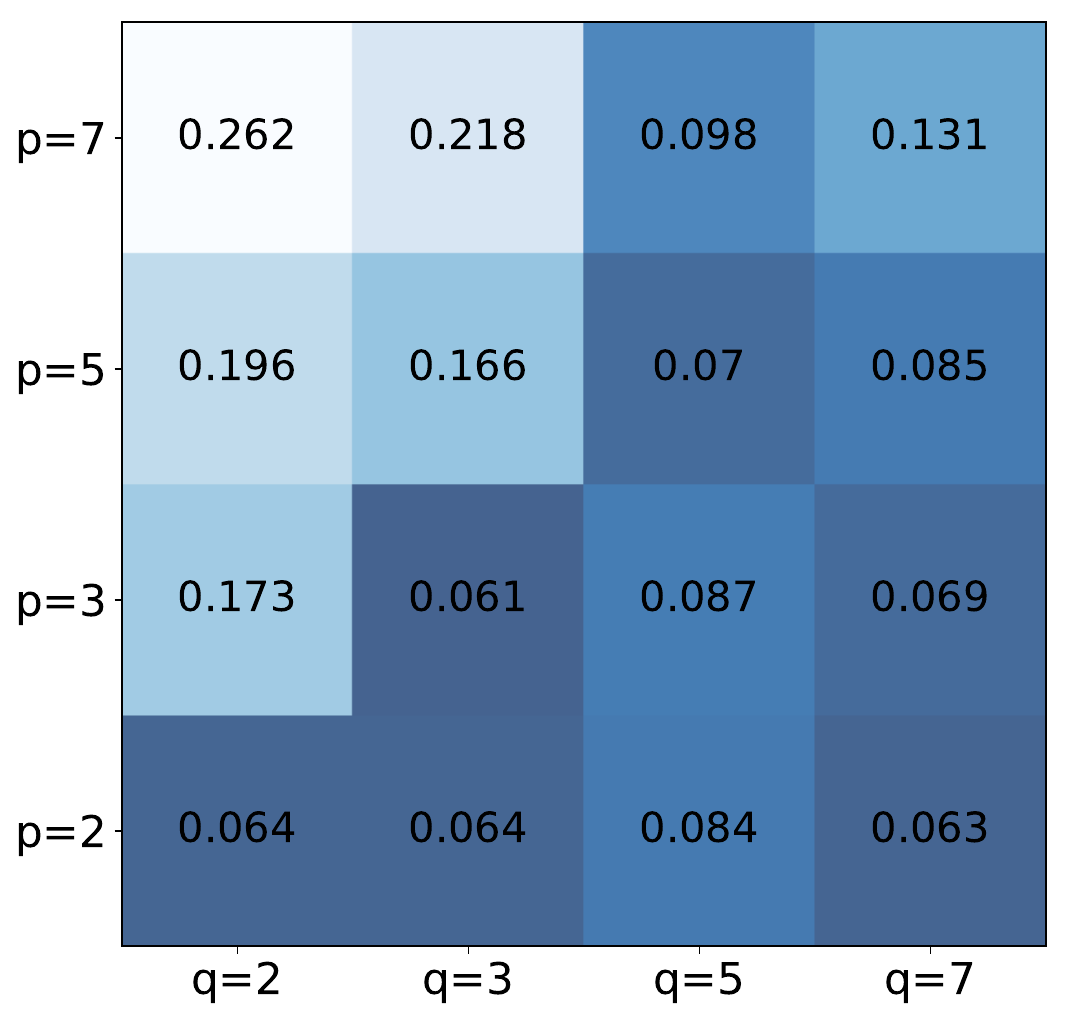}
    }\hfill
    \subfloat[$\sup_{\bar{\theta} \in \bar{\Theta}} \ d_{\text{Lip}_1}(\mu^\star, \mu_{\bar{\theta}})$, $K=9$.]
    {
        \includegraphics[width=0.31\linewidth]{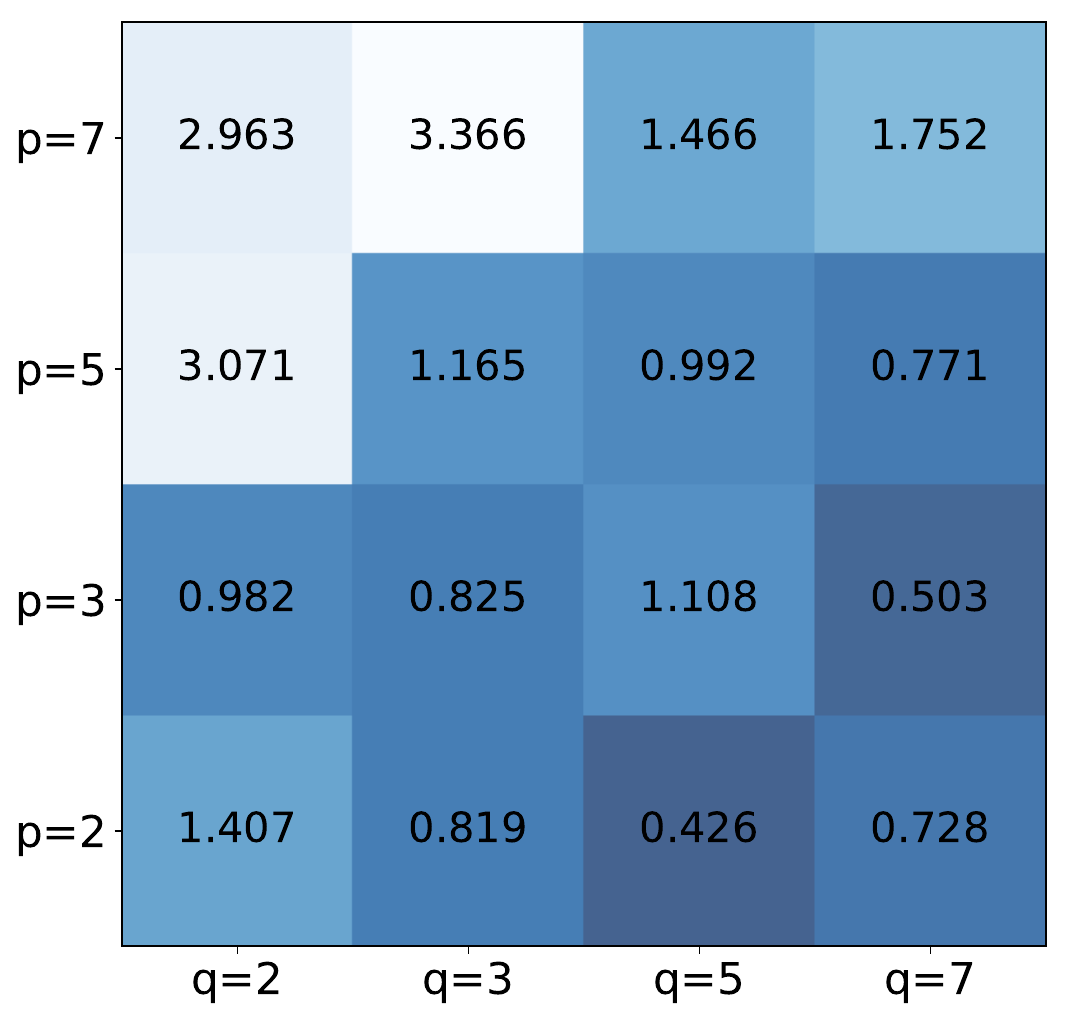}
    }\hfill
    \subfloat[$\sup_{\bar{\theta} \in \bar{\Theta}} \ d_{\text{Lip}_1}(\mu^\star, \mu_{\bar{\theta}})$, $K=25$.]
    {
        \includegraphics[width=0.31\linewidth]{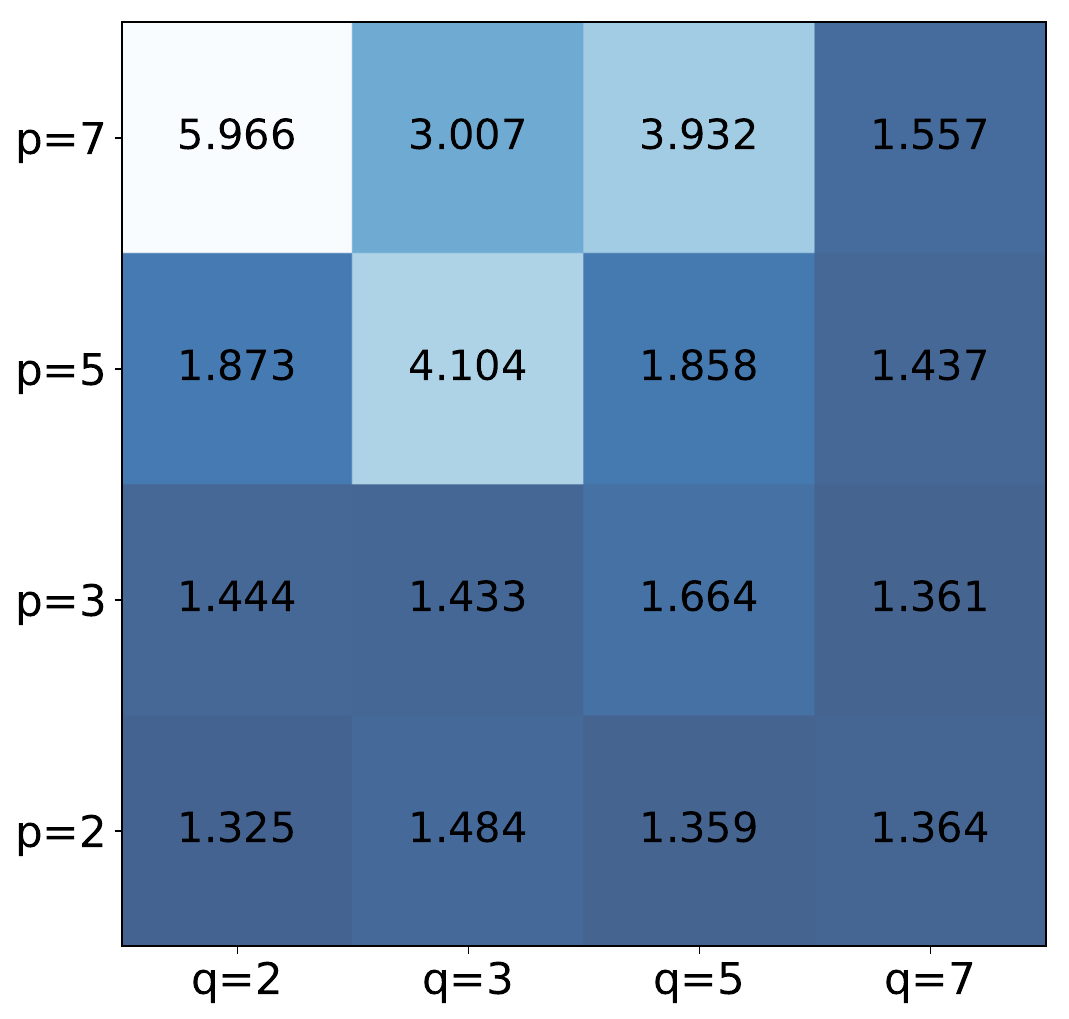}
    }
    \caption{Influence of the generator's depth $p$ and the discriminator's depth $q$ on the maximal Wasserstein distance $\sup_{\bar{\theta} \in \bar{\Theta}} \  d_{\text{Lip}_1}(\mu^\star, \mu_{\bar{\theta}})$.  \label{fig:increasing_both_gen_and_disc_asymptotic_setting_emd}}
\end{figure}
\begin{figure}[H]
    \subfloat[Av.~recall of $\mu_{\bar{\theta}}$ w.r.t.~$\mu^\star$, $K=1$.]
    {
        \includegraphics[width=0.31\linewidth]{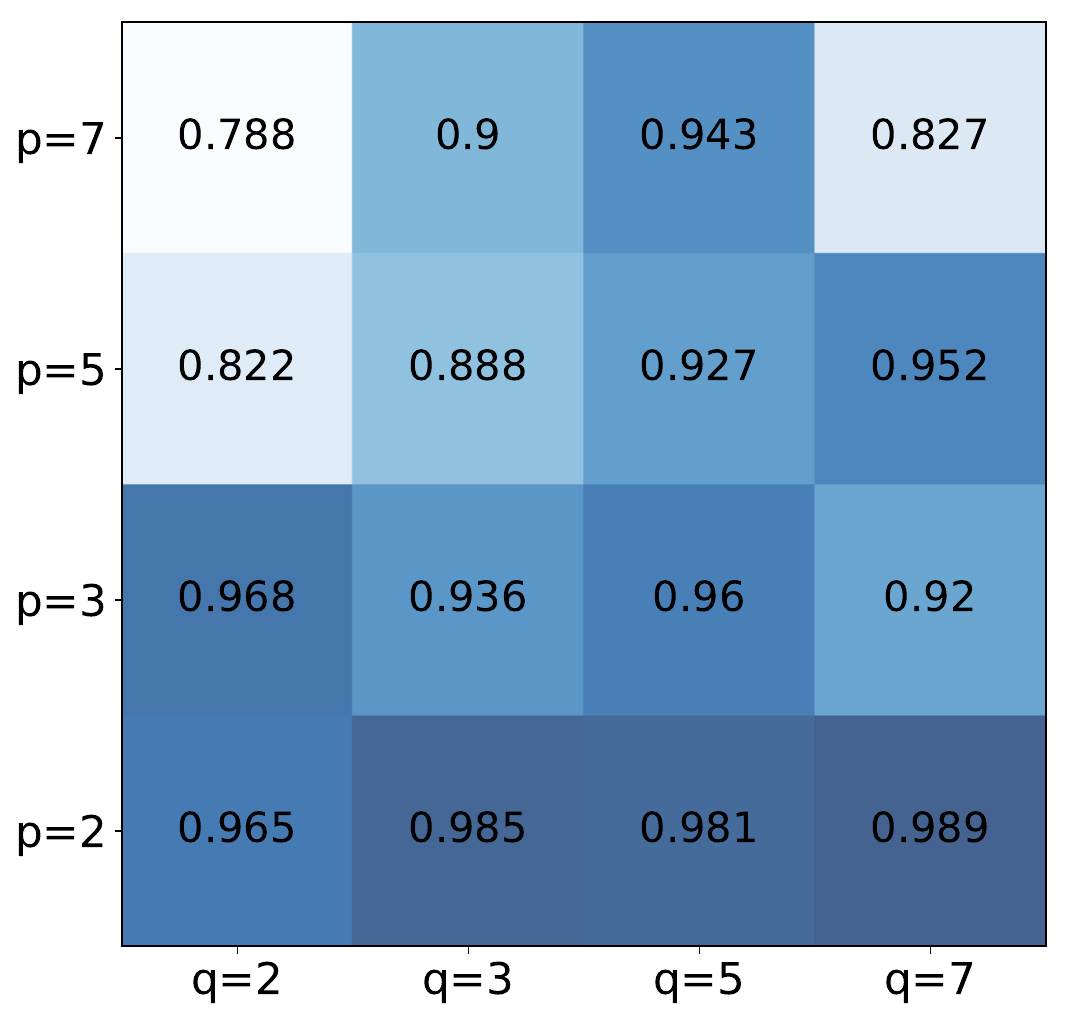}
    }\hfill
    \subfloat[Av.~recall of $\mu_{\bar{\theta}}$ w.r.t.~$\mu^\star$, $K=9$.]
    {
        \includegraphics[width=0.31\linewidth]{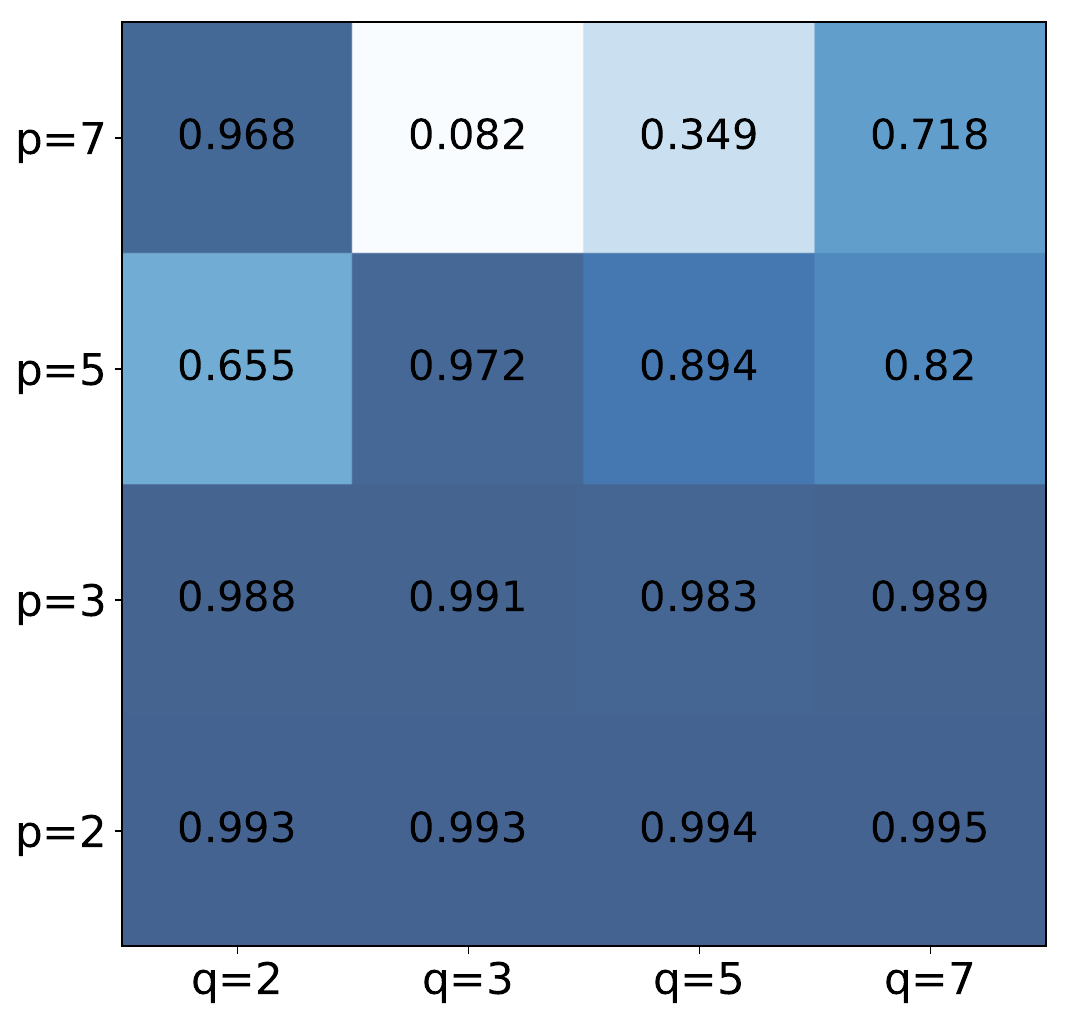}
    }\hfill
    \subfloat[Av.~recall of $\mu_{\bar{\theta}}$ w.r.t.~$\mu^\star$, $K=25$.]
    {
        \includegraphics[width=0.31\linewidth]{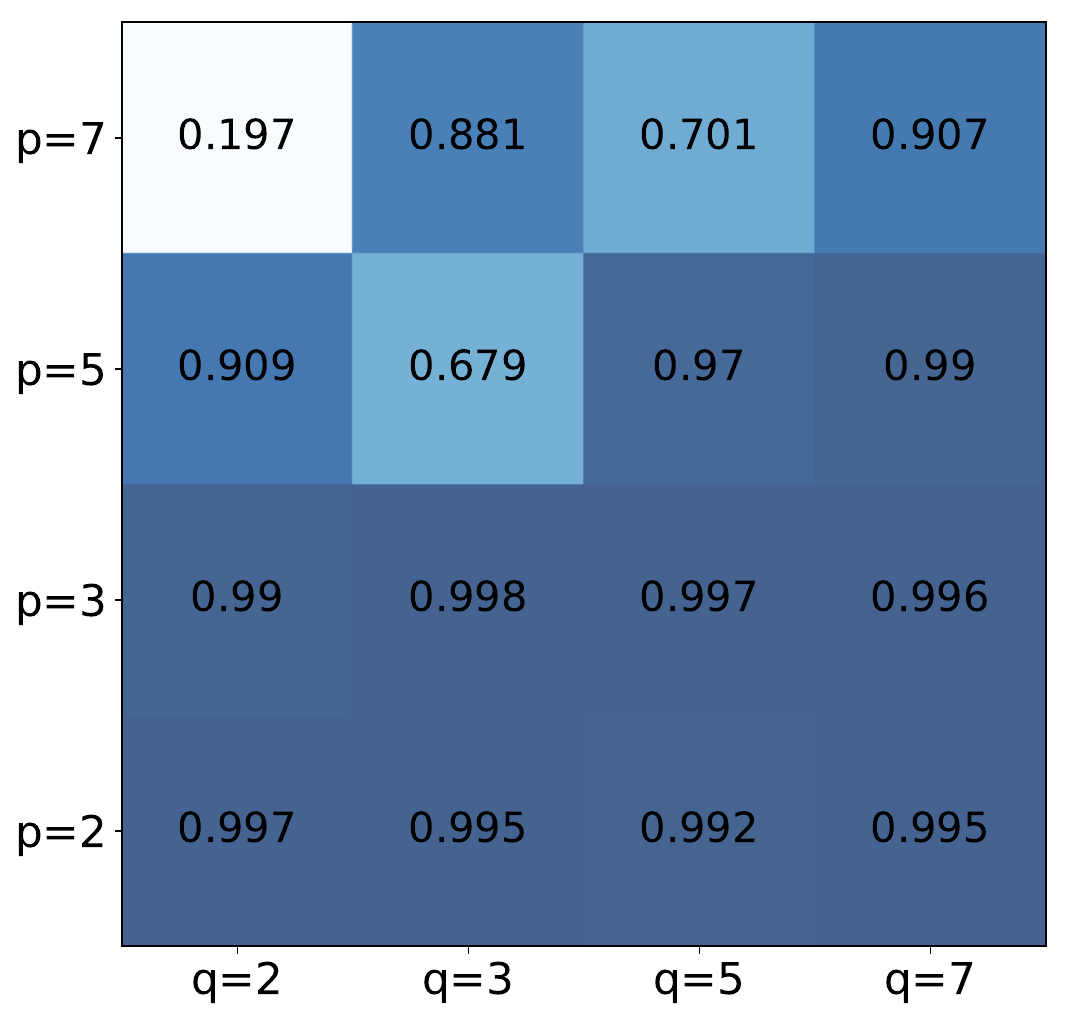}
    }
    \caption{Influence of the generator's depth $p$ and the discriminator's depth $q$ on the average recall of the estimators $\mu_{\bar{\theta}}$ w.r.t.~$\mu^\star$.}
    \label{fig:increasing_both_gen_and_disc_asymptotic_setting_recall}
\end{figure}
\paragraph{Finite-sample setting:} in this second experiment, we consider the more realistic situation where we have at hand finite samples $X_1,\hdots, X_n$ drawn from $\mu^\star$ ($n=5000$).

Recalling that $\sup_{\theta_n \in \hat{\Theta}_n} \  d_{\text{Lip}_1}(\mu^\star, \mu_{{\theta}_n}) \leqslant  {\varepsilon_{\text{estim}}} + {\varepsilon_{\text{optim}}}+{\varepsilon_{\text{approx}}}$, we plot in Figure \ref{fig:increasing_both_gen_and_disc_overall_perf_emd} the maximal Wasserstein distance $\sup_{\theta_n \in \hat{\Theta}_n} \  d_{\text{Lip}_1}(\mu^\star, \mu_{{\theta}_n})$, and in Figure \ref{fig:increasing_both_gen_and_disc_overall_perf_recall} the average recall of the estimators $\mu_{{\theta}_n}$ with respect to $\mu^\star$, as a function of $p$ and $q$. Anticipating the behavior of $\sup_{\theta_n \in \hat{\Theta}_n} \ d_{\text{Lip}_1}(\mu^\star, \mu_{{\theta}_n})$ when increasing the depth $q$ is now more involved. Indeed, according to inequality \eqref{eq:eps_estim_and_esp_optim}, which bounds $\varepsilon_{\text{estim}}+{\varepsilon_{\text{optim}}}$, a larger $\mathscr{D}$ will make $T_{\mathscr{P}}(\text{Lip}_1, \mathscr{D})$ smaller but will, on the opposite, increase $d_{\mathscr{D}}(\mu^\star, \mu_n)$. Figure \ref{fig:increasing_both_gen_and_disc_overall_perf_emd} clearly shows that, for a fixed $p$, the maximal Wasserstein distance seems to be improved when $q$ increases. This suggests that the term $T_{\mathscr{P}}(\text{Lip}_1, \mathscr{D})$ dominates $d_{\mathscr{D}}(\mu^\star, \mu_n)$. Similarly to the asymptotic setting, we also make the observation that bigger $p$ require a higher depth $q$ since larger class of generative distributions are more complex to discriminate.

\begin{figure}[H]
    \subfloat[$\sup_{\theta_n \in \hat{\Theta}_n} d_{\text{Lip}_1}(\mu^\star, \mu_{{\theta}_n})$, $K=1$.]
    {
        \includegraphics[width=0.31\linewidth]{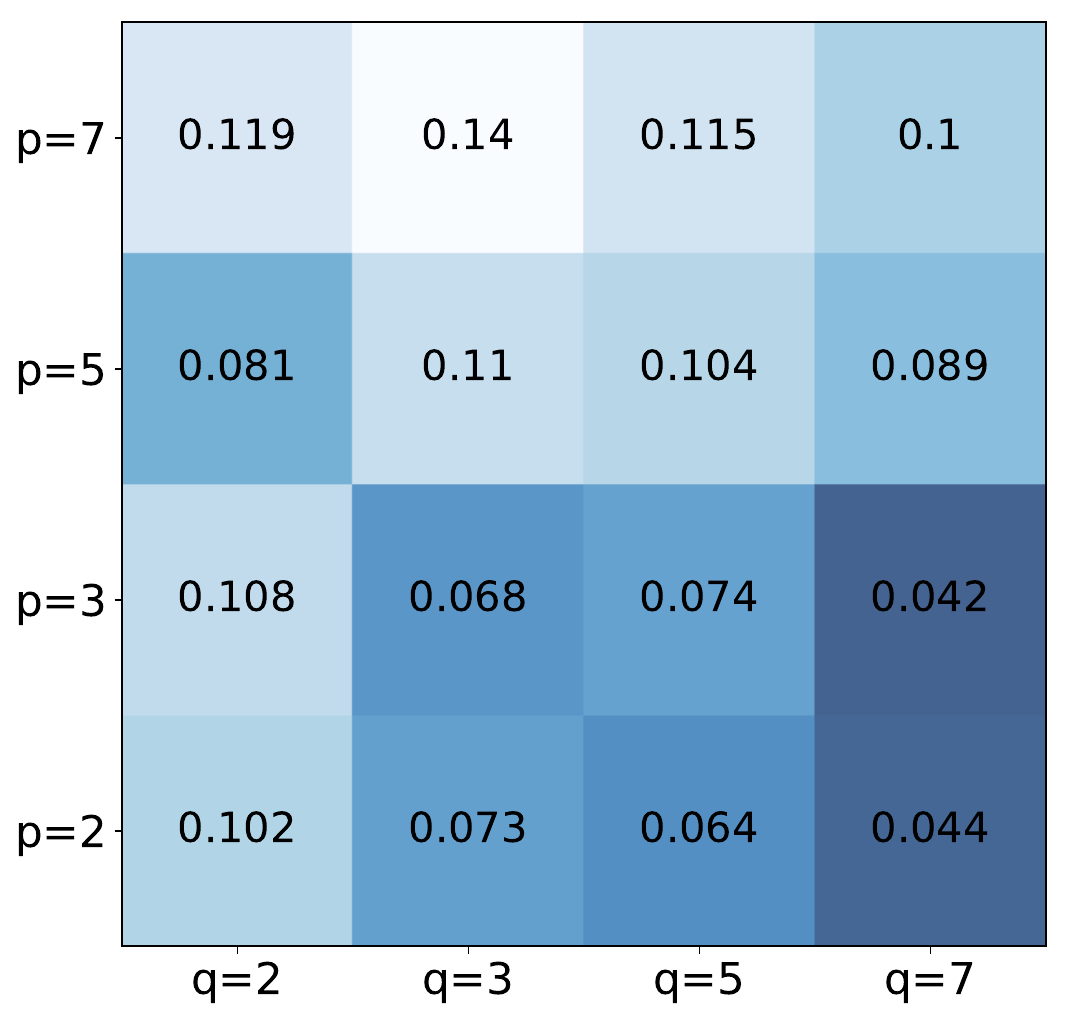}
    }\hfill
    \subfloat[$\sup_{\theta_n \in \hat{\Theta}_n} d_{\text{Lip}_1}(\mu^\star, \mu_{{\theta}_n})$, $K=4$.]
    {
        \includegraphics[width=0.31\linewidth]{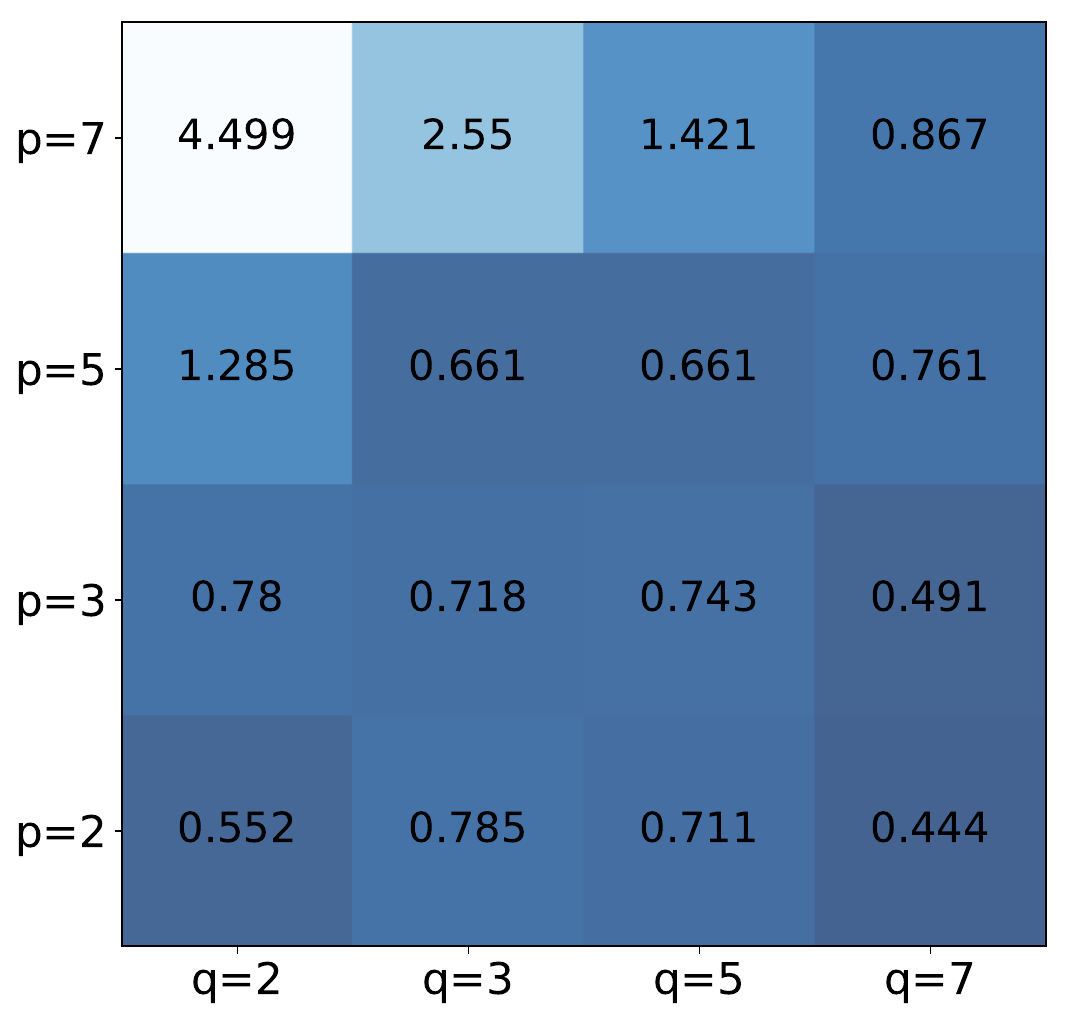}
    }\hfill
    \subfloat[$\sup_{\theta_n \in \hat{\Theta}_n} d_{\text{Lip}_1}(\mu^\star, \mu_{{\theta}_n})$, $K=9$.]
    {
        \includegraphics[width=0.31\linewidth]{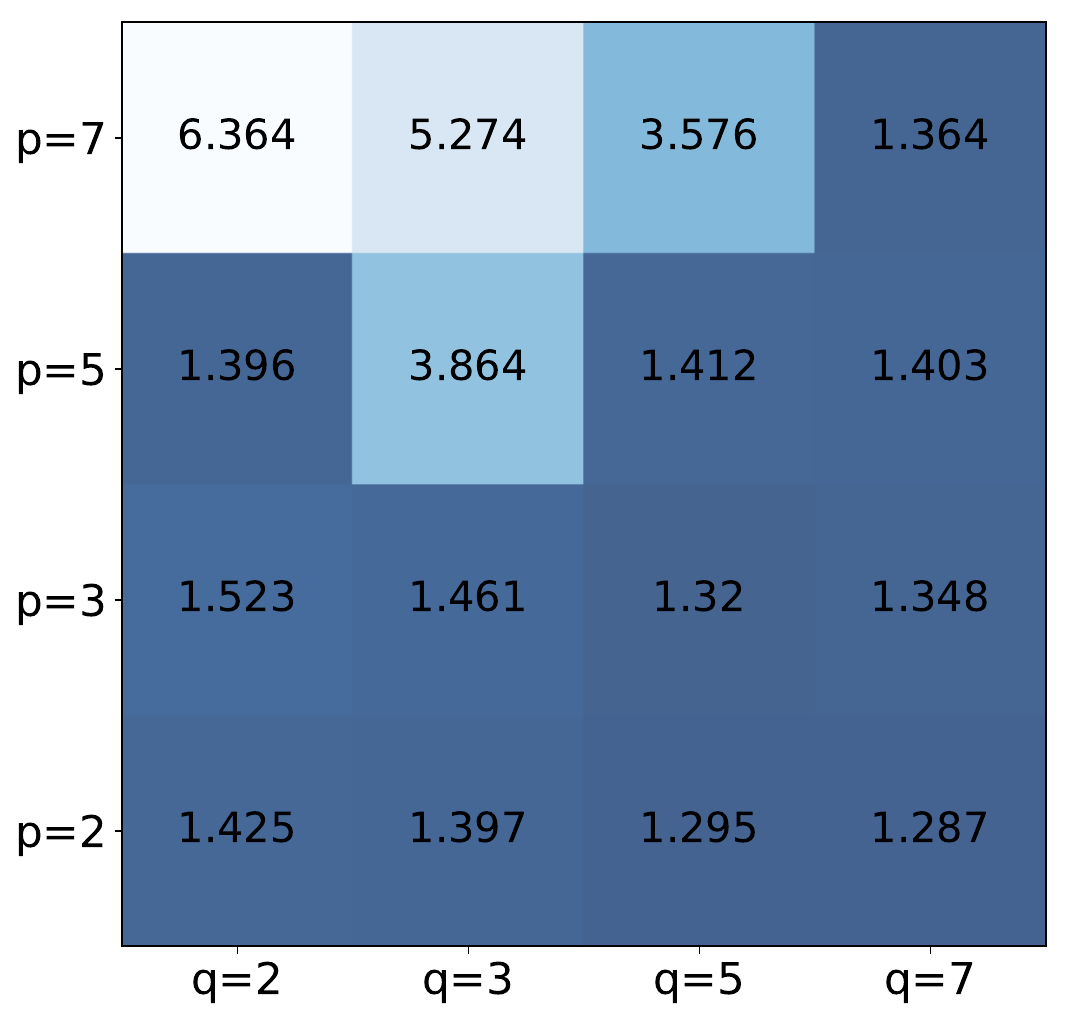}
    }
    \caption{Influence of the generator's depth $p$ and the discriminator's depth $q$ on the maximal Wasserstein distance $\sup_{\theta_n \in \hat{\Theta}_n} \ d_{\text{Lip}_1}(\mu^\star, \mu_{{\theta}_n})$, with $n=5000$.}
    \label{fig:increasing_both_gen_and_disc_overall_perf_emd}
\end{figure}

\begin{figure}[H]
    \subfloat[Av.~recall of $\mu_{\hat{\theta}_n}$ w.r.t.~$\mu^\star$, $K=1$.]
    {
        \includegraphics[width=0.31\linewidth]{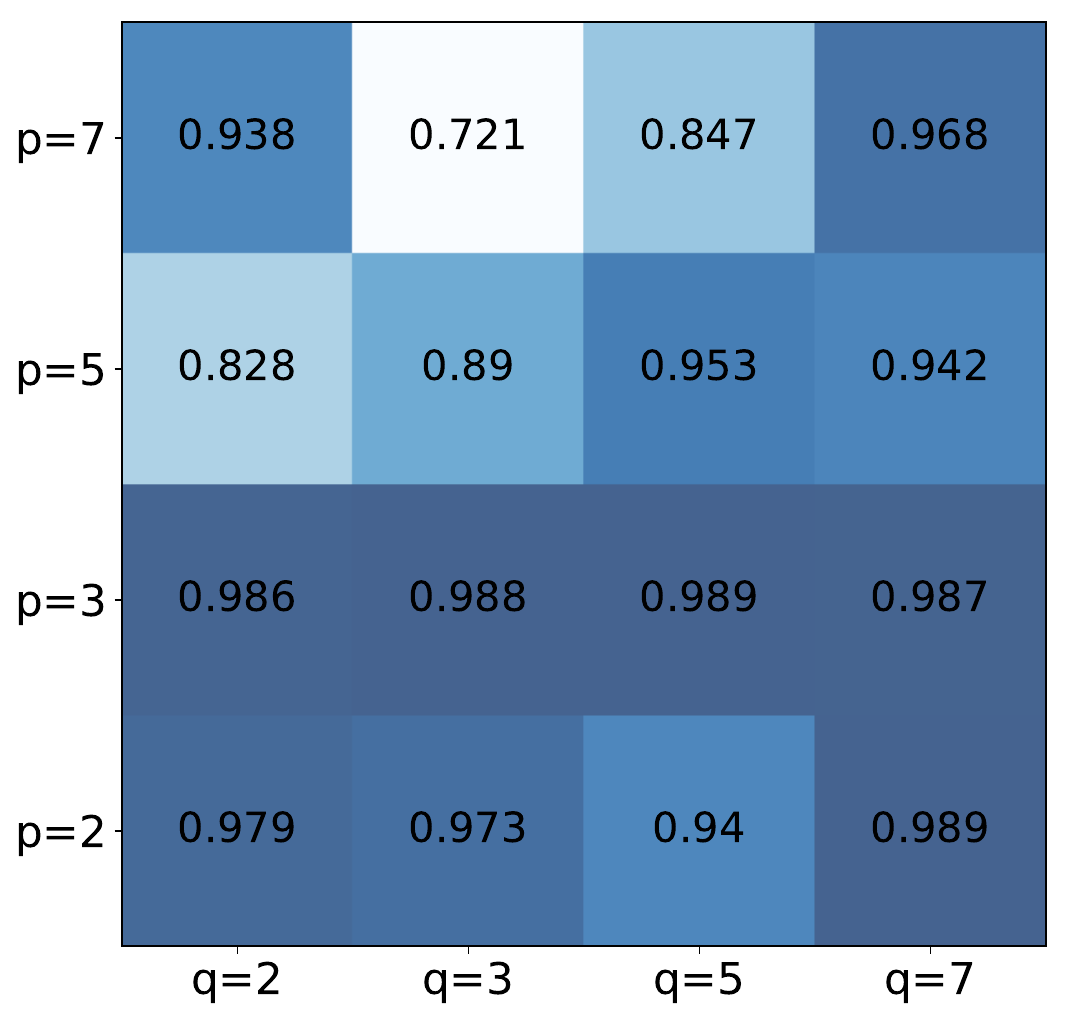}
    }\hfill
    \subfloat[Av.~recall of $\mu_{\hat{\theta}_n}$ w.r.t.~$\mu^\star$, $K=9$.]
    {
        \includegraphics[width=0.31\linewidth]{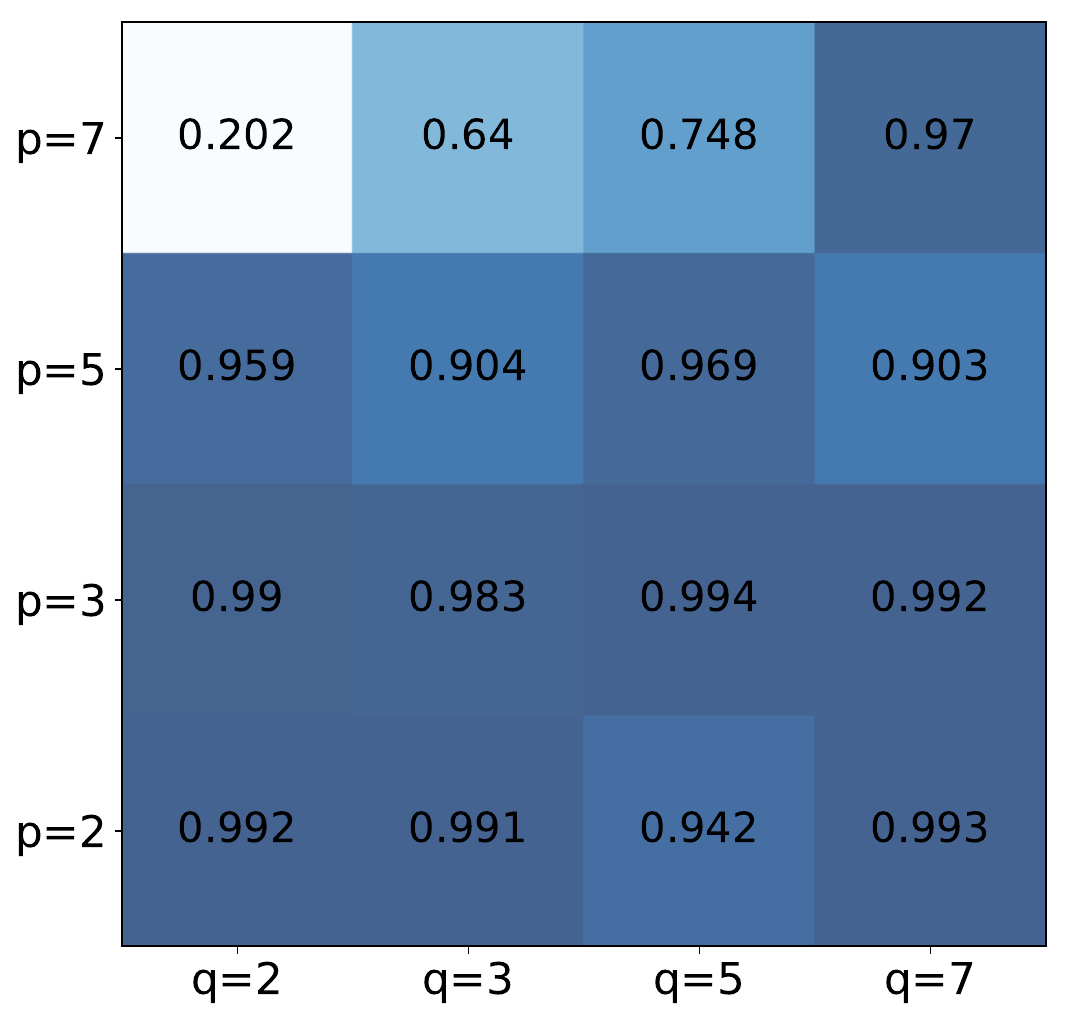}
    }\hfill
    \subfloat[Av.~recall of $\mu_{\hat{\theta}_n}$ w.r.t.~$\mu^\star$, $K=25$.]
    {
        \includegraphics[width=0.31\linewidth]{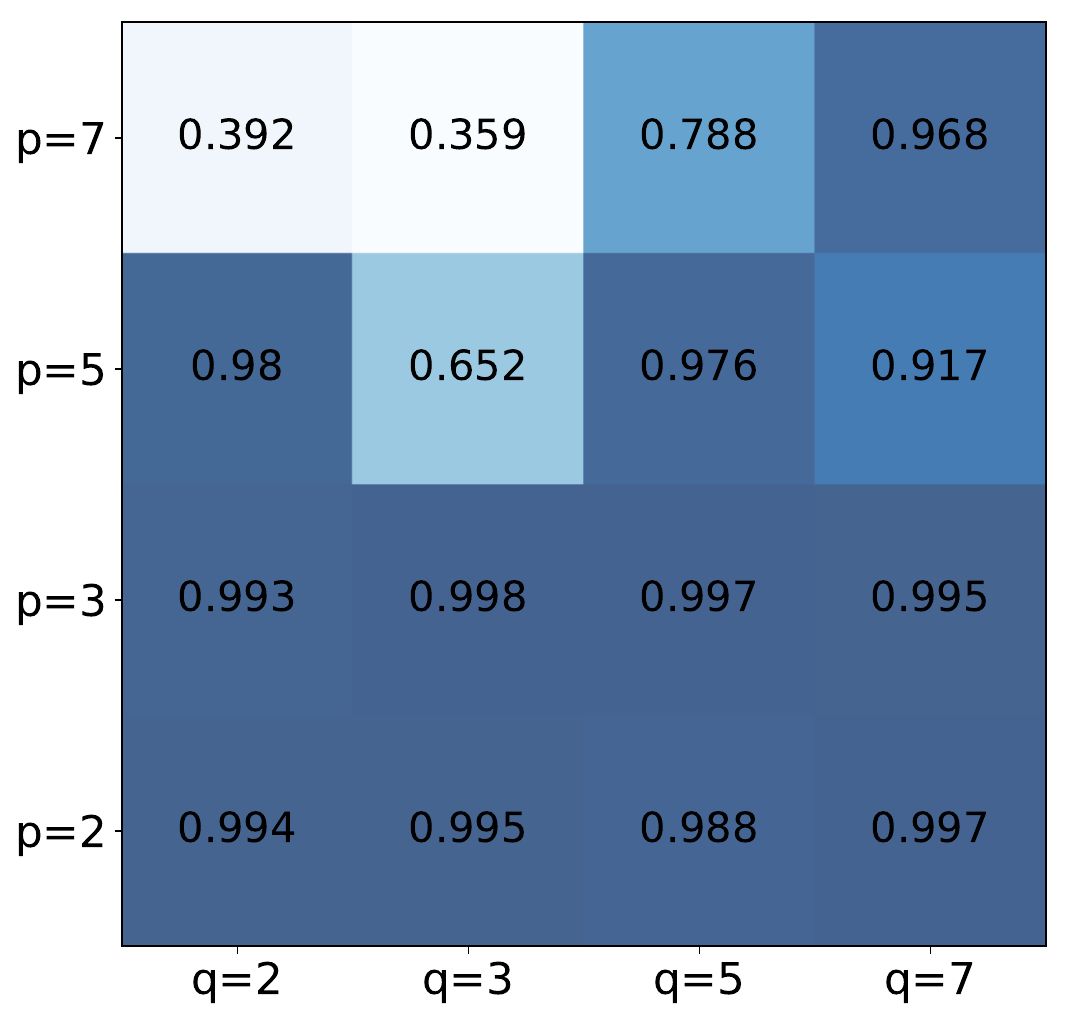}
    }
    \caption{Influence of the generator's depth $p$ and the discriminator's depth $q$ on the average recall of the estimators $\mu_{{\theta}_n}$ w.r.t.~$\mu^\star$, with $n=5000$.}
    \label{fig:increasing_both_gen_and_disc_overall_perf_recall}
\end{figure}

We end this subsection by pointing out a recurring observation across different experiments. In Figure \ref{fig:increasing_both_gen_and_disc_asymptotic_setting_recall} and Figure \ref{fig:increasing_both_gen_and_disc_overall_perf_recall}, we notice, as already stressed, that the average recall of the estimators is prone to decrease when the generator's depth $p$ increases. On the opposite, the average recall increases when the discriminator's depth $q$ increases. This is interesting because the recall metric is a good proxy for a stabilized training, insofar as a high recall means the absence of mode collapse. This is also confirmed in Figure \ref{fig:examples_approximating_9mixtures_gaussians_densities}, which compares two densities: in Figure \ref{fig:5a}, the discriminator has a small capacity ($q=3$) and the generator a large capacity ($p=7$), whereas in Figure \ref{fig:5b}, the discriminator has a large capacity ($q=7$) and the generator a small capacity ($p=3$). We observe that the first WGAN architecture behaves poorly compared to the second one. We therefore conclude that larger discriminators seem to bring some stability in the training of WGANS both in the asymptotic and finite sample regimes.

\begin{figure}[H]
    \subfloat[$p=7$ and $q=3$.\label{fig:D2G7} \label{fig:5a}]
    {
        \includegraphics[width=0.4\linewidth]{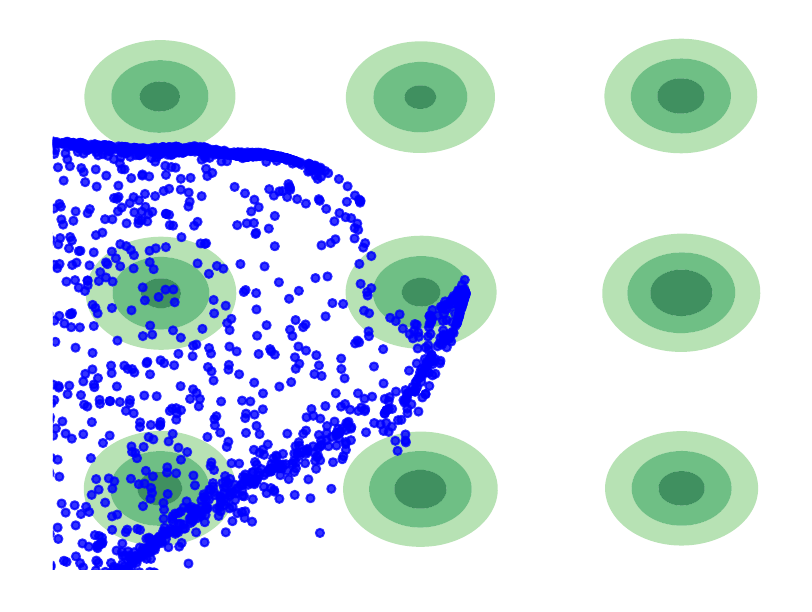}
    }\hspace{2cm}
    \subfloat[$p=3$ and $q=7$.\label{fig:D7G7}\label{fig:5b}]
    {
        \includegraphics[width=0.4\linewidth]{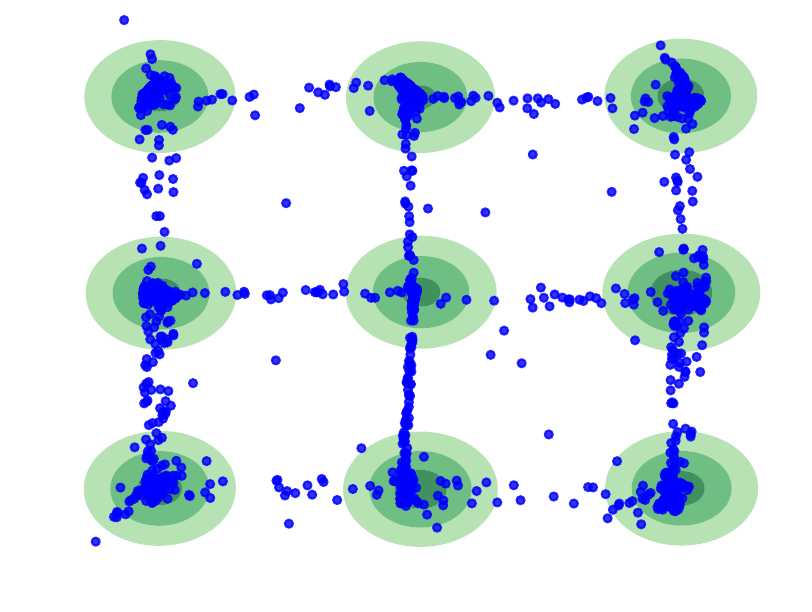}
    }
    \caption{True distribution $\mu^\star$ (mixture of $K=9$ bivariate Gaussian densities, green circles) and 2000 data points sampled from the generator $\mu_{\bar{\theta}}$ (blue dots).} \label{fig:examples_approximating_9mixtures_gaussians_densities}
\end{figure}
\subsection{Real-world experiments}
In this subsection, we further illustrate the impact of the generator's and the discriminator's capacities on two high-dimensional datasets, namely MNIST \citep{lecun98gradientbasedlearning} and Fashion-MNIST \citep{xiao2017online}. MNIST contains images in $\mathds{R}^{28\times 28}$ with 10 classes representing the digits. Fashion-MNIST is a 10-class dataset of images in $\mathds{R}^{28\times 28}$, with slightly more complex shapes than MNIST. Both datasets have a training set of 60,000 examples.

To measure the performance of WGANs when dealing with high-dimensional applications such as image generation, \citet{brock2018large} have advocated that embedding images into a feature space with a pre-trained convolutional classifier provides more meaningful information. Therefore, in order to assess the quality of the generator $\mu_{\hat{\theta}_n}$, we sample images both from the empirical measure $\mu_n$ and from the distribution $\mu_{\hat{\theta}_n}$. Then, instead of computing the Wasserstein distance directly between these two samples, we use as a substitute their embeddings output by an external classifier and compute the Wasserstein between the two new collections. Such a transformation is also done, for example, in \citet{kynkaanniemi2019improved}. Practically speaking, for any pair of images $(a,b)$, this operation amounts to using the Euclidean distance $\|\phi(a) - \phi(b)\|$ in the Wasserstein criterion, where $\phi$ is a pre-softmax layer of a supervised classifier, trained specifically on the datasets MNIST and Fashion-MNIST.

For these two datasets, as usual, we use generators of the form \eqref{eq:def_generators} and discriminators of the form \eqref{eq:def_discriminators}, and plot the performance of $\mu_{\hat{\theta}_n}$ as a function of both $p$ and $q$. The results of Figure \ref{fig:increasing_both_gen_and_disc_mnist_and_fashion_mnist_emd} confirm the fact that the worst results are achieved for generators with a large depth $p$ combined with discriminators with a small depth $q$. They also corroborate the previous observations that larger discriminators are preferred.

\begin{figure}[H]
    \centering
    \subfloat[$\sup_{\theta_n \in \hat{\Theta}_n} \  d_{\text{Lip}_1}(\mu_n, \mu_{{\theta}_n})$, MNIST dataset.]
    {
        \includegraphics[width=0.36\linewidth]{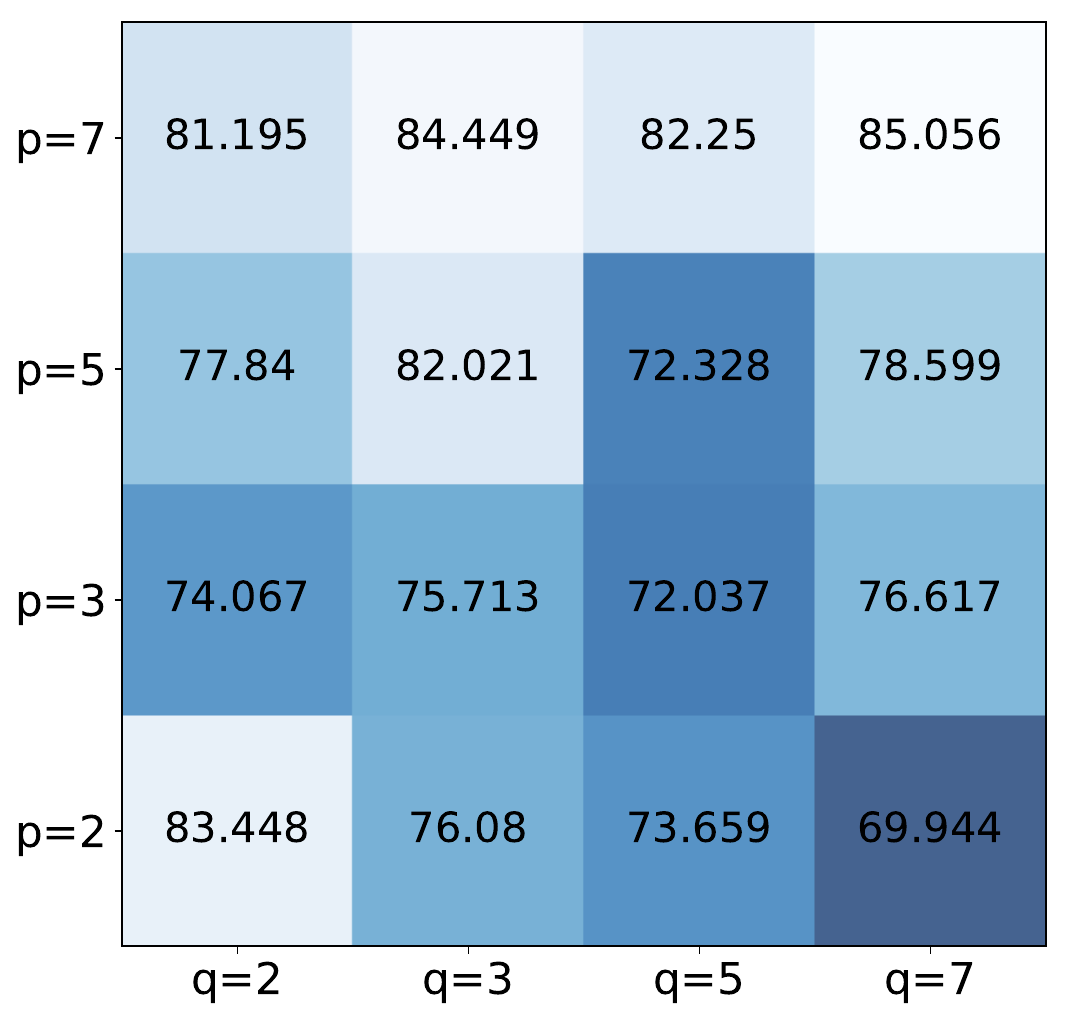}
    }
    \hspace{2.cm}
    \subfloat[$\sup_{\theta_n \in \hat{\Theta}_n} d_{\text{Lip}_1}(\mu_n, \mu_{{\theta}_n})$, FMNIST dataset.]
    {
        \includegraphics[width=0.36\linewidth]{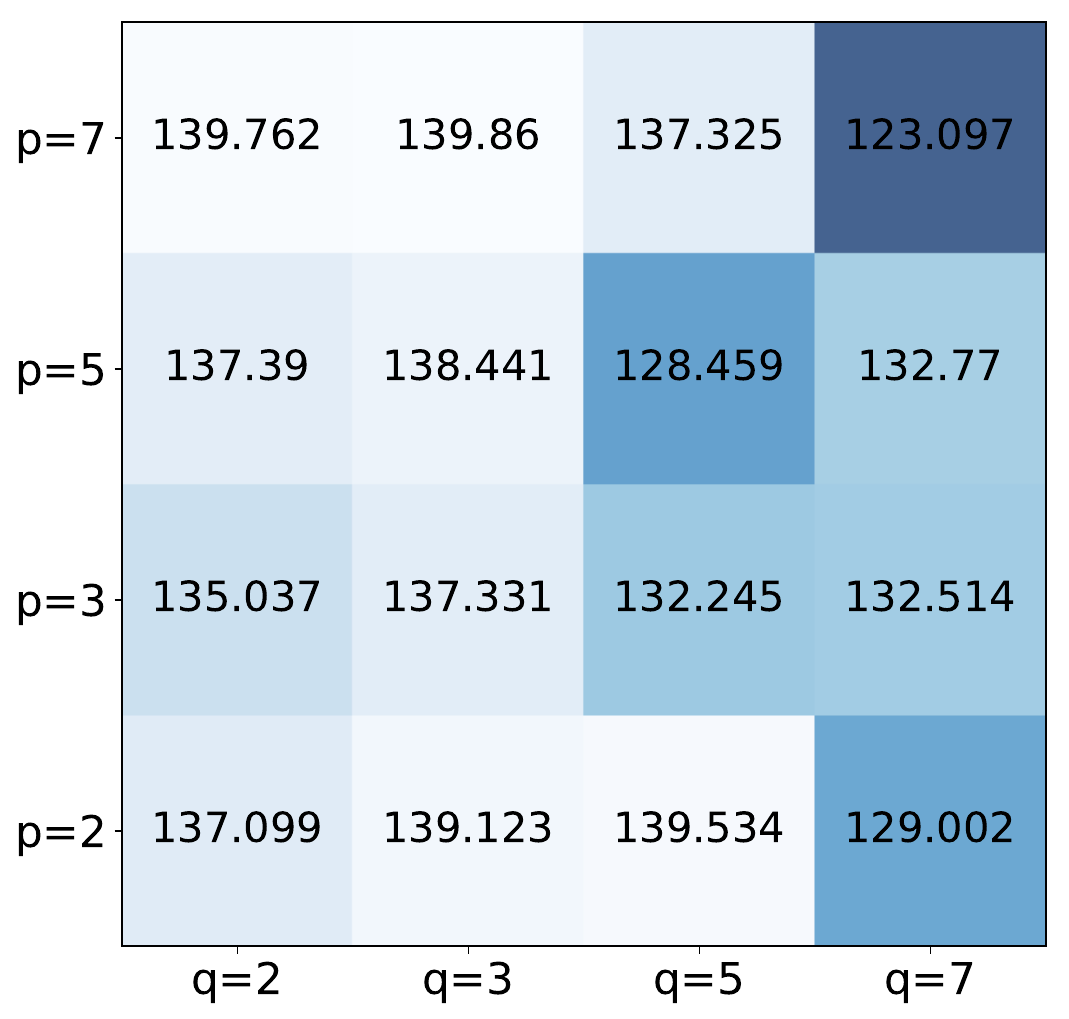}
    }
    \caption{Influence of the generator's depth $p$ and the discriminator's depth $q$ on the maximal Wasserstein distance $\sup_{\theta_n \in \hat{\Theta}_n} \ d_{\text{Lip}_1}(\mu_n, \mu_{{\theta}_n})$ for the MNIST and F-MNIST datasets.}
    \label{fig:increasing_both_gen_and_disc_mnist_and_fashion_mnist_emd}
\end{figure}
\newpage
\acks{We thank the referee and the Editor for their comments on a first version of the article. We also thank J\'er\'emie Bigot (Universit\'e de Bordeaux), Arnak Dalalyan (ENSAE), Flavian Vasile (Criteo AI Lab), and Cl\'ement Calauzenes (Criteo AI Lab) for stimulating discussions and insightful suggestions.}

\bibliography{adversarial_training}

\begin{thebibliography}{62}
\providecommand{\natexlab}[1]{#1}
\providecommand{\url}[1]{\texttt{#1}}
\expandafter\ifx\csname urlstyle\endcsname\relax
  \providecommand{\doi}[1]{doi: #1}\else
  \providecommand{\doi}{doi: \begingroup \urlstyle{rm}\Url}\fi

\bibitem[Acharya et~al.(2018)Acharya, Huang, Paudel, and
  Van~Gool]{acharya2018towards}
D.~Acharya, Z.~Huang, D.P. Paudel, and L.~Van~Gool.
\newblock Towards high resolution video generation with progressive growing of
  sliced {W}asserstein {GAN}s.
\newblock \emph{arXiv.1810.02419}, 2018.

\bibitem[Anil et~al.(2019)Anil, Lucas, and Grosse]{Anil2018SortingOL}
C.~Anil, J.~Lucas, and R.~Grosse.
\newblock Sorting out {L}ipschitz function approximation.
\newblock In K.~Chaudhuri and R.~Salakhutdinov, editors, \emph{Proceedings of
  the 36th International Conference on Machine Learning}, volume~97, pages
  291--301. PMLR, 2019.

\bibitem[Arjovsky and Bottou(2017)]{ArBo17}
M.~Arjovsky and L.~Bottou.
\newblock Towards principled methods for training generative adversarial
  networks.
\newblock In \emph{International Conference on Learning Representations}, 2017.

\bibitem[Arjovsky et~al.(2017)Arjovsky, Chintala, and
  Bottou]{arjovsky2017wasserstein}
M.~Arjovsky, S.~Chintala, and L.~Bottou.
\newblock Wasserstein generative adversarial networks.
\newblock In D.~Precup and Y.W. Teh, editors, \emph{Proceedings of the 34th
  International Conference on Machine Learning}, volume~70, pages 214--223.
  PMLR, 2017.

\bibitem[Arora et~al.(2018)Arora, Basu, Mianjy, and
  Mukherjee]{arora2018understanding}
R.~Arora, A.~Basu, P.~Mianjy, and A.~Mukherjee.
\newblock Understanding deep neural networks with rectified linear units.
\newblock In \emph{International Conference on Learning Representations}, 2018.

\bibitem[Arora et~al.(2017)Arora, Ge, Liang, Ma, and Zhang]{Arora0LMZ17}
S.~Arora, R.~Ge, Y.~Liang, T.~Ma, and Y.~Zhang.
\newblock Generalization and equilibrium in generative adversarial nets
  ({GAN}s).
\newblock In D.~Precup and Y.W. Teh, editors, \emph{Proceedings of the 34th
  International Conference on Machine Learning}, volume~70, pages 224--232.
  PMLR, 2017.

\bibitem[Belomestny et~al.(2021)Belomestny, Moulines, Naumov, Puchkin, and
  Samsonov]{Moulines}
D.~Belomestny, E.~Moulines, A.~Naumov, N.~Puchkin, and S.~Samsonov.
\newblock Rates of convergence for density estimation with {GAN}s.
\newblock \emph{arXiv:2102.00199}, 2021.

\bibitem[Biau et~al.(2020)Biau, Cadre, Sangnier, and Tanielian]{biau2018some}
G.~Biau, B.~Cadre, M.~Sangnier, and U.~Tanielian.
\newblock Some theoretical properties of {GAN}s.
\newblock \emph{The Annals of Statistics}, 48:\penalty0 1539--1566, 2020.

\bibitem[Bj{\"o}rck and Bowie(1971)]{bjorck1971iterative}
A.~Bj{\"o}rck and C.~Bowie.
\newblock An iterative algorithm for computing the best estimate of an
  orthogonal matrix.
\newblock \emph{SIAM Journal on Numerical Analysis}, 8:\penalty0 358--364,
  1971.

\bibitem[Brock et~al.(2019)Brock, Donahue, and Simonyan]{brock2018large}
A.~Brock, J.~Donahue, and K.~Simonyan.
\newblock Large scale {GAN} training for high fidelity natural image synthesis.
\newblock In \emph{International Conference on Learning Representations}, 2019.

\bibitem[Chernodub and Nowicki(2016)]{chernodub2016norm}
A.~Chernodub and D.~Nowicki.
\newblock Norm-preserving {O}rthogonal {P}ermutation {L}inear {U}nit activation
  functions ({OPLU}).
\newblock \emph{arXiv.1604.02313}, 2016.

\bibitem[Cybenko(1989)]{cybenko1989approximation}
G.~Cybenko.
\newblock Approximation by superpositions of a sigmoidal function.
\newblock \emph{Mathematics of Control, Signals and Systems}, 2:\penalty0
  303--314, 1989.

\bibitem[Dudley(2004)]{dudley_2002}
R.M. Dudley.
\newblock \emph{Real Analysis and Probability}.
\newblock Cambridge University Press, Cambridge, 2 edition, 2004.

\bibitem[Fedus et~al.(2018)Fedus, Goodfellow, and Dai]{MaskGANs}
W.~Fedus, I.~Goodfellow, and A.M. Dai.
\newblock Mask{GAN}: Better text generation via filling in the \_\_\_\_.
\newblock In \emph{International Conference on Learning Representations}, 2018.

\bibitem[Flamary and Courty(2017)]{flamary2017pot}
R.~Flamary and N.~Courty.
\newblock {POT}: {P}ython {O}ptimal {T}ransport library, 2017.
\newblock URL \url{https://github.com/rflamary/POT}.

\bibitem[Fournier and Guillin(2015)]{fournier2015rate}
N.~Fournier and A.~Guillin.
\newblock On the rate of convergence in {W}asserstein distance of the empirical
  measure.
\newblock \emph{Probability Theory and Related Fields}, 162:\penalty0 707--738,
  2015.

\bibitem[Givens and Shortt(1984)]{givens1984class}
C.R. Givens and R.M. Shortt.
\newblock A class of {W}asserstein metrics for probability distributions.
\newblock \emph{Michigan Mathematical Journal}, 31:\penalty0 231--240, 1984.

\bibitem[Glorot et~al.(2011)Glorot, Bordes, and Bengio]{glorot2011deep}
X.~Glorot, A.~Bordes, and Y.~Bengio.
\newblock Deep sparse rectifier neural networks.
\newblock In G.~Gordon, D.~Dunson, and M.~Dudík, editors, \emph{Proceedings of
  the Fourteenth International Conference on Artificial Intelligence and
  Statistics}, volume~15, pages 315--323. PMLR, 2011.

\bibitem[Goodfellow et~al.(2014)Goodfellow, Pouget-Abadie, Mirza, Xu,
  Warde-Farley, Ozair, Courville, and Bengio]{GANs}
I.~Goodfellow, J.~Pouget-Abadie, M.~Mirza, B.~Xu, D.~Warde-Farley, S.~Ozair,
  A.~Courville, and J.~Bengio.
\newblock Generative adversarial nets.
\newblock In Z.~Ghahramani, M.~Welling, C.~Cortes, N.D. Lawrence, and K.Q.
  Weinberger, editors, \emph{Advances in Neural Information Processing Systems
  27}, pages 2672--2680. Curran Associates, Inc., 2014.

\bibitem[Gulrajani et~al.(2017)Gulrajani, Ahmed, Arjovsky, Dumoulin, and
  Courville]{gulrajani2017improved}
I.~Gulrajani, F.~Ahmed, M.~Arjovsky, V.~Dumoulin, and A.C. Courville.
\newblock Improved training of {W}asserstein {GAN}s.
\newblock In I.~Guyon, U.~von Luxburg, S.~Bengio, H.~Wallach, R.~Fergus,
  S.~Vishwanathan, and R.~Garnett, editors, \emph{Advances in Neural
  Information Processing Systems 30}, pages 5767--5777. Curran Associates,
  Inc., 2017.

\bibitem[Hornik(1991)]{hornik1991approximation}
K.~Hornik.
\newblock Approximation capabilities of multilayer feedforward networks.
\newblock \emph{Neural Networks}, 4:\penalty0 251--257, 1991.

\bibitem[Hornik et~al.(1989)Hornik, Stinchcombe, and
  White]{hornik1989multilayer}
K.~Hornik, M.~Stinchcombe, and H.~White.
\newblock Multilayer feedforward networks are universal approximators.
\newblock \emph{Neural Networks}, 2:\penalty0 359--366, 1989.

\bibitem[Huster et~al.(2019)Huster, Chiang, and
  Chadha]{Huster2018LimitationsOT}
T.~Huster, C.-Y.J. Chiang, and R.~Chadha.
\newblock Limitations of the {L}ipschitz constant as a defense against
  adversarial examples.
\newblock In C.~Alzate, A.~Monreale, H.~Assem, A.~Bifet, T.S. Buda,
  B.~Caglayan, B.~Drury, E.~Garc\'{\i}a-Mart\'{\i}n, R.~Gavald\`a,
  I.~Koprinska, S.~Kramer, N.~Lavesson, M.~Madden, I.~Molloy, M.-I. Nicolae,
  and M.~Sinn, editors, \emph{ECML PKDD 2018 Workshops}, pages 16--29.
  Springer, 2019.

\bibitem[Jin et~al.(2019)Jin, Netrapalli, Ge, Kakade, and Jordan]{jin2019short}
C.~Jin, P.~Netrapalli, R.~Ge, S.M. Kakade, and M.~Jordan.
\newblock A short note on concentration inequalities for random vectors with
  sub{G}aussian norm.
\newblock \emph{arXiv.1902.03736}, 2019.

\bibitem[Kantorovich and Rubinstein(1958)]{kantorovich1958space}
L.V. Kantorovich and G.S. Rubinstein.
\newblock On a space of completely additive functions.
\newblock \emph{Vestnik Leningrad University Mathematics}, 13:\penalty0 52--59,
  1958.

\bibitem[Karras et~al.(2018)Karras, Aila, Laine, and
  Lehtinen]{karras2017progressive}
T.~Karras, T.~Aila, S.~Laine, and J.~Lehtinen.
\newblock Progressive growing of {GAN}s for improved quality, stability, and
  variation.
\newblock In \emph{International Conference on Learning Representations}, 2018.

\bibitem[Karras et~al.(2019)Karras, Laine, and Aila]{karras2018style}
T.~Karras, S.~Laine, and T.~Aila.
\newblock A style-based generator architecture for generative adversarial
  networks.
\newblock In \emph{Proceedings of the IEEE Conference on Computer Vision and
  Pattern Recognition}, pages 4401--4410, 2019.

\bibitem[Kodali et~al.(2017)Kodali, Abernethy, Hays, and
  Kira]{kodali2017convergence}
N.~Kodali, J.~Abernethy, J.~Hays, and Z.~Kira.
\newblock On convergence and stability of {GAN}s.
\newblock \emph{arXiv.1705.07215}, 2017.

\bibitem[Kontorovich(2014)]{kontorovich2014concentration}
A.~Kontorovich.
\newblock Concentration in unbounded metric spaces and algorithmic stability.
\newblock In E.P. Xing and T.~Jebara, editors, \emph{Proceedings of the 31st
  International Conference on Machine Learning}, volume~32, pages 28--36. PMLR,
  2014.

\bibitem[Kynk\"{a}\"{a}nniemi et~al.(2019)Kynk\"{a}\"{a}nniemi, Karras, Laine,
  Lehtinen, and Aila]{kynkaanniemi2019improved}
T.~Kynk\"{a}\"{a}nniemi, T.~Karras, S.~Laine, J.~Lehtinen, and T.~Aila.
\newblock Improved precision and recall metric for assessing generative models.
\newblock In H.~Wallach, H.~Larochelle, A.~Beygelzimer, F.~d'Alch\'{e} Buc,
  E.~Fox, and R.~Garnett, editors, \emph{Advances in Neural Information
  Processing Systems 32}, pages 3927--3936. Curran Associates, Inc., 2019.

\bibitem[LeCun et~al.(1998)LeCun, Bottou, Bengio, and
  Haffner]{lecun98gradientbasedlearning}
Y.~LeCun, L.~Bottou, Y.~Bengio, and P.~Haffner.
\newblock Gradient-based learning applied to document recognition.
\newblock In \emph{Proceedings of the IEEE}, pages 2278--2324, 1998.

\bibitem[Ledig et~al.(2017)Ledig, Theis, Husz{\'a}r, Caballero, Cunningham,
  Acosta, Aitken, Tejani, Totz, Wang, and Shi]{ledig2017photo}
C.~Ledig, L.~Theis, F.~Husz{\'a}r, J.~Caballero, A.~Cunningham, A.~Acosta,
  A.~Aitken, A.~Tejani, J.~Totz, Z.~Wang, and W.~Shi.
\newblock Photo-realistic single image super-resolution using a generative
  adversarial network.
\newblock In \emph{Proceedings of the IEEE Conference on Computer Vision and
  Pattern Recognition}, pages 4681--4690, 2017.

\bibitem[Li et~al.(2017)Li, Chang, Cheng, Yang, and Poczos]{li2017mmd}
C.-L. Li, W.-C. Chang, Y.~Cheng, Y.~Yang, and B.~Poczos.
\newblock {MMD} {GAN}: {T}owards deeper understanding of moment matching
  network.
\newblock In I.~Guyon, U.~von Luxburg, S.~Bengio, H.~Wallach, R.~Fergus,
  S.~Vishwanathan, and R.~Garnett, editors, \emph{Advances in Neural
  Information Processing Systems 30}, pages 2203--2213. Curran Associates,
  Inc., 2017.

\bibitem[Liang(2018)]{liang2018well}
T.~Liang.
\newblock On how well generative adversarial networks learn densities:
  Nonparametric and parametric results.
\newblock \emph{arXiv.1811.03179}, 2018.

\bibitem[Liu et~al.(2017)Liu, Bousquet, and Chaudhuri]{LiBoCh07}
S.~Liu, O.~Bousquet, and K.~Chaudhuri.
\newblock Approximation and convergence properties of generative adversarial
  learning.
\newblock In I.~Guyon, U.~von Luxburg, S.~Bengio, H.~Wallach, R.~Fergus,
  S.~Vishwanathan, and R.~Garnett, editors, \emph{Advances in Neural
  Information Processing Systems 30}, pages 5551--5559. Curran Associates,
  Inc., 2017.

\bibitem[Luise et~al.(2020)Luise, Pontil, and
  Ciliberto]{luise2020generalization}
G.~Luise, M.~Pontil, and C.~Ciliberto.
\newblock Generalization properties of optimal transport {GAN}s with latent
  distribution learning.
\newblock \emph{arXiv:2007.14641}, 2020.

\bibitem[McDiarmid(1989)]{mcdiarmid_1989}
C.~McDiarmid.
\newblock On the method of bounded differences.
\newblock In J.~Siemons, editor, \emph{Surveys in Combinatorics}, London
  Mathematical Society Lecture Note Series 141, pages 148--188. Cambridge
  University Press, Cambridge, 1989.

\bibitem[Metz et~al.(2016)Metz, Poole, Pfau, and
  Sohl-Dickstein]{metz2016unrolled}
L.~Metz, B.~Poole, D.~Pfau, and J.~Sohl-Dickstein.
\newblock Unrolled generative adversarial networks.
\newblock \emph{arXiv.1611.02163}, 2016.

\bibitem[Miyato et~al.(2018)Miyato, Kataoka, Koyama, and
  Yoshida]{spectral_normGANs}
T.~Miyato, T.~Kataoka, M.~Koyama, and Y.~Yoshida.
\newblock Spectral normalization for generative adversarial networks.
\newblock In \emph{International Conference on Learning Representations}, 2018.

\bibitem[Mogren(2016)]{mogren2016c}
O.~Mogren.
\newblock {C-RNN-GAN}: {C}ontinuous recurrent neural networks with adversarial
  training.
\newblock \emph{arXiv.1611.09904}, 2016.

\bibitem[Mont\'ufar et~al.(2014)Mont\'ufar, Pascanu, Cho, and
  Bengio]{montufar2014number}
G.~Mont\'ufar, R.~Pascanu, K.~Cho, and Y.~Bengio.
\newblock On the number of linear regions of deep neural networks.
\newblock In Z.~Ghahramani, M.~Welling, C.~Cortes, N.D. Lawrence, and K.Q.
  Weinberger, editors, \emph{Advances in Neural Information Processing Systems
  27}, pages 2924--2932. Curran Associates, Inc., 2014.

\bibitem[M{\"u}ller(1997)]{IPMsMuller}
A.~M{\"u}ller.
\newblock Integral probability metrics and their generating classes of
  functions.
\newblock \emph{Advances in Applied Probability}, 29:\penalty0 429--443, 1997.

\bibitem[Nowozin et~al.(2016)Nowozin, Cseke, and Tomioka]{NoCsTo16}
S.~Nowozin, B.~Cseke, and R.~Tomioka.
\newblock f-{GAN}: {T}raining generative neural samplers using variational
  divergence minimization.
\newblock In D.D. Lee, M.~Sugiyama, U.~von Luxburg, I.~Guyon, and R.~Garnett,
  editors, \emph{Advances in Neural Information Processing Systems 29}, pages
  271--279. Curran Associates, Inc., 2016.

\bibitem[O'Searcoid(2006)]{osearcoid2006metric}
M.~O'Searcoid.
\newblock \emph{Metric Spaces}.
\newblock Springer, Dublin, 2006.

\bibitem[Pascanu et~al.(2013)Pascanu, Mont\'ufar, and
  Bengio]{pascanu2013number}
R.~Pascanu, G.~Mont\'ufar, and Y.~Bengio.
\newblock On the number of response regions of deep feed forward networks with
  piece-wise linear activations.
\newblock In \emph{International Conference on Learning Representations}, 2013.

\bibitem[Petzka et~al.(2018)Petzka, Fischer, and
  Lukovnikov]{petzka2018regularization}
H.~Petzka, A.~Fischer, and D.~Lukovnikov.
\newblock On the regularization of {W}asserstein {GAN}s.
\newblock In \emph{International Conference on Learning Representations}, 2018.

\bibitem[Radford et~al.(2015)Radford, Metz, and
  Chintala]{radford2015unsupervised}
A.~Radford, L.~Metz, and S.~Chintala.
\newblock Unsupervised representation learning with deep convolutional
  generative adversarial networks.
\newblock \emph{arXiv:1511.06434}, 2015.

\bibitem[Roth et~al.(2017)Roth, Lucchi, Nowozin, and
  Hofmann]{roth2017stabilizing}
K.~Roth, A.~Lucchi, S.~Nowozin, and T.~Hofmann.
\newblock Stabilizing training of generative adversarial networks through
  regularization.
\newblock In I.~Guyon, U.~von Luxburg, S.~Bengio, H.~Wallach, R.~Fergus,
  S.~Vishwanathan, and R.~Garnett, editors, \emph{Advances in Neural
  Information Processing Systems 30}, pages 2018--2028. Curran Associates,
  Inc., 2017.

\bibitem[Salimans et~al.(2016)Salimans, Goodfellow, Zaremba, Cheung, Radford,
  and Chen]{salimans2016improved}
T.~Salimans, I.~Goodfellow, W.~Zaremba, V.~Cheung, A.~Radford, and X.~Chen.
\newblock Improved techniques for training {GAN}s.
\newblock In D.D. Lee, M.~Sugiyama, U.~von Luxburg, I.~Guyon, and R.~Garnett,
  editors, \emph{Advances in Neural Information Processing Systems 29}, pages
  2234--2242. Curran Associates, Inc., 2016.

\bibitem[Schreuder et~al.(2021)Schreuder, Brunel, and
  Dalalyan]{schreuder2021statistical}
N.~Schreuder, V.-E. Brunel, and A.~Dalalyan.
\newblock Statistical guarantees for generative models without domination.
\newblock In V.~Feldman, K.~Ligett, and S.~Sabato, editors, \emph{Proceedings
  of the 32nd International Conference on Algorithmic Learning Theory}, pages
  1051--1071. PMLR, 2021.

\bibitem[Serra et~al.(2018)Serra, Tjandraatmadja, and
  Ramalingam]{serra2018bounding}
T.~Serra, C.~Tjandraatmadja, and S.~Ramalingam.
\newblock Bounding and counting linear regions of deep neural networks.
\newblock In \emph{International Conference on Machine Learning}, pages
  4565--4573, 2018.

\bibitem[Singh et~al.(2018)Singh, Uppal, Li, Li, Zaheer, and
  Poczos]{nonparametric2018singh}
S.~Singh, A.~Uppal, B.~Li, C.-L. Li, M.~Zaheer, and B.~Poczos.
\newblock Nonparametric density estimation under adversarial losses.
\newblock In S.~Bengio, H.~Wallach, H.~Larochelle, K.~Grauman, N.~Cesa-Bianchi,
  and R.~Garnett, editors, \emph{Advances in Neural Information Processing
  Systems 31}, pages 10225--10236. Curran Associates, Inc., 2018.

\bibitem[Uppal et~al.(2019)Uppal, Singh, and Poczos]{nonparametric2019wallach}
A.~Uppal, S.~Singh, and B.~Poczos.
\newblock Nonparametric density estimation and convergence rates for {GAN}s
  under {B}esov {IPM} losses.
\newblock In H.~Wallach, H.~Larochelle, A.~Beygelzimer, F.~d'Alch\'{e} Buc,
  E.~Fox, and R.~Garnett, editors, \emph{Advances in Neural Information
  Processing Systems 32}, pages 9089--9100. Curran Associates, Inc., 2019.

\bibitem[van Handel(2016)]{vanhandel2016probability}
R.~van Handel.
\newblock \emph{Probability in High Dimension}.
\newblock APC 550 Lecture Notes, Princeton University, 2016.

\bibitem[Vershynin(2018)]{vershynin2018high}
R.~Vershynin.
\newblock \emph{High-Dimensional Probability: An Introduction with Applications
  in Data Science}.
\newblock Cambridge University Press, Cambridge, 2018.

\bibitem[Villani(2008)]{villani2008optimal}
C.~Villani.
\newblock \emph{Optimal Transport: Old and New}.
\newblock Springer, Berlin, 2008.

\bibitem[Wei et~al.(2018)Wei, Gong, Liu, Lu, and Wang]{wei2018improving}
X.~Wei, B.~Gong, Z.~Liu, W.~Lu, and L.~Wang.
\newblock Improving the improved training of {W}asserstein {GAN}s: A
  consistency term and its dual effect.
\newblock In \emph{International Conference on Learning Representations}, 2018.

\bibitem[Xiao et~al.(2017)Xiao, Rasul, and Vollgraf]{xiao2017online}
H.~Xiao, K.~Rasul, and R.~Vollgraf.
\newblock Fashion-{MNIST}: A novel image dataset for benchmarking machine
  learning algorithms.
\newblock \emph{arXiv:1708.07747}, 2017.

\bibitem[Yu et~al.(2017)Yu, Zhang, Wang, and Yu]{SeqGANs}
L.~Yu, W.~Zhang, J.~Wang, and Y.~Yu.
\newblock Seq{GAN}: {S}equence generative adversarial nets with policy
  gradient.
\newblock In \emph{Proceedings of the Thirty-First AAAI Conference on
  Artificial Intelligence}, page 2852–2858. AAAI Press, 2017.

\bibitem[Zhang et~al.(2018)Zhang, Liu, Zhou, Xu, and
  He]{zhang2018discrimination}
P.~Zhang, Q.~Liu, D.~Zhou, T.~Xu, and X.~He.
\newblock On the discriminative-generalization tradeoff in {GAN}s.
\newblock In \emph{International Conference on Learning Representations}, 2018.

\bibitem[Zhao et~al.(2017)Zhao, Mathieu, and LeCun]{zhao2016energy}
J.~Zhao, M.~Mathieu, and Y.~LeCun.
\newblock Energy-based generative adversarial networks.
\newblock In \emph{International Conference on Learning Representations}, 2017.

\bibitem[Zhou et~al.(2019)Zhou, Liang, Song, Yu, Wang, Zhang, Yu, and
  Zhang]{zhou2019lipschitz}
Z.~Zhou, J.~Liang, Y.~Song, L.~Yu, H.~Wang, W.~Zhang, Y.~Yu, and Z.~Zhang.
\newblock Lipschitz generative adversarial nets.
\newblock In K.~Chaudhuri and R.~Salakhutdinov, editors, \emph{Proceedings of
  the 36th International Conference on Machine Learning}, volume~97, pages
  7584--7593. PMLR, 2019.

\end{thebibliography}

\newpage
\appendix
\section{}
\label{sec:proofs}
\subsection{Proof of Lemma \ref{lem:uniformly_lipschitz_neural_nets}}
    \label{appendix:lem_uniformly_lipschitz_neural_nets}
    We know that for each $\theta \in \Theta$, $G_\theta$ is a feed-forward neural network of the form \eqref{eq:def_generators}, which maps inputs $z \in \mathds{R}^d$ into $E\subset \mathds R^D$. In particular, for $z\in \mathds R^d$, $G_\theta(z) = f_p \circ \cdots \circ f_1(z)$, where $f_i(x) = \sigma(U_ix + b_i)$ for $i=1, \hdots, p-1$ ($\sigma$ is applied componentwise), and $f_p(x) = U_px + b_p$.

    Recall that the notation $\| \cdot \|$ (respectively, $\|\cdot\|_{\infty}$) means the Euclidean (respectively, the supremum) norm, with no specific mention of the underlying space on which it acts. For $(z, z') \in (\mathds{R}^d)^2$, we have
    \begin{align*}
        \|f_1(z) - f_1(z') \| &\leqslant \| U_1 z +b_1 - U_1 z' - b_1 \| \\
        & \quad \mbox{(since $\sigma$ is $1$-Lipschitz)}\\
        &= \|U_1(z-z') \| \\
        & \leqslant \|U_1\|_2 \ \|z-z'\| \\
        & \leqslant K_1 \|z-z'\| \\
        & \quad \mbox{(by Assumption \ref{ass:compactness})}.
    \end{align*}
    Repeating this for $i = 2, \hdots, p$, we thus have, for all $(z,z') \in (\mathds R^{d})^2$, $\|G_\theta(z) - G_\theta(z') \| \leqslant K_1^p \|z-z'\|$. We conclude that, for each $\theta \in \Theta$, the function $G_\theta$ is $K_1^p$-Lipschitz on $\mathds R^{d}$.

    Let us now prove that $\mathscr D \subseteq \text{Lip}_1$. Fix $D_{\alpha} \in \mathscr D$, $\alpha \in \Lambda$. According to \eqref{eq:def_discriminators}, we have, for $x\in E$, $D_{\alpha}(x) = f_q \circ \cdots \circ f_1(x)$, where $f_i(t)= \Tilde\sigma(V_it + c_i)$ for $i=1, \hdots, q-1$ ($\Tilde \sigma$ is applied on pairs of components), and $f_{q}(t)=V_qt+c_q$.

    Consequently, for $(x,y) \in E^2$,
    \begin{align*}
        \|f_1(x) - f_1(y) \|_\infty & \leqslant \| V_1 x - V_1 y \|_\infty \\
        & \quad \mbox{(since $\Tilde{\sigma}$ is $1$-Lipschitz)}\\
        &= \|V_1(x-y) \|_{\infty} \\
        & \leqslant \|V_1\|_{2, \infty} \ \|x-y\| \\
        &\leqslant \|x-y\| \\
        & \quad \mbox{(by Assumption \ref{ass:compactness})}.
    \end{align*}
    Thus,
 \begin{align*}
        \|f_2 \circ f_1(x) - f_2 \circ f_1(y) \|_\infty &\leqslant \| V_2f_1(x)  - V_2 f_1(y) \|_\infty \\
        & \quad \mbox{(since $\Tilde{\sigma}$ is $1$-Lipschitz)}\\
        & \leqslant \|V_2\|_\infty \ \|f_1(x)-f_1(y)\|_\infty \\
        &\leqslant \|f_1(x)-f_1(y)\|_\infty \\
        & \quad \mbox{(by Assumption \ref{ass:compactness})}\\
        & \leqslant \|x-y\|.
    \end{align*}
 Repeating this, we conclude that, for each $\alpha \in \Lambda$ and all $(x, y) \in E^2$, $|D_\alpha(x) - D_\alpha(y)| \leqslant \|x-y\|$, which is the desired result.
    \subsection{Proof of Proposition \ref{prop:neural_nets_tightness}}
    We first prove that the function $\Theta \ni \theta \mapsto \mu_\theta$ is continuous with respect to the weak topology in $P_1(E)$. Let $G_{\theta}$ and $G_{\theta'}$ be two elements of $\mathscr G$, with $(\theta, \theta') \in \Theta^2$. Using \eqref{eq:def_generators}, we write $G_\theta(z) = f_p \circ \cdots \circ f_1(z)$ (respectively, $G_{\theta'}(z) = f'_p \circ \cdots \circ f'_1(z)$), where  $f_i(x) = \max(U_i x + b_i, 0)$ (respectively, $f'_i(x) = \max(U'_i x + b'_i, 0)$) for $i=1, \hdots, p-1$, and $f_p(x) = U_p x + b_p$ (respectively, $f'_p(x) = U'_p x + b'_p$).

    Clearly, for $z\in \mathds R^d$,
    \begin{align*}
            \|f_1(z) - f'_1(z) \| &\leqslant \|U_1 z +b_1 - U'_1 z - b'_1 \| \\
            &\leqslant \|(U_1 - U'_1) z\| + \|b_1 - b'_1\| \\
            &\leqslant \|U_1 - U'_1\|_2 \ \|z\| + \|b_1 - b'_1\|   \\
            &\leqslant (\|z\|+1) \|\theta - \theta'\|.
    \end{align*}
    Similarly, for any $i \in \{2, \hdots, p\}$ and any $x \in \mathds{R}^{u_i}$,
    \begin{equation*}
        \|f_i (x) - f'_i(x) \| \leqslant (\|x\| +1) \|\theta - \theta'\|.
    \end{equation*}
    Observe that
    \begin{align*}
        &\|G_\theta(z) - G_{\theta'}(z) \| \\
        & \ = \|f_p \circ \cdots \circ f_1(z) - f'_p \circ \cdots \circ f'_1(z)\| \\
        & \ \leqslant \|f_p \circ \cdots \circ f_1(z) - f_p \circ \cdots \circ f_2 \circ f'_1(z)\| + \cdots + \|f_p \circ f'_{p-1} \circ \cdots \circ f'_1(z) - f'_p \circ \cdots \circ f'_1(z)\|.
    \end{align*}
    As in the proof of Lemma \ref{lem:uniformly_lipschitz_neural_nets}, one shows that for any $i \in \{1,\hdots, p\}$, the function $f_p \circ \cdots \circ f_i$ is $K_1^{p-i+1}$-Lipschitz with respect to the Euclidean norm. Therefore,
    \begin{align*}
        &\|G_\theta(z) - G_{\theta'}(z) \| \\
        &\quad \leqslant K_1^{p-1} \|f_1(z)-f'_1(z)\| + \cdots + K_1^0 \|f_p \circ f'_{p-1} \circ \cdots \circ f'_1(z) - f'_p \circ \cdots \circ f'_1(z)\| \\
        &\quad \leqslant K_1^{p-1}(\|z\|+1) \|\theta-\theta'\| + \cdots + (\|f'_{p-1} \circ \cdots \circ f'_1(z)\|+1) \|\theta-\theta'\| \\
        &\quad \leqslant K_1^{p-1}(\|z\|+1) \|\theta-\theta'\| + \cdots + (K_1^{p-1} \|z\| + \|f'_{p-1} \circ \cdots \circ f'_1(0)\|+1) \|\theta-\theta'\|.
    \end{align*}
    Using the architecture of neural networks in \eqref{eq:def_generators}, a quick check shows that, for each $i \in \{1, \hdots, p\}$,
    \begin{equation*}
        \|f'_{i} \circ \cdots \circ f'_1(0)\| \leqslant \sum_{k=1}^i K_1^k.
    \end{equation*}
    We are led to
    \begin{equation}\label{eq:G_lipschitz_inequality}
    \|G_\theta(z) - G_{\theta'}(z) \| = (\ell_1 \|z\| + \ell_2) \|\theta-\theta'\|,
    \end{equation}
    where
    \begin{equation*}
        \ell_1 = pK_1^{p-1} \quad \mbox{and} \quad \ell_2 = \sum_{i=1}^{p-1} K_1^{p-(i+1)}\sum_{k=1}^i K_1^k + \sum_{i=0}^{p-1}K_1^i.
    \end{equation*}
    Denoting by $\nu$ the probability distribution of the sub-Gaussian random variable $Z$, we note that $\int_{\mathds{R}^d} (\ell_1 \|z\| + \ell_2) \nu({\rm d} z) < \infty$. Now, let $(\theta_k)$ be a sequence in $\Theta$ converging to $\theta \in \Theta$ with respect to the Euclidean norm. Clearly, for a given $z \in \mathds{R}^d$, by continuity of the function $\theta \mapsto G_\theta(z)$, we have $\underset{k \to \infty}{\lim} G_{\theta_k}(z) = G_\theta(z)$ and, for any $\varphi \in C_b(E)$, $\underset{k \to \infty}{\lim} \varphi(G_{\theta_k}(z)) = \varphi(G_\theta(z))$. Thus, by the dominated convergence theorem,
    \begin{equation}\label{abudhabi}
        \underset{k \to \infty}{\lim}\int_E \varphi(x) \mu_{\theta_k}({\rm d} x) = \underset{k \to \infty}{\lim} \int_{\mathds{R}^d} \varphi (G_{\theta_k}(z)) \nu({\rm d}z) = \int_{\mathds{R}^d} \varphi \big(G_{\theta}(z))\nu({\rm d}z) = \int_{E} \varphi(x) \mu_\theta({\rm d}x).
    \end{equation}
    This shows that the sequence $(\mu_{\theta_k})$ converges weakly to $\mu_\theta$. Besides, for an arbitrary $x_0$ in $E$, we have
    \begin{align*}
        &\underset{k \to \infty}{\limsup}\int_E \|x_0-x\|\mu_{\theta_k}({\rm d} x) \\
        & \quad = \underset{k \to \infty}{\limsup}\int_{\mathds{R}^d} \|x_0-G_{\theta_k}(z)\| \nu({\rm d}z) \\
        & \quad \leqslant \underset{k \to \infty}{\limsup}\int_{\mathds{R}^d} \big(\|G_{\theta_k}(z)-G_\theta(z)\| +  \|G_{\theta}(z)-x_0\|\big) \nu({\rm d} z) \\
        &\quad \leqslant \underset{k \to \infty}{\limsup}\int_{\mathds{R}^d} (\ell_1 \|z\| + \ell_2 ) \|\theta_k-\theta\| \nu({\rm d} z) + \int_{\mathds{R}^d} \|G_{\theta}(z)-x_0\| \nu({\rm d} z) \\
        &\qquad \mbox{(by inequality \eqref{eq:G_lipschitz_inequality})}.
              \end{align*}
    Consequently,
    \begin{equation*}
        \underset{k \to \infty}{\limsup}\int_E \|x_0-x\|\mu_{\theta_k}({\rm d} x) \leqslant \int_{\mathds{R}^d} \|G_{\theta}(z)-x_0\| \nu({\rm d} z) =\int_E \|x_0-x\|\mu_{\theta}({\rm d} x).
        \end{equation*}
    One proves with similar arguments that
    \begin{equation*}
         \underset{k \to \infty}{\liminf}\int_E \|x_0-x\|\mu_{\theta_k}({\rm d} x) \geqslant\int_E \|x_0-x\|\mu_{\theta}({\rm d} x).
        \end{equation*}
    Therefore, putting all the pieces together, we conclude that
    \begin{equation*}
    \underset{k \to \infty}{\lim}\int_E \|x_0-x\|\mu_{\theta_k}({\rm d} x)=\int_E \|x_0-x\|\mu_{\theta}({\rm d} x).
    \end{equation*}
    This, together with \eqref{abudhabi}, shows that the sequence $(\mu_{\theta_k})$ converges weakly to $\mu_\theta$ in $P_1(E)$, and, in turn, that the function $\Theta \ni \theta \mapsto \mu_\theta$ is continuous with respect to the weak topology in $P_1(E)$, as desired.

    The second assertion of the proposition follows upon noting that $\mathscr{P}$ is the image of the compact set $\Theta$ by a continuous function.
\subsection{Proof of Proposition \ref{cor:neural_distance}}\label{appendix:lem_neural_IPMs_are_distances}
To show the first statement, we are to exhibit a specific discriminator, say $\mathscr{D}_{\text{max}}$, such that, for all $(\mu, \nu) \in (\mathscr{P} \cup \{\mu^\star\})^2$, the identity $d_{\mathscr{D}_{\text{max}}}(\mu, \nu)=0$ implies $\mu =\nu$.

Let $\varepsilon >0$. According to Proposition \ref{prop:neural_nets_tightness}, under Assumption \ref{ass:compactness}, $\mathscr{P}$ is a compact subset of $P_1(E)$ with respect to the weak topology in $P_1(E)$. Let $x_0 \in E$ be arbitrary. For any $\mu \in \mathscr{P}$ there exists a compact $K_\mu \subseteq E$ such that $\int_{K_\mu^\complement} \|x_0-x\|\mu({\rm d}x) \leqslant \varepsilon/4$. Also, for any such $K_\mu$, the function $P_1(E)\ni \rho \mapsto \int_{K_\mu^\complement} \|x_0-x\| \rho({\rm d}x)$ is continuous. Therefore, there exists an open set $U_\mu \subseteq P_1(E)$ containing $\mu$ such that, for any $\rho \in U_\mu$, $\int_{K_\mu^\complement} \|x_0-x\| \rho({\rm d}x) \leqslant \varepsilon/2$.

The collection of open sets $\{ U_\mu : \mu \in \mathscr{P}\}$ forms an open cover of $\mathscr{P}$, from which we can extract, by compactness, a finite subcover $U_{\mu_1}, \dots, U_{\mu_n}$. Letting $K_1 = \cup_{i=1}^n K_{\mu_i}$, we deduce that, for all $\mu \in \mathscr{P}$, $\int_{K_1^\complement} \|x_0-x\| \mu({\rm d}x) \leqslant \varepsilon/2$. We conclude that there exists a compact $K \subseteq E$ and $x_0 \in K$ such that, for any $\mu \in \mathscr{P} \cup \{\mu^\star \}$,
\begin{equation*}
    \int_{K^\complement} ||x_0-x|| \mu({\rm d}x) \leqslant \varepsilon/2.
\end{equation*}

By Arzel\`a-Ascoli theorem, it is easy to see that $\text{Lip}_1(K)$, the set of $1$-Lipschitz real-valued functions on $K$, is compact with respect to the uniform norm $\|\cdot\|_\infty$ on $K$. Let $\{f_1, \hdots, f_{\mathscr{N}_\varepsilon}\}$ denote an $\varepsilon$-covering of $\text{Lip}_1(K)$. According to \citet[][Theorem 3]{Anil2018SortingOL}, for each $k=1, \hdots, \mathscr{N}_\varepsilon$ there exists under Assumption \ref{ass:compactness} a discriminator $\mathscr{D}_k$ of the form \eqref{eq:def_discriminators} such that
  \begin{equation*}
      \underset{g \in \mathscr{D}_k}{\inf} \ \|f_k- g \mathds{1}_K  \|_\infty \leqslant \varepsilon.
  \end{equation*}
Since the discriminative classes of functions use GroupSort activations, one can find a neural network of the form \eqref{eq:def_discriminators} satisfying Assumption \ref{ass:compactness}, say $\mathscr{D}_{\text{max}}$, such that, for all $k \in \{ 1, \hdots, \mathscr{N}_\varepsilon \}$, $\mathscr{D}_k \subseteq \mathscr{D}_{\text{max}}$. Consequently, for any $f \in \text{Lip}_1(K)$, letting $k_0 \in \underset{k \in \{1, \hdots, \mathscr{N}_\varepsilon\}}{\argmin} \ \|f - f_k\|_\infty$, we have
  \begin{equation*}
      \underset{g \in \mathscr{D}_{\text{max}}}{\inf} \ \|f - g \mathds{1}_K \|_\infty \leqslant \|f - f_{k_0} \|_\infty + \underset{g \in \mathscr{D}_{\text{max}}}{\inf} \ \|f_{k_0} - g \mathds{1}_K \|_\infty \leqslant 2 \varepsilon.
  \end{equation*}
Now, let $(\mu, \nu) \in (\mathscr{P} \bigcup \{\mu^\star\})^2$ be such that $d_{\mathscr{D}_{\text{max}}}(\mu, \nu) = 0$, i.e., $\sup_{f \in \mathscr{D}_{\text{max}}} \ |\mathds E_{\mu} f -\mathds E_{\nu}f| = 0$.
Let $f^\star$ be a function in $\text{Lip}_1$ such that $\mathds{E}_{\mu} f ^\star- \mathds{E}_{\nu} f^\star = d_{\text{Lip}_1}(\mu, \nu)$ (such a function exists according to \eqref{fstar}) and, without loss of generality, such that $f^\star(x_0) = 0$. Clearly,
\begin{align*}
    d_{\text{Lip}_1}(\mu, \nu) &= \mathds{E}_\mu f^\star -  \mathds{E}_\nu f^\star \\
    &\leqslant \Big| \int_K f^\star {\rm d}\mu - \int_K f^\star {\rm d}\nu \Big| + \Big| \int_{K^\complement} f^\star {\rm d}\mu - \int_{K^\complement} f^\star {\rm d}\nu \Big| \\
    & \leqslant \Big| \int_K f^\star {\rm d}\mu - \int_K f^\star {\rm d}\nu \Big| + \varepsilon.
\end{align*}
Letting $g_{f^\star} \in \mathscr{D}_{\text{max}}$ be such that
\begin{equation*}
    \|(f^\star-g_{f^\star}) \mathds{1}_K\|_{\infty} \leqslant \underset{g \in \mathscr{D}_{\text{max}}}{\inf}\ \|(f^\star-g)\mathds{1}_K\|_{\infty} + \varepsilon \leqslant 3\varepsilon,
\end{equation*}
we are thus led to
\begin{align*}
    d_{\text{Lip}_1}(\mu, \nu) \leqslant \Big| \int_K (f^\star-g_{f^\star}) {\rm d}\mu - \int_K (f^\star-g_{f^\star}) {\rm d}\nu  + \int_K g_{f^\star} {\rm d}\mu - \int_K g_{f^\star} {\rm d}\nu \Big| + \varepsilon.
\end{align*}
Observe, since $x_0 \in K$, that $|g_{f^\star}(x_0)| \leqslant 3\varepsilon$ and that, for any $x \in E$, $|g_{f^\star}(x)| \leqslant \|x_0-x\| + 3 \varepsilon$. Exploiting $\mathds E_{\mu} g_{f^\star} -\mathds E_{\nu}g_{f^\star} = 0$, we obtain
\begin{align*}
      d_{\text{Lip}_1}(\mu, \nu) & \leqslant 7\varepsilon + \Big| \int_{K^\complement} g_{f^\star} {\rm d}\mu - \int_{K^\complement} g_{f^\star} {\rm d}\nu \Big| \\
      &\leqslant 7 \varepsilon +  \int_{K^\complement} \|x_0-x\| \mu({\rm d}x) + \int_{K^\complement} \|x_0-x\| \nu({\rm d}x) + 6\varepsilon \\
      &\leqslant 14 \varepsilon.
  \end{align*}
Since $\varepsilon$ is arbitrary and $d_{\text{Lip}_1}$ is a metric on $P_1(E)$, we conclude that $\mu=\nu$, as desired.

To complete the proof, it remains to show that $d_{\mathscr{D}_{\text{max}}}$ metrizes weak convergence in $\mathscr{P} \cup \{\mu^\star\}$. To this aim, we let $(\mu_k)$ be a sequence in $\mathscr{P} \cup \{\mu^\star\}$ and $\mu$ be a probability measure in $\mathscr{P} \cup \{\mu^\star\}$.

If $(\mu_k)$ converges weakly to $\mu$ in $P_1(E)$, then $d_{\text{Lip}_1}(\mu, \mu_k) \to 0$ \citep[][Theorem 6.8]{villani2008optimal}, and, accordingly, $d_{\mathscr{D}_{\text{max}}}(\mu, \mu_k) \to 0$.

Suppose, on the other hand, that $d_{\mathscr{D}_{\text{max}}}(\mu, \mu_k) \to 0$, and fix $\varepsilon>0$. There exists $M>0$ such that, for all $k\geqslant M$,  $d_{\mathscr{D}_{\text{max}}}(\mu, \mu_k) \leqslant \varepsilon$. Using a similar reasoning as in the first part of the proof, it is easy to see that for any $k \geqslant M$, we have $d_{\text{Lip}_1}(\mu, \mu_k) \leqslant 15\varepsilon$. Since the Wasserstein distance metrizes weak convergence in $P_1(E)$ and $\varepsilon$ is arbitrary, we conclude that $(\mu_k)$ converges weakly to $\mu$ in $P_1(E)$.

\subsection{Proof of Lemma \ref{lem:d_star_not_empty}} \label{appendix:th_d_star_not_empty}
    Using a similar reasoning as in the proof of Proposition \ref{prop:neural_nets_tightness}, one easily checks that for all $(\alpha,\alpha') \in \Lambda^2$ and all $x \in E$,
    \begin{align*}
        |D_\alpha(x) - D_{\alpha'}(x)| &\leqslant Q^{1/2}\big(q \|x\| + K_2 \sum_{i=1}^{q-1} i + q \big) \|\alpha - \alpha' \| \\
        &\leqslant Q^{1/2}\big(q \|x\| + \frac{q (q-1)K_2}{2} + q \big) \|\alpha - \alpha' \|,
    \end{align*}
    where $q$ refers to the depth of the discriminator. Thus, since $\mathscr{D} \subset \text{Lip}_1$ (by Lemma  \ref{lem:uniformly_lipschitz_neural_nets}), we have, for all $\alpha \in \Lambda$, all $x \in E$, and any arbitrary $x_0 \in E$,
    \begin{align*}
        |D_\alpha(x)| &\leqslant |D_\alpha(x) - D_\alpha(x_0)| + |D_\alpha(x_0) | \\
        &\leqslant \|x_0-x\| + Q^{1/2}\big(q \|x_0\| + \frac{q (q-1)K_2}{2} + q \big) \|\alpha\| \\
        &\quad \mbox{(upon noting that $D_{0}(x_0)=0$)}\\
        &\leqslant \|x_0-x\| + Q^{1/2}\big(q \|x_0\| + \frac{q (q-1)K_2}{2} + q \big) Q^{1/2} \max(K_2,1),
    \end{align*}
    where $Q$ is the dimension of $\Lambda$. Thus, since $\mu^{\star}$ and the $\mu_{\theta}$'s belong to $P_1(E)$ (by Lemma \ref{lem:uniformly_lipschitz_neural_nets}), we deduce that all $D_\alpha \in \mathscr D$ are dominated by a function independent of $\alpha$ and integrable with respect to $\mu^\star$ and $\mu_\theta$. In addition, for all $x \in E$, the function $\alpha \mapsto D_\alpha(x)$ is continuous on $\Lambda$. Therefore, by the dominated convergence theorem, the function $\Lambda \ni \alpha \mapsto |\mathds{E}_{\mu^\star} D_\alpha - \mathds{E}_{\mu_\theta} D_\alpha|$ is continuous. The conclusion follows from the compactness of the set $\Lambda$ (Assumption \ref{ass:compactness}).
    \subsection{Proof of Theorem \ref{th:continuity}}
        Let $(\theta, \theta') \in \Theta^2$, and let $\gamma_Z$ be the joint distribution of the pair $(G_\theta(Z), G_{\theta'}(Z))$. We have
        \begin{align*}
            |\xi_{\text{Lip}_1}(\theta) - \xi_{\text{Lip}_1}(\theta')| &=
            |d_{\text{Lip}_1}(\mu^\star, \mu_\theta) - d_{\text{Lip}_1}(\mu^\star, \mu_{\theta'})| \\
            & \leqslant d_{\text{Lip}_1}(\mu_\theta, \mu_{\theta'}) \\
            &= \underset{\gamma \in \Pi(\mu_\theta, \mu_{\theta'})}{\inf} \int_{E^2} \|x-y\| \gamma({\rm d}x,{\rm d}y),
            \end{align*}
    where $\Pi(\mu_\theta, \mu_{\theta'})$ denotes the collection of all joint probability measures on $E \times E$ with marginals $\mu_{\theta}$ and $\mu_{\theta'}$.  Thus,
            \begin{align*}
             |\xi_{\text{Lip}_1}(\theta) - \xi_{\text{Lip}_1}(\theta')| & \leqslant \int_{E^2} \|x-y\| \gamma_Z({\rm d}x,{\rm d}y) \\
            &= \int_{\mathds{R}^d} \|G_\theta(z)- G_{\theta'}(z)\| \nu({\rm d} z) \\
            & \quad \mbox{(where $\nu$ is the distribution of $Z$)} \\
            &\leqslant \|\theta-\theta'\|\int_{\mathds{R}^d} (\ell_1 \|z\| + \ell_2) \nu({\rm d} z) \\
            & \quad \mbox{(by inequality \eqref{eq:G_lipschitz_inequality})}.
        \end{align*}
    This shows that the function $\theta \ni \Theta \mapsto \xi_{\text{Lip}_1}(\theta)$ is $L$-Lipschitz, with $L = \int_{\mathds{R}^d} (\ell_1 \|z\| + \ell_2) \nu({\rm d} z)$. For the second statement of the theorem, just note that
        \begin{align*}
            |\xi_{\mathscr{D}}(\theta) - \xi_{\mathscr{D}}(\theta')| &=
            |d_{\mathscr{D}}(\mu^\star, \mu_\theta) - d_{\mathscr{D}}(\mu^\star, \mu_{\theta'})| \\
            & \leqslant d_{\mathscr{D}}(\mu_\theta, \mu_{\theta'}) \\
            & \leqslant d_{\text{Lip}_1}(\mu_\theta, \mu_{\theta'}) \\
            & \quad \mbox{(since $\mathscr{D} \subseteq \text{Lip}_1$)} \\
            &\leqslant L \|\theta - \theta'\|.
        \end{align*}
    \subsection{Proof of Theorem \ref{th:approx_properties}}
    The proof is divided into two parts. First, we show that under Assumption 1, for all $\varepsilon>0$ and $\theta \in \Theta$, there exists a discriminator $\mathscr D$ (function of $\varepsilon$ and $\theta$) of the form \eqref{eq:def_discriminators} such that
        \begin{equation*}
            d_{\text{Lip}_1}(\mu^\star, \mu_\theta) - d_{\mathscr{D}}(\mu^\star, \mu_\theta) \leqslant  10 \varepsilon .
        \end{equation*}
    Let $f^\star$ be a function in $\text{Lip}_1$ such that $\mathds{E}_{\mu^\star} f ^\star- \mathds{E}_{\mu_\theta} f^\star = d_{\text{Lip}_1}(\mu^\star, \mu_\theta)$ (such a function exists according to \eqref{fstar}). 
    We may write
        \begin{align}
            d_{\text{Lip}_1}(\mu^\star, \mu_\theta) - d_{\mathscr{D}}(\mu^\star, \mu_\theta) &= \mathds{E}_{\mu^\star} f^\star - \mathds{E}_{\mu_\theta} f^\star - \underset{f \in \mathscr{D}}{\sup} \ |\mathds{E}_{\mu^\star} f - \mathds{E}_{\mu_\theta} f| \nonumber \\
            &= \mathds{E}_{\mu^\star} f^\star - \mathds{E}_{\mu_\theta} f^\star - \underset{f \in \mathscr{D}}{\sup} \ (\mathds{E}_{\mu^\star} f - \mathds{E}_{\mu_\theta} f) \nonumber\\
            & = \underset{f \in \mathscr{D}}{\inf}\ (\mathds{E}_{\mu^\star} f^\star - \mathds{E}_{\mu_\theta} f^\star - \mathds{E}_{\mu^\star} f + \mathds{E}_{\mu_\theta} f) \nonumber\\
            &= \underset{f \in \mathscr{D}}{\inf} \ ( \mathds{E}_{\mu^\star} (f^\star - f) - \ \mathds{E}_{\mu_\theta} (f^\star - f) ) \nonumber\\
            &\leqslant \underset{f \in \mathscr{D}}{\inf} \ (\mathds{E}_{\mu^\star} |f^\star - f| +  \mathds{E}_{\mu_\theta} |f^\star - f|). \label{14:02}
        \end{align}
    Next, for any $f \in \mathscr D$ and any compact $K \subseteq E$,
        \begin{align*}\label{eq:simple_conv_eq}
            \mathds{E}_{\mu^\star} |f^\star - f| &= \mathds{E}_{\mu^\star} |f^\star -f|{\mathds 1}_{K} + \mathds{E}_{\mu^\star} |f^\star-f|{\mathds 1}_{K^\complement} \nonumber\\
            &\leqslant \|(f^\star-f){\mathds 1}_{K}\|_{\infty}+ \mathds{E}_{\mu^\star} |f^\star|{\mathds 1}_{K^\complement} + \mathds{E}_{\mu^\star} |f|{\mathds 1}_{K^\complement}.
        \end{align*}

    For the rest of the proof, we will assume, without loss of generality, that $f^\star(0)=0$ and thus $|f^\star(x)| \leqslant |x|$. Therefore, there exists a compact set $K$ such that $0 \in K$ and 
    \begin{equation*}
        \max\big(\mathds{E}_{\mu^\star} |f^\star|{\mathds 1}_{K^\complement}, \ \mathds{E}_{\mu_\theta} |f^\star|{\mathds 1}_{K^\complement} \big) \leqslant \varepsilon. 
    \end{equation*}
    %since $|f^\star(x)| = |f^\star(x)-f^\star(0)| \leqslant |x|$. 
    
    Besides, according to \citet[][Theorem 3]{Anil2018SortingOL}, under Assumption 1, for any compact $K$, we can find a discriminator of the form \eqref{eq:def_discriminators} such that ${\inf}_{f \in \mathscr{D}}\ \| (f^\star - f)\mathds 1_K \|_{\infty} \leqslant \varepsilon$. So, choose $f \in \mathscr{D}$ such that $\|(f^\star - f)\mathds 1_K \|_{\infty} \leqslant 2\varepsilon$. For such a choice of $f$, we have, for any $x \in E$, $|f(x)| \leqslant |f(x)-f(0)|+|f(0)| \leqslant |x| + 2\varepsilon$, and thus, recalling that $f^\star(0)=0$,
    \begin{equation*}
        \max\big(\mathds{E}_{\mu^\star} |f|{\mathds 1}_{K^\complement}, \ \mathds{E}_{\mu_\theta} |f|{\mathds 1}_{K^\complement}\big) \leqslant 3\varepsilon.
    \end{equation*}
    
    Consequently,
    \begin{equation*}
          \mathds{E}_{\mu^\star} |f^\star - f| \leqslant\|(f^\star-f){\mathds 1}_{K}\|_{\infty} +  4\varepsilon .
    \end{equation*}
    Similarly,
       \begin{equation*}
          \mathds{E}_{\mu_{\theta}} |f^\star - f| \leqslant\|(f^\star-f){\mathds 1}_{K}\|_{\infty} +  4\varepsilon .
      \end{equation*}
    Plugging the two inequalities above in \eqref{14:02}, we obtain
     \begin{equation*}
      d_{\text{Lip}_1}(\mu^\star, \mu_\theta) - d_{\mathscr{D}}(\mu^\star, \mu_\theta) \leqslant 2 \underset{f \in \mathscr{D}}{\inf} \ \|(f^\star-f){\mathds 1}_{K}\|_{\infty} + 8\varepsilon .
     \end{equation*}
    We conclude that, for this choice of $\mathscr D$ (function of $\varepsilon$ and $\theta$),
    \begin{equation}\label{eq:33}
        d_{\text{Lip}_1}(\mu^\star, \mu_\theta) - d_{\mathscr{D}}(\mu^\star, \mu_\theta) \leqslant 10\varepsilon ,
    \end{equation}
    as desired. 

    For the second part of the proof, we fix $\varepsilon >0$ and let, for each $\theta \in \Theta$ and each discriminator of the form \eqref{eq:def_discriminators},
    \begin{align*}
        \hat{\xi}_{\mathscr{D}}(\theta) = d_{\text{Lip}_1}(\mu^\star, \mu_\theta) - d_{\mathscr{D}}(\mu^\star, \mu_\theta).
    \end{align*}
    Arguing as in the proof of Theorem \ref{th:continuity}, we see that $\hat{\xi}_{\mathscr{D}}(\theta)$ is $2L$-Lipschitz in $\theta$, where $L = \int_{\mathds{R}^d} (\ell_1\|z\| + \ell_2) \nu({\rm d}z)$ and $\nu$ is the probability distribution of $Z$.

    Now, let $\{\theta_1, \hdots, \theta_{\mathscr{N}_\varepsilon}\}$ be an $\varepsilon$-covering of the compact set $\Theta$, i.e., for each $\theta \in \Theta$, there exists $k \in \{1,\hdots, \mathscr{N}_\varepsilon\}$ such that $\|\theta - \theta_k \| \leqslant \varepsilon$. According to \eqref{eq:33}, for each such $k$, there exists a discriminator $\mathscr{D}_k$ such that $\hat{\xi}_{\mathscr{D}_k}(\theta_k) \leqslant 6\varepsilon$. Since the discriminative classes of functions use GroupSort activation functions, one can find a neural network of the form \eqref{eq:def_discriminators} satisfying Assumption \ref{ass:compactness}, say $\mathscr{D}_{\text{max}}$, such that, for all $k \in \{ 1, \hdots, \mathscr{N}_\varepsilon \}$, $\mathscr{D}_k \subseteq \mathscr{D}_{\text{max}}$. Clearly, $\hat{\xi}_{\mathscr{D}_{\text{max}}}(\theta)$ is $2L$-Lipschitz, and, for all $k \in \{1,\hdots, \mathscr{N}_\varepsilon\}$, $\hat{\xi}_{\mathscr{D}_{\text{max}}}(\theta_k) \leqslant 6\varepsilon$. Hence, for all $\theta \in \Theta$, letting
    \begin{equation*}
         \hat{k} \in \underset{k \in \{1, \hdots, \mathscr{N}_\varepsilon\}}{\argmin} \ \| \theta - \theta_k \|,
    \end{equation*}
    we have
    \begin{equation*}
        \hat{\xi}_{\mathscr{D}_{\text{max}}}(\theta) \leqslant \big|\hat{\xi}_{\mathscr{D}_{\text{max}}}(\theta) - \hat{\xi}_{\mathscr{D}_{\text{max}}}(\theta_{\hat{k}})\big| + \hat{\xi}_{\mathscr{D}_{\text{max}}}(\theta_{\hat{k}}) \leqslant (2L+6)\varepsilon.
    \end{equation*}
    Therefore,
    \begin{equation*}
        T_{\mathscr{P}}(\text{Lip}_1, \mathscr{D}_{\text{max}}) = \underset{\theta \in \Theta}{\sup} \big[d_{\text{Lip}_1}(\mu^\star, \mu_\theta) - d_{\mathscr{D}_{\text{max}}}(\mu^\star, \mu_\theta)\big] = \underset{\theta \in \Theta}{\sup} \ \hat{\xi}_{\mathscr{D}_{\text{max}}}(\theta) \leqslant (2L+6)\varepsilon.
    \end{equation*}
    We have just proved that, for all $\varepsilon >0$, there exists a discriminator $\mathscr{D}_{\text{max}}$ of the form \eqref{eq:def_discriminators} and a positive constant $c$ (independent of $\varepsilon$) such that
    \begin{align*}
        T_{\mathscr{P}}(\text{Lip}_1, \mathscr{D}_{\text{max}}) \leqslant c \varepsilon.
    \end{align*}
    This is the desired result.
    \subsection{Proof of Proposition \ref{prop:substitution}}\label{appendix:proposition_two}
    Let us assume that the statement is not true. If so, there exists $\varepsilon>0$ such that, for all $\delta>0$, there exists $\theta \in \mathscr{M}_{d_\mathscr D}(\mu^\star, \delta)$ satisfying $d(\theta, \bar{\Theta}) > \varepsilon$. Consider $\delta_n = {1}/{n}$, and choose a sequence of parameters $(\theta_n)$ such that
    \begin{equation*}\label{eq:hypothesis_proof_proposition2}
        \theta_n \in \mathscr{M}_{d_\mathscr D}\Big(\mu^\star, \frac{1}{n}\Big) \quad \text{and} \quad d(\theta_n, \bar{\Theta}) > \varepsilon.
    \end{equation*}
    Since $\Theta$ is compact by Assumption \ref{ass:compactness}, we can find a subsequence $(\theta_{\varphi_n})$ that converges to some $\theta_{\text{acc}} \in \Theta$. Thus, for all $n \geqslant 1$, we have
    \begin{equation*}
        d_{\mathscr{D}}(\mu^\star, \mu_{\theta_{\varphi_n}}) \leqslant \underset{\theta \in \Theta}{\inf} \  d_{\mathscr{D}}(\mu^\star, \mu_\theta) + \frac{1}{n},
    \end{equation*}
    and, by continuity of the function $\Theta \ni \theta \mapsto d_{\mathscr{D}}(\mu^\star, \mu_{\theta})$ (Theorem  \ref{th:continuity}),
    \begin{equation*}
        d_{\mathscr{D}}(\mu^\star, \theta_{\text{acc}}) \leqslant \underset{\theta \in \Theta}{\inf} \ d_{\mathscr{D}}(\mu^\star, \mu_\theta).
    \end{equation*}
    We conclude that $\theta_{\text{acc}}$ belongs to $\bar{\Theta}$. This contradicts the fact that $d(\theta_{\text{acc}}, \bar{\Theta}) \geqslant \varepsilon$.

    \subsection{Proof of Lemma \ref{lem:monotonous_equivalence_a_equal_b}} \label{appendix:prop_monotonous_equivalence_a_equal_b}
        Since $a=b$, according to Definition \ref{def:monotonous_equivalence}, there exists a continuously differentiable, strictly increasing function $f:\mathds R^+\to \mathds R^+$ such that, for all $\mu \in \mathscr{P}$,
        \begin{align*}
            d_{\text{Lip}_1}(\mu^\star, \mu) = f(d_{\mathscr{D}}(\mu^\star, \mu)).
        \end{align*}
    For $(\theta, \theta') \in \Theta^2$ we have, as $f$ is strictly increasing,
        \begin{equation*}
            d_{\mathscr{D}}(\mu^\star, \mu_{\theta}) \leqslant  d_{\mathscr{D}}(\mu^\star, \mu_{\theta'}) \iff
             f (d_{\mathscr{D}}(\mu^\star, \mu_{\theta})) \leqslant f (d_{\mathscr{D}}(\mu^\star, \mu_{\theta'})).
              \end{equation*}
    Therefore,
     \begin{equation*}
             d_{\mathscr{D}}(\mu^\star, \mu_{\theta}) \leqslant d_{\mathscr{D}}(\mu^\star, \mu_{\theta'}) \iff d_{\text{Lip}_1}(\mu^\star, \mu_{\theta}) \leqslant  d_{\text{Lip}_1}(\mu^\star, \mu_{\theta'}).
              \end{equation*}
    This proves the first statement of the lemma.

    Let us now show that $d_{\text{Lip}_1}$ can be fully substituted by $d_{\mathscr{D}}$. Let $\varepsilon >0$. Then, for $\delta>0$ (function of $\varepsilon$, to be chosen later) and $\theta \in \mathscr{M}_{d_{\mathscr{D}}}(\mu^\star, \delta)$, we have
    \begin{align*}
        d_{\text{Lip}_1}(\mu^\star, \mu_\theta) - \underset{\theta \in \Theta}{\inf}\ d_{\text{Lip}_1}(\mu^\star, \mu_{\theta})
        &=  f (d_{\mathscr{D}}(\mu^\star, \mu_\theta)) -\underset{\theta \in \Theta}{\inf}\ f(d_{\mathscr D}(\mu^\star, \mu_{\theta})) \\
        &=  f (d_{\mathscr{D}}(\mu^\star, \mu_\theta)) - f ( \underset{\theta \in \Theta}{\inf} \ d_{\mathscr{D}}(\mu^\star, \mu_{{\theta}}) )\\
        &\leqslant \underset{\theta \in \mathscr{M}_{d_{\mathscr{D}}}(\mu^\star, \delta)}{\sup} \big|f (d_{\mathscr{D}}(\mu^\star, \mu_\theta)) - f ( \underset{\theta \in \Theta}{\inf} \ d_{\mathscr{D}}(\mu^\star, \mu_{{\theta}}) )\big|.
    \end{align*}
    According to Theorem \ref{th:continuity}, there exists a nonnegative constant $c$ such that for any $\theta \in \Theta$, $d_{\mathscr{D}}(\mu^\star, \mu_\theta) \leqslant c$. Therefore, using the definition of $\mathscr{M}_{d_{\mathscr{D}}}(\mu^\star, \delta)$ and the fact that $f$ is continuously differentiable, we are led to
    \begin{equation*}
        d_{\text{Lip}_1}(\mu^\star, \mu_\theta) - \underset{\theta \in \Theta}{\inf} \ d_{\text{Lip}_1}(\mu^\star, \mu_{\theta})  \leqslant \delta \underset{x \in [0, c]}{\sup} \ \Big|\frac{\partial f(x)}{\partial x}\Big|.
    \end{equation*}
    The conclusion follows by choosing $\delta$ such that $\delta \underset{x \in [0,c]}{\sup} |\frac{\partial f(x)}{\partial x}| \leqslant \varepsilon$.
    \subsection{Proof of Proposition \ref{prop:monotonous_equivalence}} \label{appendix:prop_monotonous_equivalence}
    Let $\delta \in (0,1)$ and $\theta \in \mathscr{M}_{d_\mathscr{D}}(\mu^\star, \delta)$, i.e., $d_{\mathscr{D}}(\mu^\star, \mu_\theta) -{\inf}_{\theta \in \Theta} \ d_{\mathscr{D}}(\mu^\star, \mu_{{\theta}}) \leqslant \delta$. As $d_{\text{Lip}_1}$ is monotonously equivalent to $d_{\mathscr{D}}$, there exists a continuously differentiable, strictly increasing function $f: \mathds{R}^+ \to \mathds{R}^+$ and $(a, b) \in (\mathds{R}^{\star}_+)^2$ such that
    \begin{equation*}
            \forall \mu \in \mathscr{P}, \ a  f(d_{\mathscr{D}}(\mu^\star, \mu)) \leqslant d_{\text{Lip}_1}(\mu^\star, \mu) \leqslant b  f(d_{\mathscr{D}}(\mu^\star, \mu)).
        \end{equation*}
    So,
    \begin{align*}
        d_{\text{Lip}_1}(\mu^\star, \mu_\theta) &\leqslant bf ( \underset{\theta \in \Theta}{\inf} \ d_{\mathscr{D}}(\mu^\star, \mu_{{\theta}}) + \delta ) \\
        &\leqslant bf (\underset{\theta \in \Theta}{\inf} \ d_{\mathscr{D}}(\mu^\star, \mu_{{\theta}}) ) + O(\delta).
    \end{align*}
    Also,
    \begin{equation*}
        \underset{\theta \in \Theta}{\inf} \ d_{\text{Lip}_1}(\mu^\star, \mu_{\theta}) \geqslant a f (\underset{\theta \in \Theta}{\inf} \ d_{\mathscr{D}}(\mu^\star, \mu_{{\theta}})).
    \end{equation*}
    Therefore,
    \begin{equation*}
        d_{\text{Lip}_1}(\mu^\star, \mu_\theta) - \underset{\theta \in \Theta}{\inf} \ d_{\text{Lip}_1}(\mu^\star, \mu_{\theta}) \leqslant (b-a) f ( \underset{\theta \in \Theta}{\inf} \ d_{\mathscr{D}}(\mu^\star, \mu_{{\theta}})) + O(\delta).
    \end{equation*}
    \subsection{Proof of Lemma \ref{lem:ReLU_net_and_affine_functions}}
    Let $f: \mathds{R}^D \to \mathds{R}$ be in $\text{AFF} \cap \text{Lip}_1$. It is of the form $f(x) = x \cdot u + b$, where $u = (u_1, \hdots, u_D)$, $b \in \mathds{R}$, and $\|u\| \leqslant 1$. Our objective is to prove that there exists a discriminator of the form \eqref{eq:def_discriminators} with $q=2$ and $v_1=2$ that contains the function $f$. To see this, define $V_1 \in \mathscr{M}_{(2, D)}$ and the offset vector $c_1 \in \mathscr{M}_{(2, 1)}$ as
    \begin{equation*}
        V_1 = \begin{bmatrix}
            u_1& \cdots & u_D \\
            u_1& \cdots & u_D
        \end{bmatrix}
        \quad  \mbox{and} \quad c_1 = \begin{bmatrix} 0 \\ \vdots \\0 \end{bmatrix}.
        \end{equation*}
    Letting $V_2 \in \mathscr{M}_{(1, 2)}$, $c_2 \in \mathscr{M}_{(1, 1)}$ be
    \begin{equation*}
        V_2 =
        \begin{bmatrix}
            1 & 0
        \end{bmatrix},
        \quad  c_2 = \begin{bmatrix} b \end{bmatrix},
    \end{equation*}
    we readily obtain $V_2 \Tilde{\sigma} (V_1 x + c_1 )+c_2 = f(x)$. Besides, it is easy to verify that $\|V_1\|_{2,\infty} \leqslant 1$.
    \subsection{Proof of Lemma \ref{lem:wasserstein_distance_1d}}
    Let $\mu$ and $\nu$ be two probability measures in $P_1(E)$ with supports $S_\mu$ and $S_\nu$ satisfying the conditions of the lemma. Let $\pi$ be an optimal coupling between $\mu$ and $\nu$, and let $(X,Y)$ be a random pair with distribution $\pi$ such that
        \begin{equation*}
            d_{\text{Lip}_1}(\mu, \nu) = \mathds{E} \|X-Y\|.
        \end{equation*}
    Clearly, any function $f_0\in \text{Lip}_1$ satisfying $f_0(X) - f_0(Y) = \|X-Y\|$ almost surely will be such that
    \begin{equation*}
     d_{\text{Lip}_1}(\mu, \nu)=| \mathds{E}_{\mu} f_0 - \mathds{E}_{\nu} f_0 | .
    \end{equation*}
    The proof will be achieved if we show that such a function $f_0$ exists and that it may be chosen linear. Since $S_\mu$ and $S_\nu$ are disjoint and convex, we can find a unit vector $u$ of $\mathds{R}^D$ included in the line containing both $S_\mu$ and $S_\nu$ such that $(x_0-y_0) \cdot u > 0 $, where $(x_0, y_0)$ is an arbitrary pair of $S_{\mu}\times S_{\nu}$. Letting $f_0(x)= x \cdot u$ ($x\in E$), we have, for all $(x,y) \in S_{\mu} \times S_{\nu}$, $f_0(x) - f_0(y) = (x-y) \cdot u = \|x-y\|$. Since $f_0$ is a linear and $1$-Lipschitz function on $E$, this concludes the proof.

    \subsection{Proof of Lemma \ref{lem:gaussian_case}} \label{appendix:lem_gaussian_case}
    For any pair of probability measures $(\mu, \nu)$ on $E$ with finite moment of order $2$, we let $W_2(\mu,\nu)$ be the Wasserstein distance of order $2$ between $\mu$ and $\nu$. Recall  \citep[][Definition 6.1]{villani2008optimal} that
    \begin{equation*}
        W_2(\mu, \nu)=\Big(\inf_{\pi \in \Pi (\mu, \nu)}\int_{E \times E} \|x-y\|^2\pi({\rm d}x,{\rm d}y)\Big)^{1/2},
    \end{equation*}
    where $\Pi(\mu, \nu)$ denotes the collection of all joint probability measures on $E \times E$ with marginals $\mu$ and $\nu$. By Jensen's inequality,
    \begin{equation*}
        d_{\text{Lip}_1}(\mu, \nu) = W_1(\mu, \nu) \leqslant W_2(\mu, \nu).
    \end{equation*}
    Let $\Sigma \in \mathscr{M}_{(D,D)}$ be a positive semi-definite matrix, and let $\mu$ be Gaussian $\mathscr{N}(m_1, \Sigma)$ and $\nu$ be Gaussian $\mathscr{N}(m_2, \Sigma)$. Denoting by $(X,Y)$ a random pair with marginal distributions $\mu$ and $\nu$ such that
    \begin{equation*}
        \mathds{E}\|X-Y\|  = W_1 (\mu, \nu),
    \end{equation*}
    we have
    \begin{equation*}
           \|m_1 - m_2\| = \|\mathds{E} (X - Y) \| \leqslant \mathds{E}\|X-Y\| = W_1 (\mu, \nu) \leqslant W_2 (\mu, \nu) = \|m_1 - m_2\|,
    \end{equation*}
    where the last equality follows from \citet[][Proposition 7]{givens1984class}. Thus, $d_{\text{Lip}_1}(\mu, \nu) = \|m_1 - m_2\|$. The proof will be finished if we show that
    \begin{equation*}
        d_{\text{AFF} \cap \text{Lip}_1} (\mu, \nu) \geqslant \|m_1-m_2\|.
        \end{equation*}
    To see this, consider the linear and $1$-Lipschitz function $f: E \ni x \mapsto x \cdot \frac{(m_1-m_2)}{\|m_1-m_2\|}$ (with the convention $0\times \infty =0$), and note that
    \begin{align*}
        d_{\text{AFF} \cap \text{Lip}_1}(\mu, \nu)
        &\geqslant \Big| \int_{E} x \cdot \frac{(m_1 - m_2)}{\|m_1 - m_2\|} \mu({\rm d}x) - \int_{E} y \cdot \frac{(m_1 - m_2)}{\|m_1 - m_2\|} \nu({\rm d}y)  \Big| \\
        &= \Big| \int_{E} x \cdot \frac{(m_1 - m_2)}{\|m_1 - m_2\|} \mu({\rm d}x) - \int_{E} (x-m_1+m_2) \cdot \frac{(m_1 - m_2)}{\|m_1 - m_2\|} \mu({\rm d}x) \Big|  \\
        &= \|m_1 - m_2\|.
    \end{align*}
    \subsection{Proof of Proposition \ref{prop:cluster_wasserstein}}
    Let $\varepsilon >0$, and let $\mu$ and $\nu$ be two probability measures in $P_1(E)$ with compact supports $S_\mu$ and $S_\nu$ such that $\max( \text{diam}(S_\mu), \text{diam}(S_\nu)) \leqslant \varepsilon d(S_\mu, S_\nu)$. Throughout the proof, it is assumed that $d(S_\mu, S_\nu)>0$, otherwise the result is immediate. Let $\pi$ be an optimal coupling between $\mu$ and $\nu$, and let $(X,Y)$ be a random pair with distribution $\pi$ such that
        \begin{equation*}
            d_{\text{Lip}_1}(\mu, \nu) = \mathds{E} \|X-Y\|.
        \end{equation*}
    Any function $f_0\in \text{Lip}_1$ satisfying $\|X-Y\| \leqslant (1+2\varepsilon)(f_0(X) - f_0(Y))$ almost surely will be such that
    \begin{equation*}
     d_{\text{Lip}_1}(\mu, \nu) \leqslant (1+2\varepsilon) | \mathds{E}_{\mu} f_0 - \mathds{E}_{\nu} f_0 | .
    \end{equation*}
    Thus, the proof will be completed if we show that such a function $f_0$ exists and that it may be chosen affine.

    Since $S_\mu$ and $S_\nu$ are compact, there exists $(x^\star, y^\star) \in S_\mu \times S_\nu$ such that $\|x^\star - y^\star\| = d(S_\mu,S_\nu)$.
    By the hyperplane separation theorem, there exists a hyperplane $\mathscr{H}$ orthogonal to the unit vector $ u = \frac{x^\star - y^\star}{\|x^\star-y^\star\|}$ such that $d(x^\star, \mathscr{H})=d(y^\star, \mathscr{H}) = \frac{\|x^\star-y^\star\|}{2}$. For any $x \in E$, we denote by $p_\mathscr{H}(x)$ the projection of $x$ onto $\mathscr{H}$. We thus have $d(x,\mathscr{H})=\|x-p_\mathscr{H}(x) \|$, and $\frac{x^\star+y^\star}{2} = p_\mathscr{H}(\frac{x^\star+y^\star}{2}) = p_\mathscr{H}(x^\star) = p_\mathscr{H}(y^\star)$. In addition, by convexity of $S_\mu$ and $S_\nu$, for any $x \in S_\mu$, $\|x-p_\mathscr{H}(x)\| \geqslant \|x^\star-p_\mathscr{H}(x^\star)\|$. Similarly, for any $y \in S_\nu$, $\|y-p_\mathscr{H}(y)\| \geqslant \|y^\star-p_\mathscr{H}(y^\star)\|$.

    Let the affine function $f_0$ be defined for any $x \in E$ by
    \begin{equation*}
    f_0(x)= (x - p_\mathscr{H}(x)) \cdot u.
    \end{equation*}
    Observe that $f_0(x)=f_0(x+\frac{x^{\star}+y^{\star}}{2})$. Clearly, for any $(x,y) \in E^2$, one has
    \begin{align*}
        |f_0(x) - f_0(y)| &= \big|f_0\big(x-y+\frac{x^{\star}+y^{\star}}{2}\big)\big|\\
        &= \big| \big( \big(x-y+\frac{x^{\star}+y^{\star}}{2}\big) - p_\mathscr{H}\big(x-y+\frac{x^{\star}+y^{\star}}{2}\big) \big).u \big| \\
        &\leqslant \big\| \big(x-y+\frac{x^{\star}+y^{\star}}{2}) - p_\mathscr{H}\big(x-y+\frac{x^{\star}+y^{\star}}{2}\big) \big\| \\
        &\leqslant \big\|x-y+\frac{x^{\star}+y^{\star}}{2}- \frac{x^{\star}+y^{\star}}{2} \big\|\\
        & \quad \mbox{(since $\frac{x^{\star}+y^{\star}}{2} \in \mathscr H$)}\\
        &= \|x-y\|.
    \end{align*}
    Thus, $f_0$ belongs to $\text{Lip}_1$. Besides, for any $(x,y) \in S_\mu \times S_\nu$, we have
    \begin{align*}
        \|x-y\| & \leqslant \|x - p_\mathscr{H}(x)\| + \|p_\mathscr{H}(x) - p_\mathscr{H}(y)\| + \|p_\mathscr{H}(y) -y\| \\
        &\leqslant (x - p_\mathscr{H}(x)) \cdot u - (y - p_\mathscr{H}(y)) \cdot u +  \big\|p_\mathscr{H}(x) - \frac{x^\star + y^\star}{2}\big\| + \big\|p_\mathscr{H}(y) - \frac{x^\star + y^\star}{2} \big\| \\
        &= (x - p_\mathscr{H}(x)) \cdot u - (y - p_\mathscr{H}(y)) \cdot u +  \|p_\mathscr{H}(x) - p_\mathscr{H}(x^\star)\| + \|p_\mathscr{H}(y) - p_\mathscr{H}(y^\star) \|.
        \end{align*}
    Thus,
        \begin{align*}
        \|x-y\| & \leqslant (x - p_\mathscr{H}(x)) \cdot u - (y - p_\mathscr{H}(y)) \cdot u + 2 \max(\text{diam}(S_\mu), \text{diam}(S_\nu)) \\
        &\leqslant f_0(x) - f_0(y) + 2 \varepsilon d(S_\mu, S_\nu) \\
        &= f_0(x) - f_0(y) + 2 \varepsilon (f_0(x^\star) - f_0(y^\star)) \\
        &= f_0(x) - f_0(y) + 2 \varepsilon (f_0(x^\star) - f_0(x) + f_0(x) - f_0(y) + f_0(y) - f_0(y^\star) ) \\
        & \leqslant (1+2\varepsilon)(f_0(x) - f_0(y))\\
        & \quad \mbox{(using the fact that $f_0(x^\star) - f_0(x) \leqslant0$ and $f_0(y^\star) - f_0(y) \geqslant0$)}.
    \end{align*}
    Since $f_0 \in \text{Lip}_1$, we conclude that, for any $(x,y) \in S_\mu \times S_\nu$,
    \begin{equation*}
        |f_0(x) - f_0(y)| \leqslant \|x-y\| \leqslant (1+2\varepsilon) (f_0(x) - f_0(y)).
    \end{equation*}
    \subsection{Proof of Lemma \ref{10032020}}
    Using \citet[][Theorem 11.4.1]{dudley_2002} and the strong law of large numbers, the sequence of empirical measures $(\mu_n)$ almost surely converges weakly in $P_1(E)$ to $\mu^\star$. Thus, we have $\underset{n \to \infty}{\lim} d_{\text{Lip}_1}(\mu^\star, \mu_n) = 0$  almost surely, and so $\underset{n \to \infty}{\lim} d_{\mathscr{D}}(\mu^\star, \mu_n) = 0$ almost surely. Hence, recalling inequality \eqref{eq:epsilon_optim_bound}, we conclude that
        \begin{equation}\label{eq:thetan_converge_almost_surely}
            \underset{\theta_n \in \hat{\Theta}_n}{\sup} \  d_{\mathscr{D}}(\mu^\star, \mu_{{\theta}_n}) - \underset{\theta \in \Theta}{\inf} \ d_{\mathscr{D}}(\mu^\star, \mu_{{\theta}}) \to 0 \quad \text{almost surely.}
        \end{equation}
    Now, fix $\varepsilon>0$ and recall that, by our Theorem \ref{th:continuity}, the function $\Theta \ni \theta \mapsto d_{\text{Lip}_1}(\mu^\star, \mu_\theta)$ is $L$-Lipschitz, for some $L>0$. According to \eqref{eq:thetan_converge_almost_surely} and Proposition \ref{prop:substitution}, almost surely, there exists an integer $N>0$ such that, for all $n \geqslant N$, for all $\theta_n \in \hat{\Theta}_n$, the companion $\bar{\theta}_n \in \bar{\Theta}$ is such that $\|\theta_n-\bar{\theta}_n\| \leqslant \frac{\varepsilon}{L}$. We conclude by observing that $|\varepsilon_{\text{estim}}| \leqslant \sup_{\theta_n \in \hat{\Theta}_n}\ |d_{\text{Lip}_1}(\mu^\star, \mu_{{\theta}_n}) - d_{\text{Lip}_1}(\mu^\star, \mu_{\bar{\theta}_n})| \leqslant L \times \frac{\varepsilon}{L}$.
    \subsection{Proof of Proposition \ref{prop:generalization_bounds_for_dD}}\label{appendix:lem_generalization_bounds_for_dD}
    Let $\mu_n$ be the empirical measure based on $n$ i.i.d.~samples $X_1, \hdots, X_n$ distributed according to $\mu^\star$. Recall (equation \eqref{eq:IPMs}) that
    \begin{equation*}
        d_{\mathscr{D}}(\mu^\star, \mu_n) = \underset{\alpha \in \Lambda}{\sup}\ |\mathds{E}_{\mu^\star} D_\alpha - \mathds{E}_{\mu_n} D_\alpha|
        = \underset{\alpha \in \Lambda}{\sup} \ \Big|\mathds{E}_{\mu^\star} D_\alpha - \frac{1}{n} \sum_{i=1}^n D_\alpha(X_i) \Big|.
    \end{equation*}
    Let $g$ be the real-valued function defined on $E^n$ by
    \begin{equation*}
        g(x_1, \hdots, x_n) = \underset{\alpha \in \Lambda}{\sup}\ \Big|\mathds{E}_{\mu^\star} D_\alpha - \frac{1}{n} \sum_{i=1}^n D_\alpha(x_i) \Big|.
    \end{equation*}
    Observe that, for $(x_1, \hdots, x_n) \in E^n$ and $(x'_1, \hdots, x'_n) \in E^n$,
    \begin{align}
        |g(x_1, \hdots, x_n) - g(x'_1, \hdots, x'_n)| &\leqslant \underset{\alpha \in \Lambda}{\sup} \ \Big |\frac{1}{n} \sum_{i=1}^n D_\alpha(x_i) - \frac{1}{n} \sum_{i=1}^n D_\alpha(x'_i)  \Big| \nonumber \\
        &\leqslant \frac{1}{n} \underset{\alpha \in \Lambda}{\sup} \ \sum_{i=1}^n |D_\alpha(x_i) - D_\alpha(x'_i)|  \nonumber\\
        &\leqslant \frac{1}{n} \sum_{i=1}^n \ \| x_i - x'_i \|.\label{REF}
    \end{align}
    We start by examining statement $(i)$, where $\mu^\star$ has compact support with diameter $B$. In this case, letting $X'_i$ be an independent copy of $X_i$, we have, almost surely,
    \begin{equation*}
        |g(X_1, \hdots, X_n) - g(X_1, \hdots, X'_i, \hdots, X_n)| \leqslant \frac{B}{n}.
    \end{equation*}
    An application of McDiarmid's inequality \citep{mcdiarmid_1989} shows that for any $\eta \in (0,1)$,  with probability at least $1-\eta$,
    \begin{equation}\label{eq:25}
        d_{\mathscr{D}}(\mu^\star, \mu_n) \leqslant \mathds{E} d_{\mathscr{D}}(\mu^\star, \mu_n) + B \sqrt{ \frac{\log(1/\eta)}{2n}}.
    \end{equation}
    Next, for each $\alpha \in \Lambda$, let $Y_\alpha$ denote the random variable defined by \begin{equation*}
        Y_\alpha = \mathds{E}_{\mu^\star} D_\alpha - \frac{1}{n} \sum_{i=1}^n D_\alpha(X_i).
    \end{equation*}
    Using a similar reasoning as in the proof of Proposition \ref{prop:neural_nets_tightness}, one shows that for any $(\alpha, \alpha') \in \Lambda^2$ and any $x \in E$,
    \begin{equation*}
        |D_\alpha(x) - D_{\alpha'}(x) | \leqslant Q^{1/2} \big(q \|x\| + \frac{q (q-1)K_2}{2} + q \big) \|\alpha-\alpha'\|,
    \end{equation*}
    where we recall that $q$ is the depth of the discriminator. Since $\mu^\star$ has compact support,
    \begin{equation*}
        \ell = \int_{E} Q^{1/2} \big(q \|x\| + \frac{q (q-1)K_2}{2} + q \big) \mu^\star({\rm d} x) < \infty.
    \end{equation*}
    Observe that
    \begin{equation*}
        |Y_\alpha - Y_{\alpha'}| \leqslant \frac{1}{n} \|\alpha-\alpha'\| \ |\xi(n)|,
    \end{equation*}
    where
    \begin{equation*}
        \xi_n = \sum_{i=1}^n Q^{1/2} \big(\ell + q \|X_i\| + \frac{q (q-1)K_2}{2} + q \big).
    \end{equation*}
    Thus, using \citet[][Proposition 2.5.2]{vershynin2018high}, there exists a positive constant $c=O(qQ^{1/2}(D^{1/2}+q))$ such that, for all $\lambda \in \mathds R$,
    \begin{equation*}
        \mathds{E} e^{\lambda (Y_\alpha - Y_{\alpha'})} \leqslant \mathds{E} e^{\lambda \|\alpha-\alpha'\| \ |\frac{\xi_n}{n}|} \leqslant e^{c^2\frac{1}{n} \|\alpha-\alpha'\|^2 \lambda^2}.
    \end{equation*}
    We conclude that the process $(Y_\alpha)$ is sub-Gaussian \citep[][Definition 5.20]{vanhandel2016probability} for the distance $d(\alpha, \alpha') = \frac{c \|\alpha-\alpha'\|}{\sqrt{n}}$. Therefore, using \citet[][Corollary 5.25]{vanhandel2016probability}, we have
    \begin{equation*}
        \mathds{E} d_{\mathscr{D}}(\mu^\star, \mu_n)=\mathds{E} \underset{\alpha \in \Lambda}{\sup} \ \Big|\mathds{E}_{\mu^\star} D_\alpha - \frac{1}{n} \sum_{i=1}^n D_\alpha(X_i) \Big| \leqslant \frac{12 c}{\sqrt{n}} \int_0^\infty \sqrt{\log \mathscr{N}(\Lambda, \|\cdot\|, u)} {\rm d} u,
    \end{equation*}
    where $\mathscr{N}(\Lambda, \|\cdot\|, u)$ is the $u$-covering number of $\Lambda$ for the norm $\|\cdot\|$. Since $\Lambda$ is bounded, there exists $r>0$ such that $\mathscr{N}(\Lambda, \|\cdot\|, u) = 1$ for $u \geqslant rQ^{1/2}$ and
    \begin{equation*}
        \mathscr{N}(\Lambda, \|\cdot\|, u) \leq  \bigg( \frac{rQ^{1/2}}{u} \bigg)^Q \quad \mbox{for }u < rQ^{1/2}.
    \end{equation*}
    Thus,
    \begin{equation*}\label{eq:35}
        \mathds{E} d_{\mathscr{D}}(\mu^\star, \mu_n) \leqslant \frac{c_1}{\sqrt{n}}
    \end{equation*}
    for some positive constant $c_1=O(qQ^{3/2}(D^{1/2}+q))$. Combining this inequality with \eqref{eq:25} shows the first statement of the lemma.

    We now turn to the more general situation (statement $(ii)$) where $\mu^{\star}$ is $\gamma$ sub-Gaussian. According to inequality \eqref{REF}, the function $g$ is $\frac{1}{n}$-Lipschitz with respect to the $1$-norm on $E^n$. Therefore, by combining  \citet[][Theorem 1]{kontorovich2014concentration} and \citet[][Proposition 2.5.2]{vershynin2018high}, we have that for any $\eta \in (0,1)$, with probability at least $1-\eta$,
    \begin{equation}\label{eq:26}
        d_{\mathscr{D}}(\mu^\star, \mu_n) \leqslant \mathds{E} d_{\mathscr{D}}(\mu^\star, \mu_n) +  8 \gamma \sqrt{eD} \sqrt{\frac{\log(1/\eta)}{n}}.
    \end{equation}
    As in the first part of the proof, we let
    \begin{equation*}
        Y_\alpha = \mathds{E}_{\mu^\star} D_\alpha - \frac{1}{n} \sum_{i=1}^n D_\alpha(X_i),
    \end{equation*}
    and recall that for any $(\alpha, \alpha') \in \Lambda^2$ and any $x \in E$,
    \begin{equation*}
        |D_\alpha(x) - D_{\alpha'}(x) | \leqslant Q^{1/2} \big(q \|x\| + \frac{q (q-1)K_2}{2} + q\big) \|\alpha-\alpha'\|.
    \end{equation*}
    Since $\mu^\star$ is sub-Gaussian, we have \citep[see, e.g.,][Lemma 1]{jin2019short},
    \begin{equation*}
    \label{eq:local_lipschitzness}
        \ell = \int_{E} Q^{1/2} \big(q \|x\| + \frac{q (q-1)K_2}{2} + q\big) \mu^\star({\rm d} x) < \infty.
    \end{equation*}
    Thus,
    \begin{equation*}
        |Y_\alpha - Y_{\alpha'}| \leqslant \frac{1}{n} \|\alpha-\alpha'\| \ |\xi(n)|,
    \end{equation*}
    where
    \begin{equation*}
        \xi_n = \sum_{i=1}^n Q^{1/2} \big(\ell + q \|X_i\| + \frac{q (q-1)K_2}{2} + q \big).
    \end{equation*}
    According to \citet[][Lemma 1]{jin2019short}, the real-valued random variable $\xi_n$ is sub-Gaussian. We obtain that, for some positive constant $c_2=O(qQ^{3/2}(D^{1/2}+q))$,
    \begin{equation*}
        \mathds{E} d_{\mathscr{D}}(\mu^\star, \mu_n) \leqslant \frac{c_2}{\sqrt{n}},
    \end{equation*}
    and the conclusion follows by combining this inequality with \eqref{eq:26}.
    \subsection{Proof of Theorem \ref{theorem:asymptotic_approximation}}\label{appendix:theorem_asymptotic_approximation}
    Let $\varepsilon>0$ and $\eta \in (0,1)$. According to Theorem \ref{th:approx_properties}, there exists a discriminator $\mathscr{D}$ of the form \eqref{eq:def_discriminators} (i.e., a collection of neural networks) such that
    \begin{equation*}
        T_{\mathscr{P}}(\text{Lip}_1, \mathscr{D}) \leqslant \varepsilon.
    \end{equation*}
    We only prove statement $(i)$ since both proofs are similar. In this case, according to Proposition \ref{prop:generalization_bounds_for_dD}, there exists a constant $c_1>0$ such that, with probability at least $1-\eta$,
    \begin{equation*}
        d_{\mathscr{D}}(\mu^\star, \mu_n) \leqslant \frac{c_1}{\sqrt{n}} + B\sqrt{ \frac{\log(1/\eta)}{2n}}.
    \end{equation*}
    Therefore, using inequality \eqref{eq:eps_estim_and_esp_optim}, we have, with probability at least $1-\eta$,
    \begin{equation*}
        0 \leqslant \varepsilon_{\text{estim}} + \varepsilon_{\text{optim}} \leqslant 2\varepsilon + \frac{2c_1}{\sqrt{n}} + 2B\sqrt{ \frac{\log(1/\eta)}{2n}}.
    \end{equation*}

    \subsection{Proof of Proposition \ref{prop:asymptotic_optimization_properties} } \label{appendix:prop_asymptotic_optimization_properties}
    Observe that, for $\theta \in \Theta$,
    \begin{align*}
        0 &\leqslant d_{\mathscr{D}}(\mu^\star, \mu_\theta) - \underset{\theta \in \Theta}{\inf} \ d_{\mathscr{D}}(\mu^\star, \mu_{{\theta}}) \\
        & = d_{\mathscr{D}}(\mu^\star, \mu_\theta) - d_{\mathscr{D}}(\mu_n, \mu_\theta)+d_{\mathscr{D}}(\mu_n, \mu_\theta) - \underset{\theta \in \Theta}{\inf} \  d_{\mathscr{D}}(\mu_n, \mu_{{\theta}})\\
        & \quad +\underset{\theta \in \Theta}{\inf} \ d_{\mathscr{D}}(\mu_n, \mu_{{\theta}}) - \underset{\theta \in \Theta}{\inf} \ d_{\mathscr{D}}(\mu^\star, \mu_{{\theta}})\\
        & \leqslant d_{\mathscr{D}}(\mu^\star, \mu_n)+ d_{\mathscr{D}}(\mu_n, \mu_\theta) - \underset{\theta \in \Theta}{\inf} \  d_{\mathscr{D}}(\mu_n, \mu_{{\theta}}) + d_{\mathscr{D}}(\mu^\star, \mu_n)\\
        & =2d_{\mathscr{D}}(\mu^\star, \mu_n)+ d_{\mathscr{D}}(\mu_n, \mu_\theta) - \underset{\theta \in \Theta}{\inf} \  d_{\mathscr{D}}(\mu_n, \mu_{{\theta}}),
    \end{align*}
    where we used respectively the triangle inequality and
    \begin{equation*}
    |\underset{\theta \in \Theta}{\inf} \ d_{\mathscr{D}}(\mu_n, \mu_{{\theta}}) - \underset{\theta \in \Theta}{\inf} \ d_{\mathscr{D}}(\mu^\star, \mu_{{\theta}})| \leqslant\underset{\theta \in \Theta}{\sup} \ | d_{\mathscr{D}}(\mu^\star, \mu_{{\theta}}) - d_{\mathscr{D}}(\mu_n, \mu_{{\theta}})| \leqslant d_{\mathscr{D}}(\mu^\star, \mu_n).
    \end{equation*}
    Thus, assuming that $T_{\mathscr{P}}(\text{Lip}_1, \mathscr{D}) \leqslant \varepsilon$, we have
        \begin{align}
            0 &\leqslant d_{\text{Lip}_1}(\mu^\star, \mu_\theta) -  \underset{\theta \in \Theta}{\inf} \ d_{\text{Lip}_1}(\mu^\star, \mu_{{\theta}})   \nonumber \\
            & \leqslant d_{\text{Lip}_1}(\mu^\star, \mu_\theta) - d_{\mathscr{D}}(\mu^\star, \mu_\theta) + d_{\mathscr{D}}(\mu^\star, \mu_\theta) - \underset{\theta \in \Theta}{\inf} \ d_{\mathscr{D}}(\mu^\star, \mu_{{\theta}}) \nonumber\\
            & \leqslant T_{\mathscr{P}}(\text{Lip}_1, \mathscr{D}) + d_{\mathscr{D}}(\mu^\star, \mu_\theta) - \underset{\theta \in \Theta}{\inf} \ d_{\mathscr{D}}(\mu^\star, \mu_{{\theta}}) \nonumber\\
            & \leqslant \varepsilon + 2d_{\mathscr{D}}(\mu^\star, \mu_n) + d_{\mathscr{D}}(\mu_n, \mu_\theta) - \underset{\theta \in \Theta}{\inf} \  d_{\mathscr{D}}(\mu_n, \mu_{{\theta}}).\label{121019}
        \end{align}
        Let $\delta>0$ and $\theta \in \mathscr{M}_{d_{\mathscr{D}}}(\mu_n, {\delta}/{2})$, that is,
        \begin{equation*}
        d_{\mathscr{D}}(\mu_n, \mu_\theta) - \underset{\theta \in \Theta}{\inf} \ d_{\mathscr{D}}(\mu_n, \mu_{{\theta}})\leqslant \delta/2.
        \end{equation*}
    For $\eta \in (0,1)$, we know from the second statement of Proposition \ref{prop:generalization_bounds_for_dD} that there exists $N \in \mathds{N}^{\star}$ such that, for all $n \geqslant N$, $2d_{\mathscr{D}}(\mu^\star, \mu_n) \leqslant {\delta}/{2}$ with probability at least $1-\eta$. Therefore, we conclude from \eqref{121019} that for $n \geqslant N$, with probability at least $1-\eta$,
    \begin{equation*}
        d_{\text{Lip}_1}(\mu^\star, \mu_{\theta}) -  \underset{\theta \in \Theta}{\inf} \ d_{\text{Lip}_1}(\mu^\star, \mu_{{\theta}}) \leqslant \varepsilon + \delta.
    \end{equation*}
    
\end{document}